\documentclass[a4paper,11pt]{amsart}

\usepackage[T1]{fontenc}
\usepackage[utf8]{inputenc}

\usepackage{amsmath,amssymb,amsthm}
\usepackage{MnSymbol}
\usepackage{fullpage}
\usepackage[
natbib,
maxcitenames=3, 
maxbibnames=10,  
sorting=ydnt,
url=false,
doi=false,
sortcites,
defernumbers,
backend=biber
]{biblatex}
\usepackage[
]
{hyperref}

\usepackage{tikz}

\usepackage{todonotes}

\usepackage[nameinlink, capitalise, noabbrev]{cleveref}
\usepackage[normalem]{ulem}

\usepackage{caption}
\usepackage{subcaption}

\newtheorem{theorem}{Theorem}[section]
\newtheorem*{theorem*}{Theorem}
\newtheorem{proposition}[theorem]{Proposition}
\newtheorem{lemma}[theorem]{Lemma}
\newtheorem{corollary}[theorem]{Corollary}
\newtheorem{example}[theorem]{Example}
\theoremstyle{remark}
\newtheorem{remark}[theorem]{Remark}

\theoremstyle{definition}

\newtheorem{definition}[theorem]{Definition}
\newtheorem{assumption}{Assumption}

\numberwithin{equation}{section}

\newcommand{\Exp}[2]{\operatorname{\mathbb{E}}_{#1}\left[#2\right]}
\newcommand{\abs}[1]{\left|#1\right|}
\newcommand{\norm}[1]{\|#1\|}
\newcommand{\de}{\,\mathrm{d}}
\newcommand{\eps}{\varepsilon}

\newcommand{\e}{\varepsilon}

\renewcommand{\rho}{\varrho}
\newcommand{\N}{\mathbb{N}}

\newcommand{\R}{\mathbb{R}}
\newcommand{\Q}{\mathbb{Q}}
\newcommand{\Rd}{{\R^d}}

\newcommand{\domain}{\mathcal{X}}
\renewcommand{\P}{\mathcal{P}}

\newcommand{\B}{\mathfrak{B}}

\newcommand{\st}{\,:\,}
\newcommand{\dist}{\operatorname{\mathsf{dist}}}
\newcommand{\closure}[1]{\overline{#1}}
\newcommand{\wsto}{\rightharpoonup^\ast}
\DeclareMathOperator{\supp}{supp}
\newcommand{\loss}{\ell}

\newcommand{\metric}{\operatorname{\mathsf{d}}}

\DeclareMathOperator{\cl}{\mathsf{cl}}
\DeclareMathOperator{\op}{\mathsf{op}}
\newcommand{\PrePer}{\,\widetilde{\operatorname{Per}}}

\newcommand{\PreTV}{\,\widetilde{\operatorname{TV}}}

\newcommand{\TV}[1][]{
\ifthenelse { \equal {#1} {} }  
    {\operatorname{TV}} 
    {\operatorname{\ensuremath{#1}-TV}}
}

\newcommand{\Per}[1][]{
\ifthenelse { \equal {#1} {} }  
    {\operatorname{Per}} 
    {\operatorname{\ensuremath{#1}-Per}}
}

\newcommand{\esssup}[1][]{
\ifthenelse { \equal {#1} {} }  
    {\operatorname*{ess\,sup}} 
    {\operatorname*{\ensuremath{#1}-ess}\sup}
}
\newcommand{\essinf}[1][]{
\ifthenelse { \equal {#1} {} }  
    {\operatorname*{ess\,inf}} 
    {\operatorname*{\ensuremath{#1}-ess}\inf}
}
\newcommand{\essosc}[1][]{
\ifthenelse { \equal {#1} {} }  
    {\operatorname*{ess\,osc}} 
    {\operatorname*{\ensuremath{#1}-ess\,osc}}
}

\newcommand{\osc}{\operatorname*{osc}}

\renewcommand{\subset}{\subseteq}
\renewcommand{\supset}{\supseteq}

\newcommand{\olive}{\color{olive}}

\newcommand{\teal}{\color{teal}}

\newcommand{\nc}{\normalcolor}

\newcommand{\first}{\blue}
\newcommand{\second}{\magenta}
\newcommand{\third}{\orange}
\renewcommand{\first}{}
\renewcommand{\second}{}
\renewcommand{\third}{}

\newcommand{\editscolor}{\olive}
\renewcommand{\editscolor}{}

\graphicspath{{figures/}}

\newcommand{\off}[1]{{}}

\title{The Geometry of Adversarial Training \\ in Binary Classification}
\author{Leon Bungert}
\address{
Hausdorff Center for Mathematics, University of Bonn, Endenicher Allee 62, Villa Maria, 53115 Bonn, Germany.
}
\email{leon.bungert@hcm.uni-bonn.de}
\author{Nicol\'as {Garc\'ia Trillos}}
\address{Department of Statistics, University of Wisconsin-Madison, 1300 University Avenue, Madison, Wisconsin 53706, USA.}
\email{garciatrillo@wisc.edu}
\author{Ryan Murray}
\address{Department of Mathematics, 
North Carolina State University, 
2108 SAS Hall,
Raleigh, NC, 27695, USA.}
\email{rwmurray@ncsu.edu}

\date{\today}

\makeatletter
\let\blx@rerun@biber\relax
\makeatother

\begin{document}

\begin{abstract}
We establish an equivalence between a family of adversarial training problems for non-parametric binary classification and a family of regularized risk minimization problems where the regularizer is a nonlocal perimeter functional. 
The resulting regularized risk minimization problems admit exact convex relaxations of the type $L^1+\text{(nonlocal)}\TV$, a form frequently studied in image analysis and graph-based learning. 
A rich geometric structure is revealed by this reformulation which in turn allows us to establish a series of properties of optimal solutions of the original problem, including the existence of minimal and maximal solutions (interpreted in a suitable sense), and the existence of regular solutions (also interpreted in a suitable sense). In addition, we highlight how the connection between adversarial training and perimeter minimization problems provides a novel, directly interpretable, statistical motivation for a family of regularized risk minimization problems involving perimeter/total variation. 
The majority of our theoretical results are independent of the distance used to define adversarial attacks.
\end{abstract}

\maketitle

\tableofcontents

\section{Introduction}
In this paper we investigate the connection between adversarial training and regularized risk minimization in the context of non-parametric binary classification. \textit{Adversarial training} problems, in their {distributionally robust optimization} (DRO) version, can be written mathematically as min-max problems of the form:
\begin{align}
  \label{Robust problem:Intro}
  \inf\limits_{\theta \in \Theta } \sup_{\tilde \mu \st G(\mu,\tilde \mu)\leq  \eps } J(\tilde \mu,\theta),
\end{align}
where in general $\theta$ denotes the parameters of a statistical model (the parameters of a neural network, a binary classifier, the parameters of a linear statistical model, etc.), and $\mu$ denotes a data distribution to be fit by the model. 
To fully specify a DRO problem one also needs to introduce a notion of ``distance'' $G$ between data distributions that is employed to define a region of uncertainty around the original data distribution $\mu$ and that can be interpreted as the possible set of actions of an adversary who may perturb $\mu$. 
The value of $\eps\geq 0$ describes the ``power'' of the adversary and is often referred to as \emph{adversarial budget}. 
The function $J(\tilde \mu,\theta)$ is a risk relative to a data distribution $\tilde \mu$ and some loss function underlying the statistical model.
Problem \labelcref{Robust problem:Intro} is a transparent mathematical way to explicitly enforce robustness of models to data perturbations (at least of a certain type). 
Although the origins of this type of problem are now classical \cite{wald1945statistical}, recent influential research \cite{goodfellow2015} has shown that neural networks can be greatly improved by using the DRO framework, and as a result a renewed interest in this class of problems has been generated, see, e.g., the monograph \cite{OPT-026} and the paper \cite{madry2019deep}. In the context of the binary classification problem described in detail throughout this paper the works \cite{NashEquilbriaMeyer} and \cite{VarunMuni2} have explored the game theoretic interpretation of \labelcref{Robust problem:Intro} and the existence of Nash equilibria in parametric and non-parametric settings, respectively. \editscolor Other very recent works, e.g., \cite{jog2021reverse,pydi2019adversarial,MurrayNGT,awasthi2021existence_extended}, have expanded our theoretical understanding about adversarial training problems, providing results on existence of robust classifiers, and reformulating adversarial training problems in new ways that are amenable to further analysis and alternative computation schemes. \nc

By \textit{regularization}, on the other hand, we mean an optimization problem of the form 
\begin{equation}
 \inf_{\theta \in \Theta} \hat J(\mu,\theta) + \lambda\, R(\theta),
 \label{eqn:regularization}
\end{equation} 
where $\hat J(\mu,\theta)$ is a risk functional that here is taken with respect to the single data distribution~$\mu$, $R$ is the regularization functional, and $\lambda>0$ is a positive parameter describing the strength of regularization.
Regularization problems are fundamental in inverse problems \cite{benning_burger_2018,stuart_2010}, image analysis \cite{burger2013guide,Chambolle10anintroduction}, statistics \cite{Lasso}, and machine learning \cite{RKHSRegularization}; 
the previous list of references is of course non-exhaustive. 
In contrast to problem \labelcref{Robust problem:Intro}, the effect of explicit regularization on the robustness of models is less direct, but this is compensated by a richer structure that can be used to study the theoretical properties of their solutions more directly.

The connections between adversarial training and regularization have been intensely explored in recent years in the context of classical parametric learning settings; see~\cite{blanchet2019robust,Blanchet2,OPT-026} and references within. For example, when $\theta \in \Theta =\Rd$ represents the parameters of a linear regression model and the loss function for the model is the squared loss, the following identity holds
\begin{equation}
\label{eq:SrtLasso}
 \min_{\theta \in \Theta} \max_{G_p(\mu, \tilde \mu)\leq  \eps } \Exp{(x,y)\sim \tilde \mu}{(y - \langle \theta , x \rangle )^2} = \min_{\theta \in \Theta} \left\{ \sqrt{ \Exp{(x,y) \sim \mu}{(y- \langle \theta, x \rangle )^2}} + \sqrt{\eps} \: | \theta |_q \right\}^2,
\end{equation}
where $G_p$ is an optimal transport distance of the form
\[ 
G_p(\mu, \tilde \mu):= \min_{\pi \in \Gamma(\mu, \tilde \mu)} \iint_{\R^{d+1} \times \R^{d+1} } c_p((x,y), (\tilde x , \tilde y)) \de\pi((x,y), (\tilde x , \tilde y).
\] 
The cost function $c_p$ is defined by
\begin{align*}
    c_p\big((x,y),(\tilde x,\tilde y)\big) 
    := 
    \begin{cases} 
        | x-\tilde x|_p \quad &\text{ if } y = \tilde y, \\
        +\infty \quad &\text{ if } y \not = \tilde y, 
    \end{cases} 
\end{align*}
where $| \cdot |_p$ is the $\ell^p$ norm in $\R^d$ for $p$ satisfying $\frac{1}{p}+\frac{1}{q}=1$. In the definition of $G_p(\mu, \tilde \mu)$, the set $\Gamma(\mu,\tilde \mu)$ represents the set of transportation plans (a.k.a. couplings) between $\mu$ and $\tilde \mu$, namely, the set of probability measures on $\R^{d+1}\times \R^{d+1}$ with marginals given by $\mu$ and $\tilde \mu$. Notice that equation \labelcref{eq:SrtLasso} reveals a direct equivalence between a family of DRO problems \labelcref{Robust problem:Intro}, and a family of regularized risk minimization problems \labelcref{eqn:regularization} which includes the popular squared-root Lasso model from~\cite{SqRtLasso}. In particular, in this setting $|\cdot |_q$ becomes the regularization term $R$, the risk functional is $\hat J = \sqrt{J}$, \first where $J$ is the  mean squared error, \nc and $\lambda = \sqrt{\eps}$. 
Through an equivalence like \labelcref{eq:SrtLasso} it is possible to motivate new ways of calibrating regularization parameters in models with a convex loss function (where first order optimality conditions guarantee global optimality) as has been done in \cite{Blanchet2}. 
Beyond linear regression, the equivalence between adversarial training and regularization problems has also been studied in parametric binary classification settings such as logistic regression and SVMs (see \cite{Blanchet2}), as well as in distributionally robust grouped variable selection, and distributionally robust multi-output learning (see \cite{OPT-026}). 

In more general learning settings, it is often unknown whether there is a direct equivalence between \labelcref{Robust problem:Intro} and a problem of the form \labelcref{eqn:regularization} that is somewhat tractable both from a computational perspective as well as from a theoretical one. In such cases, an illuminating strategy that can be followed in order to gain insights into the regularization counterpart of \labelcref{Robust problem:Intro} is to analyze the max part of the problem for small $\eps$ and identify its leading order behavior to construct approximating regularization terms. This is a strategy that has been followed in many works that study the robust training of neural networks, e.g.,~\cite{ObermanFinlay,
GradientRegularization,
CurvatureRegularization,
SlavinRoss,
StructuredGradReg,
pmlr-v139-yeats21a,
CNGarciaTrillos2021,
bungert2021clip,
finlay2018improved}. 
The structure of the resulting approximate regularization problems can be exploited to motivate algorithms and provide a better theoretical understanding of the process of training robust deep learning models (see \cite{CNGarciaTrillos2021}).

Having discussed some of the literature exploring the connection between adversarial training and regularization, we move on to discussing, first in simple terms, the content of this work.
Through our theoretical results, this paper continues the investigation started in~\cite{MurrayNGT}, this time providing a deeper structural connection between adversarial training in the non-parametric binary classification setting and regularized risk minimization problems. In particular, we show that the equivalence between adversarial training and regularized risk minimization problems goes beyond the aforementioned parametric settings without relying on approximations. Here $\theta$ is substituted with $A$ which from now on will be interpreted as an arbitrary (measurable) subset of the data space $\domain$ (i.e., $A$ specifies a binary classifier), while $J$ is the risk associated to the $0$-$1$ loss; 
the other elements in problem \labelcref{Robust problem:Intro} will be specified in more detail in \cref{sec:setup}. 
We show that perimeter functionals penalizing the ``boundary'' of a set arise naturally as regularizers for binary classification problems regardless of the feature space or distance used to define the adversarial budget. 
This provides a more direct means of studying the evolution and regularity properties of minimizers of the adversarial problem than the ones that were implied by the evolution equations studied in~\cite{MurrayNGT}. 
This approach also provides tangible prospects for the design of new algorithms for the training of robust classifiers, and suggests which algorithms are more suitable for enforcing robustness relative to specific adversary's actions. \first Finally, through the connection between adversarial training and regularization we will deduce a variety of theoretical properties of robust classifiers, including the existence of ``regular'' solutions, where regularity is understood in a suitable technical sense. Regularity results like the ones we obtain in \cref{thm:regular-sol} are, to the best of our knowledge, the first of their kind in the context of adversarial training\nc.

In summary, our work reveals a rich geometric structure of adversarial training problems which is conceptually appealing and that at the same time opens up new avenues for the theoretical study of adversarial training for general binary classification models. In the next subsections we provide a more detailed discussion of our theoretical results and some of its conceptual consequences right after introducing the specific mathematical setup that we follow throughout the paper.



\subsection{Setup}
\label{sec:setup}

Let $(\domain,\metric)$ be a separable metric space representing the space of features of data points, and let $\B(\domain)$ be its associated Borel $\sigma$-algebra. In most applications $\domain$ is a finite dimensional vector space, e.g., $\Rd$ for some $d\in\N$, and later we will assume a certain, essentially finite dimensional, structure for some of our statements. We are given a probability measure $\mu\in\P(\domain\times\{0,1\})$ describing the distribution of training pairs $(x,y)\in\domain\times\{0,1\}$.
\editscolor Letting $\pi_1:\domain\times\{0,1\}$, $(x,y)\mapsto x$ be the projection onto the first factor of $\domain\times\{0,1\}$, \nc the first marginal of $\mu$ is denoted by $\rho:={\pi_1}_\sharp\mu\in\P(\domain)$ and represents the distribution of input data. \editscolor Here ${\pi_1}_\sharp\mu$ denotes the push-forward measure, whose definition we give in \cref{sec:appendix_technical_defs}. \nc
We decompose the data distribution as $\rho=w_0\rho_0 + w_1\rho_1$ where $w_i = \mu(\domain\times\{i\})$ and  $\rho_i\in\P(\domain)$ denote the conditional distributions:
\begin{align}
    \rho_i(A) := \frac{\mu\left(A\times\{i\}\right)}{w_i},\quad i\in\{0,1\},\;A\in \B(\domain).
\end{align}

Throughout this paper we make the assumption that all measures are Radon measures on~$\domain$.
In the following example we lay out two canonical situations which are highly relevant in machine learning.
\begin{example}[Absolutely continuous and empirical data distribution]
We let $\domain=\Rd$, equipped with an arbitrary $\ell^p$-metric for $p\in[1,\infty]$, i.e., $\metric(x_1,x_2):=\abs{x_1-x_2}_p$.
If we know the true distribution $\rho$ of the data, and this distributions is assumed to be absolutely continuous with respect to the Lebesgue measure, we can work with $\rho$ directly.
If we are only given a finite number of data points $\{x_i\}_{i=1}^N$, we can work with the empirical measure $\rho=\frac{1}{N}\sum_{i=1}^N \delta_{x_i}$. 
Both measures are Radon measures, and the main results of the paper apply in both settings. 
\end{example}
In binary classification, we seek a set $A\in\B(\domain)$ and its induced classifier:
\begin{align*}
    x\in A\phantom{^c} :\iff x\text{ is assigned label }1,\\
    x\in A^c :\iff x\text{ is assigned label }0.
\end{align*}
The most natural approach to constructing such a classifier is to minimize the \editscolor empirical risk:\nc
\begin{align}\label{eq:bayes_class}
    \inf_{A\in\B(\domain)} \Exp{(x,y)\sim\mu}{\abs{1_A(x)-y}}.
\end{align}
A minimizer of this problem is known as a Bayes classifier relative to $\mu$.

\begin{remark}[$0$-$1$-loss]
Note that introducing the $0$-$1$-loss function $\ell(\hat y,y)=0$ if $\hat y=y$ and $\ell(\hat y,y)=1$ if $\hat{y}\neq y$, one can equivalently express \labelcref{eq:bayes_class} as $\inf_{A\in\B(\domain)} \Exp{(x,y)\sim\mu}{\ell(1_A(x),y)}$.
\end{remark}
Applying the law of total expectation (or equivalently disintegrating the measure $\mu$) one obtains that \labelcref{eq:bayes_class} coincides with the following geometric problem
\begin{align}\label{eq:bayes_problem}
\begin{split}
    \inf_{A\in\B(\domain)}
    \int_A w_0\de\rho_0 + \int_{A^c} w_1\de\rho_1.
\end{split}    
\end{align}
Problem \labelcref{eq:bayes_problem} forces the set $A$ to be concentrated in places where the  measure~$w_1\rho_1$ is larger than $w_0\rho_0$.
Defining the signed measure $\sigma:=w_1\rho_1-w_0\rho_0$, one can take a Hahn decomposition of $\domain$ into $\domain=P\uplus N$, where $P$ is a positive set and $N$ is a negative set under $\sigma$ \first(see \cref{sec:appendix_technical_defs} for the definition of a Hahn decomposition)\nc. We then let $A:=P$, and deduce that such a set $A$ is a Bayes classifier, i.e., a minimizer of problem \labelcref{eq:bayes_problem}. Notably, the Hahn-decomposition is not unique and so neither is the Bayes classifier~$A$.
Furthermore, there is no control over the set where $w_0\rho_0=w_1\rho_1$. 
Those points might be arbitrarily assigned to either of the classes without affecting the objective functional; this is a potential source of non-robustness in classification. 

\begin{figure}[t]
    \centering
    \includegraphics[width=\textwidth]{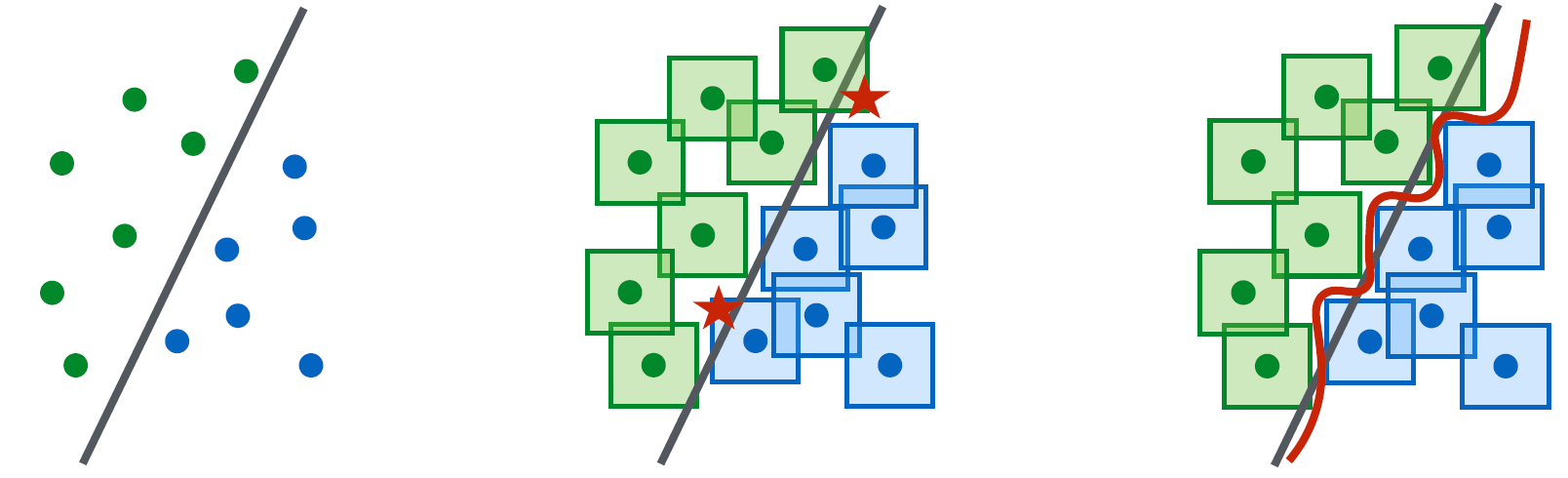}
    \caption{Picture taken from \cite{madry2019deep}. An affine \nc Bayes classifier (left) is not robust with respect to adversarial attacks (middle). 
    The robust classifier in red (right) has a smaller \emph{nonlocal} perimeter than the Bayes classifier. 
    In this case the $\ell^\infty$ adversarial attacks play an important role in the geometry of the robust classifier, and will define a particular form of nonlocal perimeter regularization. \nc}
    \label{fig:my_label}
\end{figure}
Throughout the paper we focus our attention on the following adversarial training problem for robust binary classification:
\begin{align}\label{eq:baseline_adv}
    \inf_{A\in\B(\domain)}\Exp{(x,y)\sim\mu}{\sup_{\tilde x\in B_\eps(x)}\abs{1_A(\tilde x)-y}}.
\end{align}
The model allows an adversary to choose the worst possible point in an open $\eps$-ball $B_\eps(x):=\{\tilde x\in\domain\st\metric(x,\tilde x)<\eps\}$ (relative to the metric $\de$) around $x$ to corrupt the classification.
We emphasize that we \emph{do not} use the essential supremum with respect to some measure but rather the actual supremum which potentially makes the adversarial attack much stronger. 
However, under mild assumptions on the space $\domain$ and the measure $\rho$, it is possible to draw a connection between \labelcref{eq:baseline_adv} and the following problem
\begin{align}\label{eq:esssup_adv}
    \inf_{A\in\B(\domain)}\Exp{(x,y)\sim\mu}{\esssup[\nu]_{\tilde x\in B_\eps(x)}\abs{1_A(\tilde x)-y}},
\end{align}
where $\nu$ is a suitably chosen reference measure.
This problem has favorable functional analytic properties and we will use it as intermediate step to construct solutions of the original problem~\labelcref{eq:baseline_adv}, as well as to analyze the structure of the set of solutions of \labelcref{eq:baseline_adv}. 


Before proceeding to an informal presentation of our main results, we emphasize that, in contrast to some papers in the literature, here we consider \textit{open} balls $B_\eps(x)$ to describe the set of possible attacks available to the adversary around the point $x$. By making this modelling choice we can simplify some technical steps in our analysis (e.g., see \cref{rem:OnOpenvsClosedBalls,sec:open_vs_closed_balls}) and avoid measurability issues that may arise when working with closed balls (see \cite{VarunMuni2} for a discussion on the measurability issue and contrast it with \cref{rem:Measurability}).


\nc


\subsection{Informal Main Results and Discussion}
\label{sec:main_results}

Our first main result, at this stage stated informally, is a reformulation of the adversarial training problem \labelcref{eq:baseline_adv} in terms of a variational regularization problem:
\begin{theorem*}
The objective in the adversarial training problem \labelcref{eq:baseline_adv} can be rewritten as
\begin{align}\label{eq:objective_Per}
    \Exp{(x,y)\sim\mu}{\sup_{\tilde x\in B_\eps(x)}\abs{1_A(\tilde x)-y}}
    =
    \Exp{(x,y)\sim\mu}{\abs{1_A(x)-y}} + \eps \PrePer_\eps(A;\mu),
\end{align}
where $\PrePer_\eps(A;\mu)$ is a \emph{nonlocal} and \emph{weighted} perimeter of $A$, \first defined as\nc 
\begin{align}\label{eq:pre_Per}
    \PrePer_\eps(A;\mu) =
    \frac{w_0}{\eps}\rho_0({\teal\{x\in A^c\st\dist(x,A)<\eps\}})+
    \frac{w_1}{\eps}\rho_1({\olive\{x\in A \st  \dist(x,A^c)<\eps\}}),
\end{align}
\second see \cref{fig:perimeter} for a color-coded illustration.
\end{theorem*}
\begin{figure}[htb]
    \centering
    \begin{tikzpicture}
    \node at (0,0) {\includegraphics[width=0.7\textwidth]{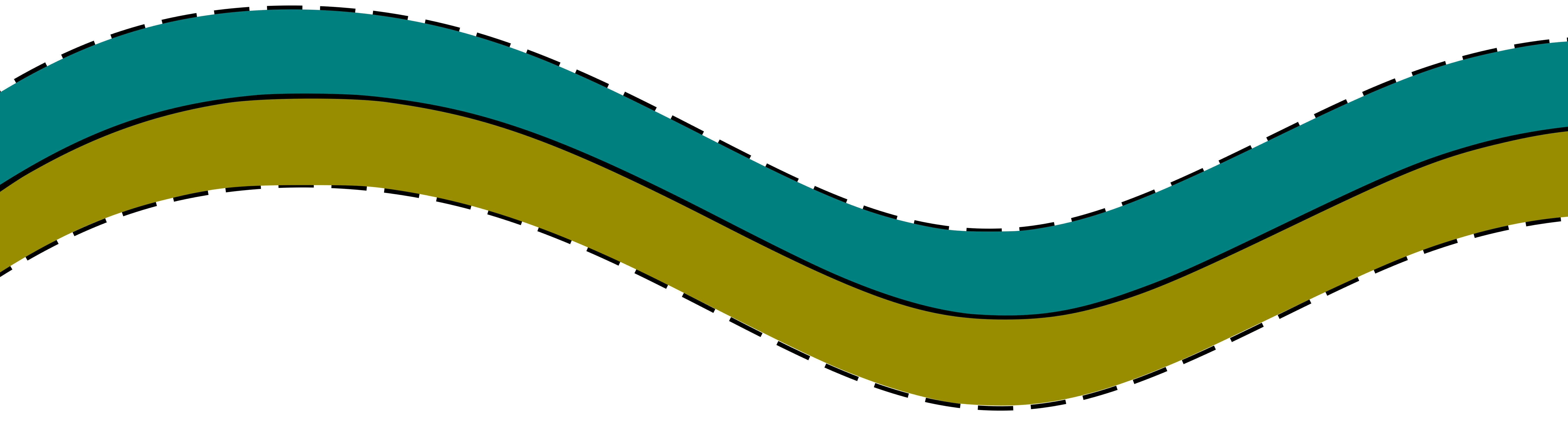}};
    \node at (-2.2,-0.5) {\nc$A\triangleq\text{class 1}$};
    \node at (1.2,0.5) {\nc$A^c\triangleq\text{class 0}$};
    \node at (-6,0.2) {\nc$\partial A$};
    \end{tikzpicture}
    \caption{\second Illustration of the ``perimeter'' defined in \labelcref{eq:pre_Per}.
    The blue strip outside $A$ is measured with $\rho_0$. The olive strip inside $A$ is measured with $\rho_1$.
    The sum of these two quantities being small means that the indicator of $A$ is an adversarially robust classifier.
    \nc}
    \label{fig:perimeter}
\end{figure}
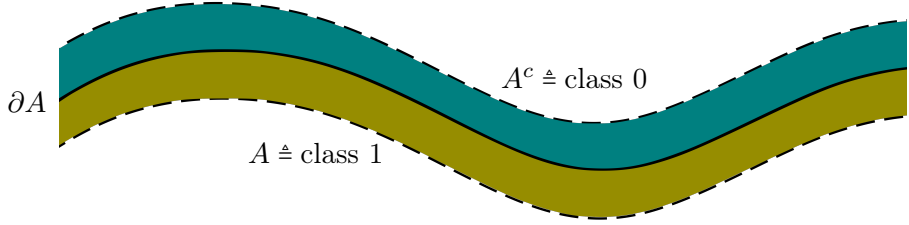
The functional $\PrePer_\eps$ \editscolor can be called a type of  ``perimeter'' since \nc it is a non-negative functional over sets with the important submodularity property:
\begin{align*}
    \PrePer_\eps(A\cup B;\mu) +
    \PrePer_\eps(A\cap B;\mu) 
    \leq 
    \PrePer_\eps(\first A\nc;\mu)
    +
    \PrePer_\eps(B;\mu), \quad \forall A, B \in \B(\domain).
\end{align*}
Submodular functionals over sets typically induce convex functionals over functions (referred to as a total variation), defined through the coarea formula:
\[ \PreTV_\eps(u;\mu) := \int_{-\infty}^\infty \PrePer_\eps(\{ u \geq t \};\mu) \de t. \]
\first 
We will show that the so defined total variation takes the form
\begin{align}\label{eq:pre_TV}
    \PreTV_\eps(u;\mu)
    =
    \frac{w_0}{\eps} \int_\domain \sup_{\tilde x\in B_\eps(x)}u(\tilde x) - u(x) \de\rho_0(x) + 
    \frac{w_1}{\eps} \int_\domain u(x) - \inf_{\tilde x\in B_\eps(x)}u(\tilde x) \de\rho_1(x).
\end{align}
\nc 
By our notation we emphasize that both the perimeter and the total variation depend on the data distribution $\mu$ through $w_i, \varrho_i$ and not just through $\varrho$, as will be detailed in the course of the paper. 
Hence, as opposed to standard (nonlocal) perimeters and total variations, they constitute a family of data-driven regularizers.
Such regularizers, typically learned in a supervised manner, have recently been shown to be superior over model-based regularizers for certain tasks in medical imaging, see \cite{mukherjee2020learned}.

It turns out that using the $\widetilde{\TV}_\eps$ functional we can define an \textit{exact} convex relaxation for the problem \labelcref{eq:baseline_adv}:
\begin{theorem*}
The variational problem:
\begin{align}\label{eq:convex_relaxation}
    \inf_{u: \domain \rightarrow [0,1]}\Exp{(x,y)\sim\mu}{\abs{u(x)-y}} + \eps \PreTV_\eps(u;\mu),
\end{align}
is an exact convex relaxation of problem \labelcref{eq:baseline_adv}. 
In particular, any solution to problem \labelcref{eq:baseline_adv} is also a solution to \labelcref{eq:convex_relaxation}, and conversely, for any solution $u$ of problem \labelcref{eq:convex_relaxation} we can obtain a solution to problem \labelcref{eq:baseline_adv} by considering level sets of $u$.
\end{theorem*}

We move on to study the existence of solutions to problem \labelcref{eq:baseline_adv}.
\begin{theorem*}[informal]
Under some technical conditions on the metric space $\domain$ and the measure $\rho$, the adversarial training problem \labelcref{eq:baseline_adv} admits a solution $A\in\B(\domain)$.
\end{theorem*}
Our existence proof is technical and is based on the lifting of the variational problem \labelcref{eq:baseline_adv} to a problem of the form \labelcref{eq:esssup_adv}.
This problem admits an application of the direct method of the calculus of variations after establishing lower semicontinuity and compactness in a suitable weak-* Banach space topology, \editscolor see \cref{sec:appendix_technical_defs} for a definition of the weak-* topology. \nc
In the course of this, we will introduce well-defined versions of $\PrePer_\eps$ and $\PreTV_\eps$, which will not carry the tilde anymore, and study their associated variational problems. We discuss how to build solutions to the original problem \labelcref{eq:baseline_adv} from solutions to the modified problems. 

\third 
\begin{remark}[Relation to previous results]\label{rem:previous}
Existence of solutions to other adversarial training problems has also been obtained recently in the work \cite{awasthi2021existence_extended}. The existence results in \cite{awasthi2021existence_extended} and ours are highly complementary to each other and in what follows we highlight the differences in their settings, which are apparent in at least three ways: First, the adversarial model in \cite{awasthi2021existence_extended} is defined in terms of closed balls rather than open balls as done here; existence of solutions in the open ball model was left as an open question.
Second, the collection of subsets $A$ of $\mathcal{X}$ over which the optimization takes place in \cite{awasthi2021existence_extended} is the so called universal $\sigma$-algebra, which is larger than the Borel $\sigma$-algebra considered here. For the adversarial model with closed balls it is essential to use the universal $\sigma$-algebra in order to assure the measurability of the adversarial loss function, which we get for free in our open ball model. Naturally, lower semicontinuity is less of an issue in \cite{awasthi2021existence_extended}, whereas in our proof we rely on relaxation methods and on the explicit construction of representatives. Lastly, we highlight that our setting is very general since we work on a metric measure space whereas the results in \cite{awasthi2021existence_extended} hold in the setting of norm balls in Euclidean space.
\end{remark}
\nc

It is also worth highlighting that objective functions of type $L^1+\TV$ and their relation to perimeter regularization have been extensively studied in the mathematical imaging community \cite{chan2005aspects,darbon2005total,duval2009tvl1,zeune2017combining,Chambolle10anintroduction}. Further background on total variation methods in imaging is provided in \cite{burger2013guide,Chambolle10anintroduction}.

After establishing existence of solutions to \labelcref{eq:baseline_adv} we proceed to studying their properties. In particular, we exploit  the underlying convexity made manifest by our theorems and deduce a series of strong implications on the geometry and regularity of the family of solutions to the adversarial training problem \labelcref{eq:baseline_adv}. As a first step, we prove that solutions are closed under intersections and unions. From this we will be able to prove the following:
\begin{theorem*}[informal]
 There exist (unique) minimal and maximal solutions to \labelcref{eq:baseline_adv} in the sense of set inclusion.
\end{theorem*}
It is then possible to show that maximal and minimal solutions satisfy, respectively, inner and outer regularity conditions (in a suitable sense discussed in detail throughout the paper), providing in this way the first results on regularity of (certain) solutions to \labelcref{eq:baseline_adv}. We investigate the regularity of solutions further and establish \editscolor Hölder regularity \nc results like the following \first (see \cref{sec:appendix_technical_defs} for definitions)\nc:
\begin{theorem*}[informal]
Let $(\domain, \metric)$ be $\R^d$ with the Euclidean distance. 
For any $\eps>0$ there exists a solution to the problem \labelcref{eq:baseline_adv} \editscolor whose boundary \nc is locally the graph of a $C^{1,1/3}$ function\nc. 
\end{theorem*}
Although stated for the Euclidean setting only, we highlight that similar results can be proved in more general settings provided that one adjusts the interpretation of regularity of solutions to non-Euclidean contexts. 
A more detailed investigation of this will be the topic of follow-up work. 
It is also important to reiterate that our results apply to a general measure $\mu$ regardless of whether it has densities with respect to Lebesgue measure or if it is an empirical measure. 
In particular, from our results we can conclude that the presence of the adversary always enforces regularization of decision boundaries, even when the original unrobust problem does not possess regular solutions (i.e., when the Bayes classifiers are not regular). 
We remark that we do not claim any sharpness in our regularity results. However, in general one should not expect better regularity than $C^{1,1}$ (in the Euclidean setting) based on the discussion that we present in \cref{sec:RegularitySolutions} and on the results from \cite{lewicka2020domains}.

\first 

\first
Finally, we remark that our results suggest that one should use algorithms for adversarial training that are based on training parametric models that are able to produce or approximate regular classifiers. Some examples of these models are suggested by recent results in the literature of approximation theory; these results state that it is possible to approximate characteristic functions of regular sets with neural networks whose size is determined by the level of regularity of the target set, see \cite{petersen2018optimal}. We believe that our regularity results can indeed inform new implementations of adversarial training, but there are still several points to be resolved before being able to carry out an actual algorithmic implementation. Moreover, since the notion of regularity depends on the distance function used to define the action space of the adversary, one should naturally adapt algorithms to produce robust classifiers of the specified type. The above discussion will be expanded in future work. 
\nc

In addition to the adversarial model \labelcref{eq:baseline_adv}, we discuss other adversarial models that admit a representation of the form $L^1+\text{(nonlocal)}\TV$. 
From this we will be able to conclude that perimeter functionals penalizing the boundaries of sets indeed arise naturally as regularizers for binary classification problems.
This fact can be interpreted conversely: it is possible to give a game theoretic interpretation for a class of variational problems that involve the use of (nonlocal) total variation (including those that have been used in graph-based learning for classification \cite{garciatrillos_murray_2017}).
This work can then be naturally related to a collection of works that provide game theoretical interpretations of variational problems. 
For example, \cite{caffarelli2012non} and \cite{peres2008tug} connect fractional Dirichlet energies with a two-player game.
Moreover, \cite{kohn2006deterministic} connects mean curvature flow with a different two-player game. 
While our energies do not directly coincide with the ones in those papers, they are similar in form.

\subsection{Outline}
The rest of the paper is organized as follows. 
In \cref{sec:Reformulations} we discuss different reformulations of the adversarial training problem \labelcref{eq:baseline_adv}. 
First, in \cref{sec:relaxation} we \editscolor relax \nc the problem in a suitable way in order to make it amenable to functional analytic treatment; this reformulation will be crucial for our latter exploration on existence of solutions to \labelcref{eq:baseline_adv} and the study of some of their properties. 
\nc
In \cref{sec:RegularizationReformulation} we discuss the reformulation of \labelcref{eq:baseline_adv} as the regularized risk minimization that has already been introduced in \cref{sec:main_results}, cf. \labelcref{eq:objective_Per}.

\cref{sec:analysis} is devoted to the study of properties of the regularization reformulation of \labelcref{eq:baseline_adv}. 
We define suitable relaxations of the functionals $\PrePer_\eps$ and $\PreTV_\eps$ appearing in \labelcref{eq:objective_Per} and establish key properties including submodularity, convexity, and lower semi-continuity with respect to suitable topologies. 
With these properties at hand we show existence of solutions to problem \labelcref{eq:baseline_adv} in \cref{sec:ExistenceSolutions}. 
\first 
In \cref{sec:ExtremalSolutions} we study maximal and minimal solutions and in \cref{sec:RegularitySolutions} we investigate regularity.
\nc

In \cref{sec:OtherModels} we explain how to generalize our insights to regression tasks and other adversarial training models that give rise to perimeter minimization problems with different perimeter functionals. In particular, we recover data-driven regularizers as well as statistically robust interpretations to regularization approaches used in graph-based learning. 

We wrap up the paper in \cref{sec:Conclusions} where we present further discussion on the implications of our work and provide some directions for future research.

\first
Technical definitions, some proofs, and further remarks on the advantage of using open balls are given in the appendix.
\nc

\section{Reformulations of Adversarial Training}
\label{sec:Reformulations}

\subsection{Relaxation in Quotient \texorpdfstring{$\sigma$}{sigma}-Algebra}
\label{sec:relaxation}

To be able to prove existence of minimizers for \labelcref{eq:baseline_adv} we have to relax it to make it amenable to functional analytic treatment.
Note that, because of the presence of the non-essential supremum in \labelcref{eq:baseline_adv}, two sets $A$ and $A'$ whose symmetric difference
\begin{align}
    A\triangle A' := (A\setminus A') \cup (A'\setminus A)
\end{align}
meets $\nu(A\triangle A')=0$ for some reference measure $\nu$ do not have to have the same value of the objective function, in general.
This is a major difference to unregularized problem \labelcref{eq:bayes_class} and will cause problems, for instance, when proving existence of minimizers.

To fix this  we define the set
\begin{align}
    \mathfrak N_\nu :=\{A\in\B(\domain) \st \nu(A)=0\}
\end{align}
where $\nu$ is an arbitrary reference measure on $\domain$, to be specified later.
The set $\mathfrak N_\nu$ is a two-sided ideal in the $\sigma$-algebra $\B(\domain)$, interpreted as ring with addition $\triangle$ and multiplication $\cap$.
This allows us to define the quotient $\sigma$-algebra
\begin{align}
    \B_\nu(\domain) := {\B(\domain)}\Big/{\mathfrak N_\nu}
\end{align}
with the equivalence relation $\sim_\nu$, defined by
\begin{align}\label{eq:equiv_rel}
A\sim_\nu B:\iff A\triangle B\in\mathfrak N_\nu \iff \nu(A\triangle B)=0.    
\end{align}
\editscolor
The function $d_\nu(A,B) := \nu(A \triangle B)$ is non-negative, symmetric and sub-additive, and hence defines a pseudo-metric on $\B(\domain)$. This function is also zero if and only if $A\sim_\nu B$, and hence it is a metric on the quotient $\sigma$-algebra $\B_\nu(\domain)$. In some sources this metric is called the Fr\'echet--Nikod\'ym pseudo-metric, see e.g. Section 1.12 in \cite{bogachev2007measure}.
\nc
The following proposition states that the minimization in \labelcref{eq:baseline_adv} can be rewritten as the minimization of some sort of quotient norm on the quotient $\sigma$-algebra $\B_\nu(\domain)$.
Interestingly, the choice of $\nu$ does not yet matter here.
\begin{proposition}\label{prop:quotient_algebra}
For any Borel measure $\nu$ on $\domain$ it holds that
\begin{align}\label{eq:adv_relaxed}
    \begin{split}
    \labelcref{eq:baseline_adv}
    =
    \inf_{A\in\B(\domain)}\inf_{\substack{B\in\B(\domain)\\A\sim_\nu B}}\Exp{(x,y)\sim\mu}{\sup_{\tilde x\in B_\eps(x)}\abs{1_B(\tilde x)-y}}.
    \end{split}
\end{align}
\end{proposition}
\begin{remark}[Similarity to quotient norms]
The reason why we connect this reformulation with the quotient $\sigma$-algebra is that the objective function in \labelcref{eq:adv_relaxed} has strong similarities with the quotient norm on a quotient Banach space $X/N$, which is given by
\begin{align*}
    \norm{x}_{X/N}:=\inf_{\substack{y\in X\\y-x\in N}}\norm{y}_X,\quad x\in X/N.
\end{align*}
\end{remark}
\begin{proof}
We have to prove the equality
\begin{align*}
    &\inf_{A\in\B(\domain)}\Exp{(x,y)\sim\mu}{\sup_{\tilde x\in B_\eps(x)}\abs{1_A(\tilde x)-y}}
    =
    \inf_{A\in\B(\domain)}\inf_{\substack{B\in\B(\domain)\\A\sim_\nu B}}\Exp{(x,y)\sim\mu}{\sup_{\tilde x\in B_\eps(x)}\abs{1_B(\tilde x)-y}}.
\end{align*}
First, choosing $B=A$, which obviously fulfills $A\sim_\nu B$, we obtain the inequality $\geq$.
Second, omitting the constraint $A\sim_\nu B$ yields the inequality $\leq$.
\end{proof}

\begin{remark}
\label{rem:Measurability}
In the definition of the adversarial problem \labelcref{eq:baseline_adv} and throughout the rest of the paper we will be working with quantities like $\sup_{\tilde x \in B_\eps(x)} 1_A$ for a Borel measurable set $A$ and $\sup_{\tilde x \in B_\eps(x)} u$ for a Borel measurable function $u$. We remark that the resulting sets/functions are Borel measurable. Indeed, the function $x \mapsto \sup_{\tilde x \in B_\eps(x)} 1_A$ is nothing but the indicator function of the set $\bigcup_{x \in A} B_\eps(x)$ which is Borel measurable since it is an open set. Likewise, the function $x \mapsto \tilde u (x) := \sup_{\tilde x \in B_\eps(x)} u$ is measurable because the sets $\{ \tilde u >t \}$ are open sets. 
\end{remark}
\nc
An alternative way of avoiding ambiguities arising from equivalent sets with respect to $\nu$ is to consider the essential version of the adversarial problem given by \labelcref{eq:esssup_adv}, where the adversarial attack is performed using the essential supremum of the measure $\nu$. This problem can be fundamentally different to our problem \labelcref{eq:baseline_adv} or the relaxed one \labelcref{eq:adv_relaxed} since for example, in the case $\nu=\rho$, the attack can only be performed within the support of the given data distribution which is much weaker than \labelcref{eq:baseline_adv}. Still, in \cref{sec:analysis} we shall construct a measure $\nu$ such that the problems \labelcref{eq:baseline_adv,eq:esssup_adv} \emph{do} coincide, a property we will exploit later for proving existence of solutions to the original adversarial problem. 
\begin{example}
\label{ex:Example1}
Consider the simple situation with the measure $\rho = \frac{1}{2}\delta_{-1}+\frac{1}{2}\delta_{1}$ on $\domain=\R$ and $\nu=\rho$.
The labels are set to be equal to zero on the left axis and one on the right one.
Then it holds
\begin{align*}
    \Exp{(x,y)\sim\mu}{\sup_{(x-\eps,x+\eps)}\abs{1_A(x)-y}} 
    &= \frac{1}{2}\sup_{(-1-\eps,-1+\eps)}1_A+\frac{1}{2}\sup_{(1-\eps,1+\eps)}1_{A^c}.
\end{align*}
Let us assume that $1<\eps<2$.
In this case the intervals $(-1-\eps,-1+\eps)$ and $(1-\eps,1+\eps)$ overlap. Therefore, for any choice of $A\in\B(\R)$, either $A$ or $A^c$ intersect both intervals.
This implies that the optimal adversarial risk is $\geq\frac{1}{2}$.
Furthermore, choosing $A=\{1\}$ we find the risk equals $\frac{1}{2}$.

For comparison, the objective of the quotient problem \labelcref{eq:adv_relaxed} is given by
\begin{align*}
    \inf_{\substack{B\in\B(\domain)\\\rho(A\triangle B)=0}}\Exp{(x,y)\sim\mu}{\sup_{(x-\eps,x+\eps)}\abs{1_B-y}}
    =
    \inf_{\substack{B\in\B(\domain)\\\rho(A\triangle B)=0}}
    \frac{1}{2}\sup_{(-1-\eps,-1+\eps)}1_B+\frac{1}{2}\sup_{(1-\eps,1+\eps)}1_{B^c}
\end{align*}
and, arguing as before, any choice of $B$ leads to this term being $\geq\frac{1}{2}$ independently of $A$.


On the other hand, the objective in \labelcref{eq:esssup_adv} for $\nu:=\rho$ and $0<\eps<2$ is
\begin{align*}
    \Exp{(x,y)\sim\mu}{\esssup[\rho]_{B_\eps(x)}\abs{1_A(x)-y}} 
    &= \frac{1}{2}1_A(-1)+\frac{1}{2}1_{A^c}(1)
\end{align*}
which does not even depend on $\eps$.
This is due to the fact that the $\esssup[\rho]$ prevents the adversary from leaving the set of data points.
For instance, the half axes ${x\geq\alpha}$ with $-1<\alpha<1$ have risk $0$ and thus are optimal.
\end{example}

\subsection{Nonlocal Variational Regularization Problem}
\label{sec:RegularizationReformulation}

We now show how to express the adversarial training problem \labelcref{eq:baseline_adv} as a variational regularization problem in the form of \labelcref{eqn:regularization}.
More precisely, we show that it can be written as $L^1 + \TV$-type problem.
This class of problems has been intensively studied in the context of image processing, following the seminal paper~\cite{chan2005aspects}.
The model there was related to a geometric problem involving the Lebesgue measure $\mathcal{L}^d(\cdot)$ and the standard perimeter functional $\Per(\cdot)$, namely
\begin{align}\label{eq:geometric_problem}
    \min_{A\in\B(\domain)} \mathcal{L}^d(A\triangle\Omega) + \lambda \Per(A).
\end{align}
This functional was shown to exhibit a range of different behaviors in terms of the regularization parameter~$\lambda$.

In our context, we will show that the adversarial problem \labelcref{eq:baseline_adv} can be interpreted analogously, with the modification that we use a weighted volume and a weighted and nonlocal perimeter, see \cref{rem:geometric_problem} below.
Let us therefore first introduce the set function $\PrePer_\eps(\cdot;\mu) : \B(\domain) \to [0,+\infty]$ for $\eps>0$ which we refer to as nonlocal pre-perimeter and which is defined as
\begin{align}\label{eq:pre-perimeter}
    \PrePer_\eps(A;\mu) 
    :=
    \frac{w_0}{\eps}\int_\domain  \sup_{\tilde x\in B_\eps(x)}1_A(\tilde x)-1_A(x) \de\rho_0(x)
    + \frac{w_1}{\eps}\int_\domain  1_{A}(x)-\inf_{\tilde x\in B_\eps(x)}1_{A}(\tilde x) \de\rho_1(x).
\end{align}
Here the dependency on the data distribution $\mu$ is captured by the presence of the conditional distributions $\rho_i$ and the class probabilities $w_i$ for $i\in\{0,1\}$. Here the tilde serves as a reminder that we are using supremum and infimum as opposed to their $\nu$-essential forms.
To see that $\PrePer_\eps(A;\mu)$ has units of a perimeter we rewrite it as follows
\begin{align}\label{eq:perimeter_geometric}
    \PrePer_\eps(A;\mu) =
    \frac{w_0}{\eps}\rho_0({\{x\in A^c\st\dist(x,A)<\eps\}})+
    \frac{w_1}{\eps}\rho_1({\{x\in \dist(x,A^c)<\eps\}}),
\end{align}
where the distance of a point $x\in\domain$ to a set $A\subseteq\domain$ is defined as $\dist(x,A) := \inf_{\tilde x\in A}\metric(x,\tilde x)$.
The quantity \labelcref{eq:perimeter_geometric} is a weighted and nonlocal Minkowski content \cite{cesaroni2017isoperimetric,cesaroni2018minimizers} of the ``thickened boundary'' $\partial^\eps A:=\{x\in\domain\st\dist(x,\partial A)<\eps\}$, cf. \cref{fig:perimeter} in \cref{sec:main_results}.
For sufficiently smooth sets and measures $\rho_{0/1}$, and for small $\eps$ one expects \cite{ambrosio2017perimeter} that $\PrePer_\eps(A;\mu)$ behaves like a weighted perimeter of $A$, see \cite{chambolle2010continuous} for similar results. 

Importantly, for two sets $A,B\in\B(\domain)$ which differ only by a nullset with respect to some reference measure $\nu$ the associated pre-perimeters will generally be different.
Therefore, using the technique from \cref{sec:relaxation}, we define the nonlocal perimeter with respect to $\nu$ as
\begin{align}\label{eq:perimeter}
\begin{split}
    \Per[\nu]_\eps(A;\mu) 
    :=\inf_{\substack{B\in\B(\domain)\\A\sim_\nu B}}
    \PrePer_\eps(B;\mu).
\end{split}    
\end{align}
This way of defining a nonlocal and weighted perimeter generalizes approaches from \cite{cesaroni2017isoperimetric,cesaroni2018minimizers}, which deal with the case of the Lebesgue measure.

\begin{remark}
The nonlocal perimeter $\Per[\nu]_\eps(\cdot;\mu)$ in \labelcref{eq:perimeter} is a generalization of the nonlocal perimeter studied in~\cite{chambolle2012nonlocal,chambolle2015nonlocal,cesaroni2017isoperimetric,cesaroni2018minimizers} which can be recovered by setting $\domain=\Rd$ and \first by replacing $w_0\rho_0$ and $w_1\rho_1$ by the Lebesgue measure $\mathcal{L}^d$ and choosing $\nu:=\mathcal{L}^d$\nc.
Our results from \cref{sec:analysis}, in particular \cref{prop:representative_Per}, show that \labelcref{eq:perimeter} becomes
\begin{align}\label{eq:std_perimeter}
    \Per_\eps(A) := \frac{1}{2\eps}\int_\Rd\essosc_{B_\eps(\cdot)}(1_A)\de x,
\end{align}
where $\essosc=\esssup-\essinf$ is the essential oscillation with respect to the Lebesgue measure.
\end{remark}
\begin{remark}[Asymmetry]
It is obvious that the nonlocal perimeter \labelcref{eq:std_perimeter} from~\cite{chambolle2015nonlocal} satisfies $\Per_\eps(A^c)=\Per_\eps(A)$ and the same is true for the usual local perimeter.
For our perimeter \labelcref{eq:perimeter} this is not the case if $w_0\rho_0\neq w_1\rho_1$.
\end{remark}
Let us now reformulate the adversarial training problem as a regularization problem with respect to the nonlocal perimeter \labelcref{eq:perimeter}.
Our central observation is that the adversarial risk in \labelcref{eq:baseline_adv} can be decomposed into an unregularized risk and the pre-perimeter.
Then, using \cref{prop:quotient_algebra}, we will rewrite \labelcref{eq:baseline_adv} as a variational regularization problem for the perimeter.
\begin{proposition}\label{prop:adv_risk_per}
For any Borel set $B\in\B(\domain)$ it holds
\begin{align}\label{eq:adv_risk_per}
    \Exp{(x,y)\sim\mu}{\sup_{\tilde x\in B_\eps(x)}\abs{1_B(\tilde x) - y}}
    =
    \Exp{(x,y)\sim\mu}{\abs{1_B(x)- y}}
    + \eps \PrePer_\eps(B;\mu).
\end{align}
\end{proposition}
\begin{proof}
Disintegrating $\mu$ and doing elementary calculations yields
\begin{align*}
    &\phantom{=}
    \Exp{(x,y)\sim\mu}{\sup_{\tilde x\in B_\eps(x)} \abs{1_B(\tilde x)-y}}
    -
    \Exp{(x,y)\sim\mu}{\abs{1_B(x)-y}}
    \\
    &=
    \iint_{\domain\times\{0,1\}}{\sup_{\tilde x\in B_\eps(x)}\abs{1_B(\tilde x)-y}}\de\mu(x,y)
    -
    \iint_{\domain\times\{0,1\}}\abs{1_B(x)-y}\de\mu(x,y)
    \\
    &=
    w_0\int_\domain\sup_{B_\eps(\cdot)}1_B\de\rho_0 
    +
    w_1\int_\domain\sup_{B_\eps(\cdot)}1_{B^c}\de\rho_1
    -
    w_0\int_\domain 1_B\de\rho_0 
    -
    w_1\int_\domain 1_{B^c}\de\rho_1
    \\
    &=
    w_0\int_\domain \sup_{B_\eps(\cdot)}1_B - 1_B\de\rho_0
    +
    w_1\int_\domain 1_{B} - \inf_{B_\eps(\cdot)}1_{B}\de\rho_1
    \\
    &=
    \eps\PrePer_\eps(B;\mu).
\end{align*}    
\end{proof}
Now we can finally state the equivalence of the adversarial training problem \labelcref{eq:baseline_adv} and the variational regularization problem involving the nonlocal perimeter $\Per[\nu]_\eps(\cdot;\mu)$.
For this we have to choose the measure $\nu$ in the definition of the perimeter \labelcref{eq:perimeter} such that $\rho$ is absolutely continuous with respect to the reference measure $\nu$, written $\rho\ll\nu$.
\begin{proposition}[Perimeter-regularized problem]\label{prop:perimeter_problem}
Let $\nu$ be a measure on $\domain$ such that $\rho\ll\nu$. 
Then it holds that
\begin{align}\label{eq:adv_perimeter}
    \labelcref{eq:baseline_adv}
    =
    \inf_{A\in\B(\domain)}
    \Exp{(x,y)\sim\mu}{\abs{1_A(x)-y}} + \eps\Per[\nu]_\eps(A;\mu).
\end{align}
\end{proposition}
\begin{remark}[Geometric problem]\label{rem:geometric_problem}
Note that if the measures $\rho_0$ and $\rho_1$ have non-overlapping support, \labelcref{eq:adv_perimeter} can indeed be brought into the form of the geometric problem \labelcref{eq:geometric_problem} which generalizes the problem studied in \cite{chan2005aspects}.
For this we \second assume that there exists $\Omega\subseteq\domain$ such that $\supp\rho_1\subseteq\Omega\subseteq(\supp\rho_0)^c$.
Then \nc the first term in \labelcref{eq:adv_perimeter} equals
\begin{align*}
    \Exp{(x,y)\sim\mu}{\abs{1_A(x)-y}}
    &=
    w_0\int_\domain 1_A\de\rho_0 + w_1\int_\domain 1_{A^c}\de\rho_1 
    = w_0\rho_0(A)+w_1\rho_1(A^c)
    \\
    &=w_0\rho_0(A\cap\Omega^c)+w_1\rho_1(A^c\cap\Omega)
    = 
    w_0\rho_0(A\setminus\Omega) + w_1\rho_1(\Omega\setminus A)
    \\
    &= 
    \rho(A\setminus\Omega) + \rho(\Omega\setminus A) - w_1\rho_1(A\setminus\Omega) - w_0\rho_0(\Omega\setminus A)
    \\
    &= 
    \rho(A\setminus\Omega) + \rho(\Omega\setminus A) - w_1\rho_1(\Omega^c\setminus A^c) - w_0\rho_0(\Omega\setminus A)
    \\
    &=
    \rho((A\setminus\Omega) \cup (\Omega\setminus A))
    =
    \rho(A\triangle \Omega).
\end{align*}
This implies that \labelcref{eq:adv_perimeter} equals the geometric problem
\begin{align}
    \inf_{A\in\B(\domain)}\rho(A\triangle \Omega) + \eps\Per[\nu]_\eps(A;\mu).
\end{align}
\end{remark}
\begin{proof}
Fixing $A\in\B(\domain)$ and taking the infimum over sets $B\in\B(\domain)$ with $A\sim_\nu B$ we get from \cref{prop:adv_risk_per} that
\begin{align*}
    \inf_{\substack{B\in\B(\domain)\\A\sim_\nu B}}\Exp{(x,y)\sim\mu}{\sup_{\tilde x\in B_\eps(x)} \abs{1_B(\tilde x)-y}}
    =
    \inf_{\substack{B\in\B(\domain)\\A\sim_\nu B}}
    \Exp{(x,y)\sim\mu}{\abs{1_B(x)-y}}
    +
    \eps\PrePer_\eps(B;\mu).
\end{align*}
Now we note that for $A\sim_\nu B$ it holds
\begin{align*}
    \Exp{(x,y)\sim\mu}{\abs{1_B(x)-y}}
    =
    w_0\int_B \de\rho_0 + 
    w_1\int_{B^c}\de\rho_1
    =
    \Exp{(x,y)\sim\mu}{\abs{1_A(x)-y}},
\end{align*}
since $A\sim_\nu B$ implies $\nu(A\triangle B)=0$ which by the absolute continuity implies $\rho(A\triangle B)=0$.
Hence, we obtain
\begin{align*}
    \inf_{\substack{B\in\B(\domain)\\A\sim_\nu B}}\Exp{(x,y)\sim\mu}{\sup_{\tilde x\in B_\eps(x)} \abs{1_B(\tilde x)-y}}
    &=
    \Exp{(x,y)\sim\mu}{\abs{1_A(x)-y}}
    +
    \eps
    \inf_{\substack{B\in\B(\domain)\\A\sim_\nu B}}
    \PrePer_\eps(B;\mu)
    \\
    &=
    \Exp{(x,y)\sim\mu}{\abs{1_A(x)-y}}
    +
    \eps\Per[\nu]_\eps(A;\mu).
\end{align*}
Finally, using \cref{prop:quotient_algebra} concludes the proof.
\end{proof}

\section{Analysis of the Adversarial Training Problem}
\label{sec:analysis}

In the previous section we have shown that the adversarial training problem \labelcref{eq:baseline_adv} is equivalent to the variational regularization problem \labelcref{eq:adv_perimeter} involving a nonlocal perimeter term. Problems like \labelcref{eq:adv_perimeter} are very well understood in the context of inverse problems~\cite{benning_burger_2018}. We will use the structure of the objective in problem \labelcref{eq:adv_perimeter} to make strong mathematical statements about our original adversarial training problem under very general conditions on the space $(\domain, \metric)$. The aim of this section is then to use the insights stemming from the reformulation in terms of perimeter in order to perform a rigorous analysis on the adversarial problem \labelcref{eq:baseline_adv}, focusing on proving existence of solutions and studying their properties. In particular, we will define convenient notions of uniqueness of solutions and show the existence of ``regular" solutions, at least in the Euclidean setting.    

For this, we first introduce a nonlocal total variation which is associated with the perimeter \labelcref{eq:perimeter} and that turns out to be useful for proving existence.
Then, we prove important properties of the perimeter and the total variation related to convexity and lower semicontinuity.
Here the key ingredient is to construct suitable representatives which attain the infimum in the definition of the perimeter $\Per[\nu]_\eps(\cdot;\mu)$. 
For this we will have to focus on reference measures $\nu$ which satisfy a certain geometric assumption.
Finally, we can use these insights to prove existence of solutions to \labelcref{eq:baseline_adv} and study their geometric properties.
Due to the lack of uniqueness of minimizers, we will investigate minimal and maximal solutions.

\subsection{The Associated Total Variation}

Similar to the nonlocal pre-perimeter \labelcref{eq:pre-perimeter} and perimeter \labelcref{eq:perimeter} we can also define an associated pre-total variation and total variation with respect to the measure $\nu$ of a measurable function $u:\domain\to\R$ as
\begin{align}
    \label{eq:PreTV}
    \PreTV_\eps(u;\mu)
    &:=
    \frac{w_0}{\eps} \int_\domain \sup_{\tilde x\in B_\eps(x)}u(\tilde x) - u(x) \de\rho_0(x) + 
    \frac{w_1}{\eps} \int_\domain u(x) - \inf_{\tilde x\in B_\eps(x)}u(\tilde x) \de\rho_1(x),
    \\
    \label{eq:TV}
    \TV[\nu]_\eps(u;\mu)
    &:=
    \inf_{\substack{v\in L^\infty(\domain;\nu)\\\text{$v=u$ $\nu$-a.e.}}}
    \PreTV_\eps(v;\mu).
\end{align}
\begin{remark}
If $\domain=\Rd$ and $w_1\rho_1=w_0\rho_0=1/2\mathcal{L}^d$ and $\nu=\mathcal{L}^d$, our results in this section, in particular \cref{prop:representative_TV}, show that the total variation reduces to
\begin{align}
    \TV_\eps(u;\mu) = \frac{1}{2\eps}\int_\Rd \essosc_{B_\eps(x)}(u) \de x,
\end{align}
which is precisely the nonlocal total variation associated to \labelcref{eq:std_perimeter} which was studied in~\cite{chambolle2012nonlocal}.
\end{remark}
\begin{remark}
We could have defined $\PreTV_\eps$ and $\TV[\nu]_\eps$ using the coarea formula. 
For the sake of clarity we decided to define the functionals directly and prove the coarea formula later.
\end{remark}
Having the total variation at hand, a natural convex relaxation of the perimeter-regularized variational problem \labelcref{eq:adv_perimeter} to functions instead of sets is
\begin{align}\label{eq:TV_problem}
    \inf_{\substack{u\in L^\infty(\domain;\nu)\\0\leq u\leq 1,\,\text{$\nu$-a.e.}}}\Exp{(x,y)\sim\mu}{\abs{u(x)-y}} + \eps \TV[\nu]_\eps(u;\mu),
\end{align}
where we again assume $\rho\ll\nu$.
Indeed, we will use this relaxation as an intermediate step in order to prove existence for minimizers of \labelcref{eq:adv_perimeter}.
Notably, since the first term in \labelcref{eq:TV_problem} \first involves integrals with respect to $\rho_0$ and $\rho_1$, \nc as shown in the proof of \cref{prop:perimeter_problem}, the condition $\rho\ll\nu$ implies that it makes sense to perform the optimization in \labelcref{eq:TV_problem} over $L^\infty(\domain;\nu)\subseteq L^\infty(\domain;\rho)$.

\subsection{Properties of the Nonlocal Perimeter}

The nonlocal perimeter $\Per[\nu]_\eps(\cdot;\mu)$ satisfies many of the same properties as the classical perimeter, which will also ensure that the total variation \labelcref{eq:TV} is well-defined and convex.
\begin{proposition}\label{prop:submodularity}
The set function $\Per[\nu]_\eps(\cdot;\mu)$ defined in \labelcref{eq:perimeter} satisfies the following:
\begin{itemize}
    \item $0\leq \Per[\nu]_\eps(A;\mu)<\infty$ for all sets $A\in\B(\domain)$.
    \item $\Per[\nu]_\eps(\emptyset;\mu)=\Per[\nu]_\eps(\domain;\mu)=0$.
    \item $\Per[\nu]_\eps(A;\mu)=\Per[\nu]_\eps(A';\mu)$ if $\nu(A\triangle A')=0$.
    \item It is submodular, meaning that for all $A,A'\in\B(\domain)$ it holds
    \begin{align*}
        \Per[\nu]_\eps(A\cup A';\mu) + \Per[\nu]_\eps(A\cap A';\mu) \leq \Per[\nu]_\eps(A;\mu) + \Per[\nu]_\eps(A';\mu).
    \end{align*}
\end{itemize}
\end{proposition}
\begin{remark}[Properties of the pre-perimeter]\label{rem:submod_preper}
If we choose $\nu$ to be the measure defined by $\nu(\emptyset)=0$ and $\nu(A)=\infty$ for all $A\in\B(\domain)\setminus\{\emptyset\}$ it holds $\PrePer_\eps(A;\mu)=\Per[\nu]_\eps(A;\mu)$ for all $A\in\B(\domain)$.
Hence, the pre-perimeter admits the same properties.
\end{remark}
\begin{proof}
The first statement follows from the fact that $\osc_{B_\eps(x)}(1_B)\leq 1$ for all sets $B\in\B(\domain)$ and $\rho$ is a probability measure.
The second statement is obvious since $\osc_{B_\eps(x)}(1_\domain)=0$ for all $x\in\domain$.
The third statement follows from the very definition of the perimeter, involving the infimum over sets $B\in\B(\domain)$ with $\nu(A\triangle B)=0$.

Let us now prove submodularity.
Elementary properties of the symmetric difference show
\begin{align*}
    (A\cup A') \triangle (B\cup B') \subseteq(A\triangle B) \cup (A'\triangle B'), \\
    (A\cap A') \triangle (B\cap B') \subseteq(A\triangle B) \cup (A'\triangle B').
\end{align*}
Using subadditivity of the measure $\rho$ this implies that for all $B,B'\in\B(\domain)$ with $\nu(A\triangle B)=0$ and $\nu(A'\triangle B')=0$ we can estimate
\begin{align*}
    &\phantom{=}
    \Per[\nu]_\eps(A\cup A';\mu) 
    + 
    \Per[\nu]_\eps(A\cap A';\mu) 
    \\
    &\leq 
    \frac{w_0}{\eps}\int_\domain  \sup_{B_\eps(\cdot)}1_{B\cup B'}-1_{B\cup B'} \de\rho_0 
    + \frac{w_1}{\eps}\int_\domain 1_{B\cup B'}-\inf_{B_\eps(\cdot)}1_{B\cup B'}\de\rho_1
    \\
    &\quad
    +
    \frac{w_0}{\eps}\int_\domain  \sup_{B_\eps(\cdot)}1_{B\cap B'}-1_{B\cap B'} \de\rho_0 
    + \frac{w_1}{\eps}\int_\domain 1_{B\cap B'}-\inf_{B_\eps(\cdot)}1_{B\cap B'}\de\rho_1
\end{align*}
Since $1_{B\cup B'}+1_{B\cap B'}=1_B+1_{B'}$ and $1_B - \inf_{B_\eps(\cdot)}1_B=\sup_{B_\eps(\cdot)}1_{B^c}-1_{B^c}$ for all $B,B'\in\B(\domain)$, it suffices to show
\begin{align}\label{ineq:pointwise_char_fun}
    \begin{split}
    \sup_{B_\eps(x)}1_{B\cup B'} + \sup_{B_\eps(x)}1_{B\cap B'} 
    \leq 
    \sup_{B_\eps(x)}1_{B}  + \sup_{B_\eps(x)}1_{B'}.
    \end{split}
\end{align}

\textbf{Case 0:}
If the left hand side is zero, we are done. 

\textbf{Case 1:}
Let us therefore assume that the first term in \labelcref{ineq:pointwise_char_fun} is equal to one and the second one equal to zero.
This means that there exists $y\in B\cup B'$ such that $\metric(x,y)\leq\eps$.
In particular at least one of the two terms on the right hand side in \labelcref{ineq:pointwise_char_fun} is $\geq 1$ which proves the inequality in this case.

\textbf{Case 2:} 
Now we assume that the second term is one.
This implies that there exists $y\in B\cap B'$ such that $\metric(x,y)\leq\eps$.
Hence, both terms on the right hand side in \labelcref{ineq:pointwise_char_fun} are $=1$ which makes the inequality correct independent of the first term.
\end{proof}

Next we prove that the infimum in the definition of the perimeter \labelcref{eq:perimeter} is actually attained.
In fact, for a suitable measure $\nu$ with $\rho\ll\nu$ we even construct a precise representative, i.e., a set $A^\star\in\B(\domain)$ with $A\sim_\nu A^\star$ in the equivalence relation \labelcref{eq:equiv_rel}, which attains this minimal value.
Even more, we show that the perimeter coincides with the essential perimeter with respect to the measure $\nu$.
The measure $\nu$ has to satisfy the following assumption.
\begin{assumption}\label{ass:nu_eps}
We assume that there exists a \editscolor $\sigma$-finite \nc measure $\nu$ on $\domain$ such that
\begin{enumerate}
    \item $\rho\ll\nu$,
    \item $\{x\in\domain\st\dist(x,\supp\rho)<\eps\}\subseteq\supp\nu$,
    \item $\nu$ is locally doubling (a Vitali measure), i.e., 
    \begin{align}\label{eq:local_doubling}
        \limsup_{r\downarrow 0}\frac{\nu(B_{2r}(x))}{\nu(B_{r}(x))} < \infty,\quad \text{for $\nu$-a.e. }x\in\domain.
    \end{align}
\end{enumerate}
\end{assumption}
\begin{remark}
Let us comment on these assumptions:
\begin{enumerate}
    \item The absolute continuity $\rho\ll\nu$ is needed for proving the reformulation as variational regularization problem, cf. \cref{prop:perimeter_problem}.
    \item The condition on $\supp\nu$ makes sure that problem \labelcref{eq:adv_perimeter} detects the effect of the adversary on the balls around points in the support of $\rho$.
    \item The doubling assumption \labelcref{eq:local_doubling} is a very weak assumption under which the Lebesgue differentiation theorem \editscolor (\cref{thm:Lebesgue} in \cref{sec:appendix_technical_defs}) \nc is valid.
\end{enumerate}
\end{remark}
\begin{remark}[Choice of the measure $\nu$]
If $\domain=\Rd$, then one can utilize a full support Gaussian $\gamma$ to define $\nu:=\rho + \gamma$.
In that case (1)-(3) are true by definition and it is straightforward to show that if $\varrho$ is locally doubling, then so is $\nu$, see also \cite[p.81]{heinonen2015sobolev}.
In turn, notice that if $\varrho$ is supported on finitely many points (e.g., an empirical measure) or if $\varrho$ is absolutely continuous with respect to the Lebesgue measure, then the measure $\varrho$ is locally doubling.

More generally, if $(\domain,\metric)$ is a finite-dimensional smooth Riemannian manifold (intrinsically defined without the need of an Euclidean ambient space) with a Riemannian volume form $\omega$, and finite total volume, then $\nu$ can be taken to be of the form $\nu = \varrho + \omega$.
\nc
\end{remark}
Using such a measure $\nu$ we can state the following proposition which says that a) the infimum in the definition of the perimeter \labelcref{eq:perimeter} is attained, and b) that the perimeter can be expressed as the essential perimeter with respect to $\nu$.
\begin{proposition}\label{prop:representative_Per}
Under \cref{ass:nu_eps} for any $A\in\B(\domain)$ there exists $A^\star\in\B(\domain)$ with $A\sim_{\nu}A^\star$ such that
\begin{align}
    \Per[\nu]_\eps(A;\mu) = \PrePer_\eps(A^\star;\mu).
\end{align}
Furthermore, the perimeter admits the characterization
\begin{align}
    \Per[\nu]_\eps(A;\mu) =
    \frac{w_0}{\eps}\int_\domain  \esssup[\nu]_{B_\eps(\cdot)}1_A-1_A \de\rho_0 
    + \frac{w_1}{\eps}\int_\domain  1_{A}-\essinf[\nu]_{B_\eps(\cdot)}1_{A} \de\rho_1.
\end{align}
\end{proposition}
For the proof of the proposition we need a preparatory lemma which deals with the construction of the representative set.
\begin{lemma}
\label{lem:ModifySets}
Under \cref{ass:nu_eps} for any $A\in\B(\domain)$ there exists $A^\star\in\B(\domain)$ with $A\sim_{\nu}A^\star$ such that
\begin{equation}
\label{eqn:PropertySupEsssup}
  \sup_{B_\eps(x)} {1}_{A^\star} = \esssup[\nu]_{B_\eps(x)} {1}_{A^\star}, \quad  \inf_{B_\eps(x)} {1}_{A^\star} =  \essinf[\nu]_{B_\eps(x)} {1}_{A^\star}, \quad \forall x\in \supp\rho.
\end{equation}
\end{lemma}

\begin{proof}
Let $u= {1}_A$ and let $D_+, D_-$ be the sets defined by:
\begin{align*}
    D_+ &:= \left\{ x \in \supp\varrho \: : \: \esssup[\nu]_{B_\eps (x)} u =1 , \quad \essinf[\nu]_{B_\eps (x)} u =1   \right\},
    \\
    D_- &:= \left\{ x \in \supp\varrho \: : \: \esssup[\nu]_{B_\eps (x)} u =0 , \quad \essinf[\nu]_{B_\eps (x)} u =0   \right\}.
\end{align*}
Also, let $D_+^\eps $ and $D_-^\eps$ be the sets:
\begin{align*}
    D_\pm^\eps := \left\{ x \in \R^d \: : \: \dist(x, D_\pm) < \eps  \right \}
\end{align*}
We claim that $D_+^\eps$ and $D_-^\eps$ are disjoint. Indeed, suppose for the sake of contradiction that there is a point $\tilde x$ in their intersection. Then, we would be able to find $x_1 \in D_+$ and $x_0 \in D_-$ such that $ \tilde x \in B_\eps(x_1)$ and $\tilde{x} \in B_\eps(x_0) $. In particular, we could find $\delta>0$ small enough such that
\[ B_\delta(\tilde x)  \subseteq B_\eps(x_1) \cap B_\eps(x_0). \]
In addition, since $D_+^\eps$ (or $D_-^\eps$) is by \cref{ass:nu_eps} a subset of the support of $\nu$, we would conclude that $\tilde x $ belongs to the support of $\nu$ and thus $\nu(B_\delta(\tilde x))>0 $. However, this would be a contradiction, because the above inclusion implies that, for example, $\essinf[\nu]_{B_\eps(x_1)} u =0$, contrary to the fact that $x_1 \in D_+$. 

Since $D_+^\eps$ and $D_-^\eps$ are disjoint we can now define the function $u^\star $ as:
\[ u^\star (x) := \begin{cases} 1 \quad \text{ if }  x \in D_+^\eps \\ 0 \quad \text{ if } x \in D_- ^\eps  \\ u(x) \quad \text{ if } x \in \R^d \setminus  (D_+^\eps \cup D_-^\eps ). \end{cases} \]
Notice that the function $u^\star$ is Borel measurable since the sets $D^\eps_{\pm}$ are open sets. We claim that $\nu$-a.e. it holds $u =u^\star$. To see this, notice that it suffices to show that $u(x) =1$ for $\nu$-a.e. $x \in D_+^\eps$ and that $u(x)=0$ for $\nu$-a.e. $x \in D_-^\eps$; we can focus on the first case as the second one is completely analogous. 
By definition of $D_+^\eps$ it holds
\[ 
\left\{ x \in D_+^\eps \st u(x)=0  \right\}
\subseteq 
\left\{ x \in D_+^\eps \st u(x) \not = \lim_{r \rightarrow 0}   \frac{1}{\nu(B_r(x))} \int_{B_r(x)} u(\tilde x) \de \nu(\tilde x ) \right\}. \]
Notice that this is the case since for $r>0$ small enough $\nu$-a.e. it holds $u=1$ in $B_r(x)$ when $x \in D_+^\eps$.
However, by the Lebesgue differentiation theorem applied to the measure $\nu$ and the measurable function $u$ (which is possible thanks to \cref{ass:nu_eps}, \editscolor see \cref{thm:Lebesgue}\nc) the latter set must have $\nu$ measure zero. This implies our claim. 

On the other hand, for every $x \in D_+$, by definition of $u^\star$ we have
\[ \sup_{B_\eps(x) } u^\star  = 1=  \esssup[\nu]_{B_\eps (x)} u^\star , \quad \inf_{B_\eps(x) } u^\star  = 1=  \essinf[\nu]_{B_\eps (x)} u^\star  \]
and for every $x \in D_-$  
\[ \inf_{B_\eps(x) } u^\star  = 0=\essinf[\nu]_{B_\eps (x)} u^\star , \quad \sup_{B_\eps(x) } u^\star  = 0=  \esssup[\nu]_{B_\eps (x)} u^\star .   \]
Finally, if $x \in \supp (\varrho) \setminus (D_+ \cup D_- )$ we have:
\[  \esssup[\nu]_{B_\eps(x)} u^\star = 1 , \quad  \essinf[\nu]_{B_\eps(x)} u^\star = 0.  \]
In particular, we also have:
\[ \esssup[\nu]_{B_\eps(x)} u^\star =1= \sup_{B_\eps(x)} u^\star  , \quad  \essinf[\nu]_{B_\eps(x)} u^\star =0= \inf_{B_\eps(x)} u^\star .   \]
The set $A^\star$ is now defined as $A^\star:=(u^\star )^{-1}(\{ 1\})$. This concludes the proof.
\end{proof}

\begin{remark}
\label{rem:OnOpenvsClosedBalls}
Notice that in the previous proof, specifically when we state that there is a $\delta>0$ such that $B_\delta(\tilde x) \subseteq B_\eps(x_1) \cap B_\eps(x_0)$, we implicitly use the fact that the adversarial model was defined in terms of \textit{open} balls $B_\eps$ as opposed to closed balls. The bottom line is that the construction of $u^*$ in the proof would not carry through if we replaced open with closed balls since in that case the sets $D_{\pm}^\eps$(appropriately modified) would not necessarily be disjoint. 
\end{remark}
\nc

Now we are ready to prove \cref{prop:representative_Per}.
\begin{proof}[Proof of \cref{prop:representative_Per}]
Using the construction from \cref{lem:ModifySets}, the definition of the perimeter \labelcref{eq:perimeter}, and the fact that $\sup\geq\esssup$, we compute
\begin{align*}
    &\;\phantom{=}
    \frac{w_0}{\eps}\int_\domain  \esssup[\nu]_{B_\eps(\cdot)}1_A-1_A \de\rho_0 
    + \frac{w_1}{\eps}\int_\domain  1_A-\essinf[\nu]_{B_\eps(\cdot)}1_{A} \de\rho_1
    \\
    &=
    \frac{w_0}{\eps}\int_\domain  \esssup[\nu]_{B_\eps(\cdot)}1_{A^\star}-1_{A^\star} \de\rho_0 
    + \frac{w_1}{\eps}\int_\domain  1_{A^\star}-\essinf[\nu]_{B_\eps(\cdot)}1_{A^\star} \de\rho_1
    \\
    &=
    \frac{w_0}{\eps}\int_\domain  \sup_{B_\eps(\cdot)}1_{A^\star}-1_{A^\star} \de\rho_0 
    + \frac{w_1}{\eps}\int_\domain  1_{A^\star}-\inf_{B_\eps(\cdot)}1_{A^\star} \de\rho_1
    \\
    &\geq
    \Per[\nu]_\eps(A;\mu)
    \\
    &=
    \inf_{\substack{B\in\B(\domain)\\A\sim_\nu B}}\frac{w_0}{\eps}\int_\domain  \sup_{B_\eps(\cdot)}1_B-1_B \de\rho_0 
    + \frac{w_1}{\eps}\int_\domain  1_B-\inf_{B_\eps(\cdot)}1_{B} \de\rho_1
    \\
    &\geq
    \inf_{\substack{B\in\B(\domain)\\A\sim_\nu B}}
    \frac{w_0}{\eps}\int_\domain  \esssup[\nu]_{B_\eps(\cdot)}1_B-1_B \de\rho_0 
    + \frac{w_1}{\eps}\int_\domain  1_B-\essinf[\nu]_{B_\eps(\cdot)}1_{B}\de\rho_1
    \\
    &=
    \frac{w_0}{\eps}\int_\domain  \esssup[\nu]_{B_\eps(\cdot)}1_A-1_A \de\rho_0 
    + \frac{w_1}{\eps}\int_\domain  1_A-\essinf[\nu]_{B_\eps(\cdot)}1_{A} \de\rho_1.
\end{align*}
Hence, equality holds everywhere, which completes the proof.
\end{proof}

\subsection{Properties of the Total Variation}

\editscolor
We start with an elementary homogeneity property of the total variation $\TV[\nu](\cdot;\mu)$ which follow immediately from its definition.
\begin{proposition}\label{prop:elementary_TV}
The functional $\TV[\nu](\cdot;\mu)$ defined in \labelcref{eq:TV} satisfies the following for all measurable functions $u:\domain\to\R$, $c\in\R$, and $\alpha\geq 0$:
\begin{align*}
    \TV[\nu](\alpha u + c;\mu) = \alpha\TV[\nu](u;\mu).
\end{align*}
\end{proposition}
\begin{proof}
The proof is trivial and we omit it.
\end{proof}
\nc

Now we prove the analogous result of \cref{prop:representative_Per} for the total variation.
We rely heavily on the construction from \cref{lem:ModifySets}.
\begin{proposition}\label{prop:representative_TV}
Under \cref{ass:nu_eps}, for any $u\in L^\infty(\domain;\nu)$ there exists $u^\star\in L^\infty(\domain;\nu)$ such that $u=u^\star$ holds $\nu$-almost everywhere and
\begin{align}
    \TV[\nu]_\eps(u;\mu) = \PreTV_\eps(u^\star;\mu).
\end{align}
Furthermore, the total variation admits the characterization
\begin{align}\label{eq:TV_repr}
    \TV[\nu]_\eps(u;\mu) =
    \frac{w_0}{\eps}\int_\domain  \esssup[\nu]_{B_\eps(\cdot)}u-u \de\rho_0 
    + \frac{w_1}{\eps}\int_\domain  u-\essinf[\nu]_{B_\eps(\cdot)}u \de\rho_1.
\end{align}
\end{proposition}
\begin{proof}
The proof works just as the proof of \cref{prop:representative_Per}, however, using \cref{lem:ModifyFcts} below.
\end{proof}

\first 
The following lemma extends the construction of \cref{lem:ModifyFcts} from sets to functions. 
The proof is given in \cref{sec:appendix_proofs}.
\nc 

\begin{lemma}\label{lem:ModifyFcts}
Under \cref{ass:nu_eps} for any Borel measurable function $u\in L^\infty(\domain;\nu)$ there exists $u^\star:\domain\to\R$ such that $u=u^\star$ holds $\nu$-almost everywhere and
\begin{equation}
    \sup_{B_\eps(x)} u^\star = \esssup[\nu]_{B_\eps(x)} u^\star , \quad  \inf_{B_\eps(x)} u^\star  = \essinf[\nu]_{B_\eps(x)} u^\star , \quad \forall x \in \supp\rho.  
\end{equation}
\end{lemma}

In fact, the nonlocal perimeter and total variation are connected via a coarea formula, as it is the case for their local counterparts.
Thanks to the characterizations as essential perimeter and total variation from \cref{prop:representative_Per,prop:representative_TV} the proof becomes very simple.

\begin{proposition}[Coarea formula]\label{prop:coarea}
Under \cref{ass:nu_eps} it holds \editscolor for any $u\in L^\infty(\domain;\nu)$ that \nc \begin{align}
    \TV[\nu]_\eps(u;\mu) = \int_\R\Per[\nu]_\eps(\{u\geq t\};\mu)\de t.
\end{align}
\end{proposition}
\begin{proof}
\second
Let us first assume that $u\geq 0$.
\nc
Using \cref{prop:representative_TV,prop:representative_Per}, the layer cake representation, monotone convergence to swap integrals and supremum/infima, and Tonelli's theorem to swap integrals we can compute
\begin{align*}
    \TV[\nu]_\eps(u;\mu) 
    &=
    \frac{w_0}{\eps} \int_\domain \esssup[\nu]_{B_\eps(\cdot)}u - u \de\rho_0 + 
    \frac{w_1}{\eps} \int_\domain u - \essinf[\nu]_{B_\eps(\cdot)}u \de\rho_1
    \\
    &=
    \frac{w_0}{\eps} \int_\domain \esssup[\nu]_{B_\eps(\cdot)}\int_0^\infty 1_{\{u\geq t\}}\de t - \int_0^\infty 1_{\{u\geq t\}}\de t \de\rho_0 
    \\
    &\qquad
    + 
    \frac{w_1}{\eps} \int_\domain \int_0^\infty 1_{\{u\geq t\}}\de t - \essinf[\nu]_{B_\eps(\cdot)}\int_0^\infty 1_{\{u\geq t\}}\de t \de\rho_1
    \\
    &=
    \int_0^\infty
    \left(
    \frac{w_0}{\eps} \int_\domain \esssup[\nu]_{B_\eps(\cdot)} 1_{\{u\geq t\}}- 1_{\{u\geq t\}} \de\rho_0 \right.
    \\
    &\qquad
    + 
    \left.
    \frac{w_1}{\eps} \int_\domain 1_{\{u\geq t\}} - \essinf[\nu]_{B_\eps(\cdot)}1_{\{u\geq t\}} \de\rho_1
    \right)
    \de t
    \\
    &=
    \int_0^\infty 
    \Per[\nu]_\eps(\{u\geq t\})\de t.
\end{align*}
\second 
In the general case we have that $\nu$-a.e. it holds $m\leq u \leq M$ for some $m,M\in\R$ with $m\leq M$.
We can define the function $\tilde u := u-m$ which satisfies $\tilde u \geq 0$ and hence
\begin{align}\label{eq:coarea_rescaling}
    \TV[\nu](\tilde u;\mu) = \int_0^\infty\Per[\nu](\{\tilde u\geq t\};\mu)\de t.
\end{align}
Using \cref{prop:elementary_TV} it holds
\begin{align*}
    \TV[\nu](\tilde u;\mu)
    =
    \TV[\nu](u-m;\mu)
    =
    \TV[\nu](u;\mu).
\end{align*}
Furthermore, the perimeter integral satisfies
\begin{align*}
    \int_0^\infty\Per[\nu](\{\tilde u\geq t\};\mu)\de t
    =
    \int_0^\infty\Per[\nu](\{u\geq t + m\};\mu)\de t
    =
    \int_m^\infty\Per[\nu](\{u\geq t\};\mu)\de t.
\end{align*}
Plugging these two reformulations into \labelcref{eq:coarea_rescaling} shows
\begin{align*}
    \TV[\nu](u;\mu) = \int_m^\infty\Per[\nu](\{u\geq t\};\mu)\de t
    =
    \int_\R\Per[\nu](\{u\geq t\};\mu)\de t.
\end{align*}
\nc 
\end{proof}

The main consequence of the previous properties of the perimeter and the total variation is that the the latter constitutes a convex and weak-* lower semicontinuous functional on $L^\infty(\domain;\nu)$.

Showing the weak-* lower semicontinuity on $L^\infty(\domain;\nu)$ requires a little bit more work.
For this we need a couple of preparatory lemmas.
These depend on the validity of the Lebesgue differentiation theorem which requires the doubling condition in \cref{ass:nu_eps}.

\begin{lemma}\label{lem:wstar_pointwise}
Assume that $(\domain,\metric,\nu)$ is a Vitali metric measure space, meaning that $\nu$ satisfies \labelcref{eq:local_doubling}, assume that $\nu$ is $\sigma$-finite, and suppose that $u_k\wsto u$ in $L^\infty(\domain;\nu)$.
Then for $\nu$-almost every $x\in\domain$ and all $\eps>0$
\begin{align*}
    \limsup_{k\to\infty}
    \essinf[\nu]_{B_\eps(x)}u_k
    \leq
    u(x)
    \leq\liminf_{k\to\infty}
    \esssup[\nu]_{B_\eps(x)}u_k.
\end{align*}
\end{lemma}
\begin{proof}
Since $\nu$ is $\sigma$-finite, $L^\infty(\domain;\nu)$ is the dual of $L^1(\domain;\nu)$ and hence by definition of weak-* convergence \editscolor(see \cref{sec:appendix_technical_defs}) \nc it holds
\begin{align*}
    \int_\domain u\,\phi\de\nu = \lim_{k\to\infty}\int_\domain u_k\,\phi\de\nu,\quad\forall\phi\in L^1(\domain;\nu).
\end{align*}
Choosing $\phi=\frac{1}{\nu(B_r(x))}1_{B_r(x)}$ for $r>0$ it holds that
\begin{align*}
    \frac{1}{\nu(B_r(x))}\int_{B_r(x)}u\de\nu =
    \lim_{k\to\infty}
    \frac{1}{\nu(B_r(x))}\int_{B_r(x)}u_k\de\nu.
\end{align*}
Hence, using \cref{thm:Lebesgue} we obtain for $\nu$-a.e. $x\in\domain$ any $\eps>0$ 
\begin{align*}
    u(x) 
    &= 
    \lim_{r\downarrow 0}\frac{1}{\nu(B_r(x))}\int_{B_r(x)}u\de\nu
    \\
    &=
    \lim_{r\downarrow 0}
    \lim_{k\to\infty}
    \frac{1}{\nu(B_r(x))}\int_{B_r(x)}u_k\de\nu
    \\
    &\leq 
    \limsup_{r\downarrow 0}
    \liminf_{k\to\infty}
    \esssup[\nu]_{B_r(x)}u_k
    \\
    &\leq
    \liminf_{k\to\infty}
    \esssup[\nu]_{B_\eps(x)}u_k.
\end{align*}
Similarly, one establishes the inequality $u(x)\geq\limsup_{k\to\infty}\essinf[\nu]_{B_\eps(x)}u_k$.
\end{proof}
\begin{lemma}\label{lem:lsc_osc}
Under the conditions of \cref{lem:wstar_pointwise} it holds for $\nu$-a.e. $x\in\domain$ and all $\eps>0$
\begin{align*}
     \esssup[\nu]_{B_\eps(x)}u
     &\leq
     \liminf_{k\to\infty}
     \esssup[\nu]_{B_\eps(x)}u_k,
     \\
     \essinf[\nu]_{B_\eps(x)}u
     &\geq
     \limsup_{k\to\infty}
     \essinf[\nu]_{B_\eps(x)}u_k.
\end{align*}
\end{lemma}
\begin{proof}
Let us choose $0<\delta<\eps$.
For $\nu$-almost every $y\in\domain$ \cref{lem:wstar_pointwise} implies
\begin{align*}
    u(y) \leq \liminf_{k\to\infty}\esssup[\nu]_{B_{\delta}(y)}u_k.
\end{align*}
Taking the $\esssup[\nu]$ over $y\in B_{\eps-\delta}(x)$ yields
\begin{align*}
    \esssup[\nu]_{B_{\eps-\delta}(x)} u
    &\leq 
    \esssup[\nu]_{y\in B_{\eps-\delta}(x)}
    \liminf_{k\to\infty}
    \esssup[\nu]_{B_{\delta}(y)}u_k
    \\
    &\leq
    \liminf_{k\to\infty}
    \esssup[\nu]_{y\in B_{\eps-\delta}(x)}
    \esssup[\nu]_{B_{\delta}(y)}u_k
    \\
    &\leq
    \liminf_{k\to\infty}
    \esssup[\nu]_{B_{\eps}(x)}u_k    .
\end{align*}
Choosing $\delta>0$ arbitrarily small yields
\begin{align*}
    \esssup[\nu]_{B_\eps(x)}u \leq \liminf_{k\to\infty}\esssup[\nu]_{B_\eps(x)}u_k.
\end{align*}
Applying this reasoning to $-u_k$ one shows analogously that
\begin{align*}
    \limsup_{k\to\infty}\essinf[\nu]_{B_\eps(x)}u_k 
    \leq 
    \essinf[\nu]_{B_\eps(x)}u.
\end{align*}
\end{proof}

Now we are ready to prove the following proposition which states important properties of the total variation.
\begin{proposition}\label{prop:TV_prop}
Under \cref{ass:nu_eps} the functional $\TV[\nu]_\eps(\cdot;\mu)$, defined in \labelcref{eq:TV}, is a positively homogeneous, weak-* lower semicontinuous, and convex functional on $L^\infty(\domain;\nu)$.
Lower semicontinuity is understood in the sense that
\begin{align*}
    u_k \wsto u \text{ in }L^\infty(\domain;\nu) \implies \TV[\nu]_\eps(u;\mu)\leq\liminf_{k\to\infty}\TV[\nu]_\eps(u_k;\mu).
\end{align*}
\end{proposition}
\begin{proof}
The positive homogeneity \editscolor was already proved in \cref{prop:elementary_TV}\nc.
To prove lower semicontinuity we use \cref{prop:representative_TV} to write
\begin{align*}
    \TV[\nu]_\eps(u;\mu) 
    &=
    \frac{w_0}\eps\left(\int_\domain \esssup[\nu]_{B_\eps(x)}u\de\rho_0-\int_\domain u \de\rho_0\right)
    \\
    &\qquad
    +
    \frac{w_1}\eps\left(\int_\domain  u\de\rho_1 - \int_\domain \essinf[\nu]_{B_\eps(x)}u\de\rho_1\right)
    \\
    &=
    \frac{w_0}\eps\left(\int_\domain \esssup[\nu]_{B_\eps(x)}u\frac{\de\rho_0}{\de\nu}\de\nu-\int_\domain u \frac{\de\rho_0}{\de\nu}\de\nu\right)
    \\
    &\qquad
    +
    \frac{w_1}\eps\left(\int_\domain  u\frac{\de\rho_1}{\de\nu}\de\nu - \int_\domain \essinf[\nu]_{B_\eps(x)}u\frac{\de\rho_1}{\de\nu}\de\nu\right),
\end{align*}
\editscolor
where $\frac{\de\rho_i}{\de\nu}$ denotes the Radon--Nikod\'ym derivative of $\rho_i$ with respect to $\nu$ (note that $\rho\ll\nu$ and $\rho=w_0\rho_0+w_1\rho_1$ implies $\rho_i\ll\nu$ for $i\in\{0,1\}$).
Let $(u_k)_{k\in\N}\subset L^\infty(\domain;\nu)$ be a sequence such that $u_k\wsto u$ as $k\to\infty$ where $u\in L^\infty(\domain;\nu)$.
Then it holds
\begin{align*}
    \int_\domain \esssup[\nu]_{B_\eps(x)}u\frac{\de\rho_0}{\de\nu}\de\nu
    \leq
    \int_\domain 
    \liminf_{k\to\infty}
    \esssup[\nu]_{B_\eps(x)}u_k\frac{\de\rho_0}{\de\nu}\de\nu
    \leq
    \liminf_{k\to\infty}
    \int_\domain 
    \esssup[\nu]_{B_\eps(x)}u_k\frac{\de\rho_0}{\de\nu}\de\nu.
\end{align*}
Note that, being a weakly-* convergent sequence, $\{u_k\}_{k\in\N}$ is uniformly bounded in $L^\infty(\domain;\nu)$ by some constant $C>0$.
Furthermore, $\int_\domain C \frac{\de\rho_0}{\de\nu}\de\nu=C\rho_0(\domain)<\infty$ which justifies an application of Fatou's lemma to the sequence $\esssup[\nu]_{B_\eps(x)}u_k + C$ for the second inequality.
One argues analogously for the other integral containing the $\essinf[\nu]$, using the reverse Fatou lemma.

Furthermore, since $\frac{\de\rho_i}{\de\nu}\in L^1(\domain;\nu)$ for $i\in\{0,1\}$, weak-* convergence of $u_k$ to $u$ directly implies
\begin{align*}
    \int_\domain u\frac{\de\rho_i}{\de\nu}\de\nu = 
    \lim_{k\to\infty}\int_\domain u_k\frac{\de\rho_i}{\de\nu}\de\nu,\qquad i\in\{0,1\}.
\end{align*}
Hence, we have established weak-* lower semicontinuity of $\TV[\nu]$.
\nc 

Convexity is a direct consequence of the submodularity of the perimeter, the coarea formula from \cref{prop:coarea}, and the lower semicontinuity; the proof works just as in~\cite[Prop. 3.4]{chambolle2010continuous}.
\end{proof}

\subsection{Existence of Solutions}
\label{sec:ExistenceSolutions}
We have completed all preparations to finally state our existence result for the adversarial problem \labelcref{eq:baseline_adv}.
The proof uses the direct method to establish existence of a minimizer of the variational problem \labelcref{eq:adv_perimeter}.
Then we use the representative constructed in \cref{prop:representative_Per} to turn this minimizer into a minimizer of the original problem \labelcref{eq:baseline_adv}.
This last step is shown in the following lemma.
\begin{lemma}\label{lem:minimizer2minimizer}
Let $A\in\B(\domain)$ be a solution of \labelcref{eq:adv_perimeter}.
Then $A^\star$, \first constructed in \cref{prop:representative_Per}\nc, is a solution of \labelcref{eq:baseline_adv}.
\end{lemma}
\begin{proof}
Using \cref{prop:perimeter_problem,prop:representative_Per}, we get
\begin{align*}
    \Exp{(x,y)\sim\mu}{\abs{1_{A^\star}-y}}+\eps\PrePer_\eps(A^\star;\mu)
    &=
    \Exp{(x,y)\sim\mu}{\abs{1_{A}-y}}+\eps\Per[\nu]_\eps(A;\mu)
    \\
    &=
    \inf_{A\in\B(\domain)}
    \Exp{(x,y)\sim\mu}{\abs{1_{A}-y}}+\eps\Per[\nu]_\eps(A;\mu)
    \\
    &=
    \inf_{A\in\B(\domain)}
    \Exp{(x,y)\sim\mu}{\abs{1_{A}-y}}+\eps\PrePer_\eps(A;\mu),
\end{align*}
which, thanks to \cref{prop:adv_risk_per}, is equivalent to $A^\star$ solving \labelcref{eq:baseline_adv}.
\end{proof}

\begin{theorem}[Existence of Minimizers]\label{thm:existence}
Under \cref{ass:nu_eps} there exists a solution $A\in\B(\domain)$ of problem \labelcref{eq:baseline_adv}.
\end{theorem}
\begin{proof}
Let $(A_k)_{k\in\N}\subseteq\B(\domain)$ be a minimizing sequence of \labelcref{eq:adv_perimeter} \third which is trivially bounded in $L^\infty(\domain;\nu)$\nc.
Using weak-* \third precompactness of bounded subsets of $L^\infty(\domain;\nu)$ (see \cref{thm:Banach-Alaoglu} in \cref{sec:appendix_technical_defs}) \nc we know that there exists $u\in L^\infty(\domain;\nu)$ such that a subsequence (which we don't relabel) satisfies $1_{A_k}\wsto u$ in $L^\infty(\domain;\nu)$.
Furthermore, from \cref{lem:wstar_pointwise} we know that $0\leq u(x)\leq 1$ for $\nu$-a.e. $x\in\domain$.

\editscolor
Let us first show that the empirical risk $\Exp{(x,y)\sim\mu}{\abs{u(x)-y}}$ is weak-* lower semicontinuous, in fact even continuous, along this sequence. 
For this we compute is as
\begin{align*}
    \Exp{(x,y)\sim\mu}{\abs{u(x)-y}} 
    &=
    w_0\int_\domain \abs{u(x)}\de\rho_0(x) + w_1\int_\domain\abs{u(x)-1}\de\rho_1(x)
    \\
    &=
    w_0\int_\domain u(x)\frac{\de\rho_0}{\de\nu}\de\nu(x) + w_1\int_\domain (1-u(x))\frac{\de\rho_1}{\de\nu}\de\nu(x)
    \\
    &=
    \lim_{k\to\infty}
    w_0\int_\domain 1_{A_k}(x)\frac{\de\rho_0}{\de\nu}\de\nu(x) + w_1\int_\domain (1-1_{A_k}(x))\frac{\de\rho_1}{\de\nu}\de\nu(x)
    \\
    &=
    \lim_{k\to\infty}
    \Exp{(x,y)\sim\mu}{\abs{1_{A_k}(x)-y}}.
\end{align*}
\nc 
Using \cref{prop:TV_prop} and the fact that $\TV[\nu](1_A;\mu)=\Per[\nu]_\eps(A;\mu)$ for all $A\in\B(\domain)$ we infer that
\begin{align}\label{eq:u_minimizer}
\begin{split}
    \Exp{(x,y)\sim\mu}{\abs{u(x)-y}} + \eps\TV[\nu]_\eps(u;\mu)
    &\leq
    \liminf_{k\to\infty}\Exp{(x,y)\sim\mu}{\abs{1_{A_k}(x)-y}} + \eps\Per[\nu]_\eps(1_{A_k};\mu)
    \\
    &=
    \inf_{A\subseteq\B(\domain)}
    \Exp{(x,y)\sim\mu}{\abs{1_A(x)-y}} + \eps\Per[\nu]_\eps(A;\mu).
\end{split}    
\end{align}

For $t\in[0,1]$ define the set $A_t :=\{u \geq t\}$.
It trivially holds
\begin{align*}
    \inf_{A\subseteq\B(\domain)}
    \Exp{(x,y)\sim\mu}{\abs{1_A(x)-y}} + \eps\Per[\nu]_\eps(A;\mu)
    \leq
    \Exp{(x,y)\sim\mu}{\abs{1_{A_t}(x)-y}} + \eps\Per[\nu]_\eps(A_t;\mu).
\end{align*}
Aiming for a contradiction we assume this inequality to be strict on a subset of $[0,1]$ with positive Lebesgue measure.
Integrating over $t\in[0,1]$ and using \cref{prop:coarea} we get
\begin{align*}
    &\phantom{<}
    \inf_{A\subseteq\B(\domain)}
    \Exp{(x,y)\sim\mu}{\abs{1_A(x)-y}} + \eps\Per[\nu]_\eps(A;\mu)
    \\
    &<
    \int_0^1 
    \Exp{(x,y)\sim\mu}{\abs{1_{A_t}(x)-y}} + \eps\Per[\nu]_\eps(A_t;\mu)\de t
    \\
    &=
    \Exp{(x,y)\sim\mu}{\abs{u(x)-y}} + \eps\TV[\nu]_\eps(u;\mu),
\end{align*}
which contradicts \labelcref{eq:u_minimizer}.
Hence, the inequality is an equality which shows that also $A_t$ is a minimizer of \labelcref{eq:adv_perimeter} for almost all $t\in[0,1]$.
In particular, \cref{lem:minimizer2minimizer} shows that $A_t^\star$ solves \labelcref{eq:baseline_adv} for almost every $t\in[0,1]$.
\end{proof}

The previous proposition establishes the existence of minimizers of the adversarial problem \labelcref{eq:baseline_adv}. 
However, it is not yet clear whether minimizers are unique or regular (of course considering equivalence classes modulo $\nu$). 
However, since the problem is not strictly convex in nature---cf. the relaxation \labelcref{eq:TV_problem}---in general uniqueness cannot be expected. 
This can trivially arise due to a separation between the supports of $\rho_0$ and $\rho_1$, as evidenced by the following example.

\first

\begin{example}\label{ex:four-balls}
Fixing $\e > 0$, suppose that $\mu$ is given by four Dirac masses centered at $(\pm \e, \pm \e)$ in $\R^2$, and that opposing corners are (deterministically) given the same label, namely $\rho_0 = \frac{1}{2} \delta_{(\e, -\e)} + \frac{1}{2}\delta_{(-\e,\e)}$ and $\rho_1 = \frac{1}{2} \delta_{(-\e, -\e)} + \frac{1}{2}\delta_{(\e,\e)}$, and $w_0 = w_1 = \frac{1}{2}$. Then it is straightforward to check that any set $A$ such that $B_\e((\e,\e)) \subset A$, $B_\e( (-\e,-\e)) \subset A$, $B_\e((-\e,\e)) \cap A = \emptyset$ and $B_\e((\e,-\e)) \cap A = \emptyset$ will be minimizers of the adversarial risk: indeed any such set has zero risk. The largest and smallest such sets (in blue color) are demonstrated in \cref{fig:four-balls}.
\end{example}

\begin{figure}
     \centering
     \begin{subfigure}[b]{0.45\textwidth}
         \centering
         \includegraphics[width=\textwidth]{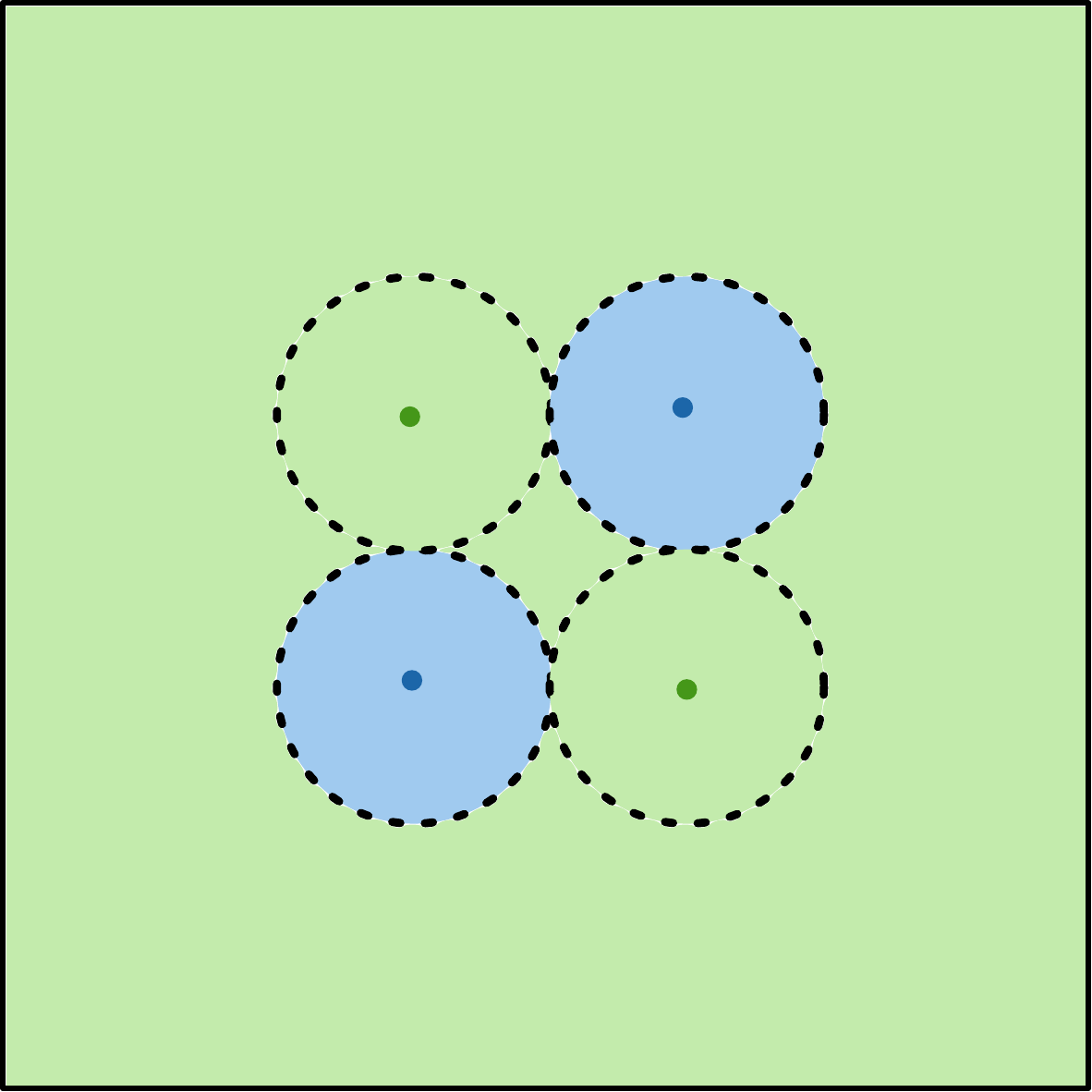}
         \caption{Smallest adversarial minimizer}
         \label{fig:four-balls-a}
     \end{subfigure}
     \hfill
     \begin{subfigure}[b]{0.45\textwidth}
         \centering
         \includegraphics[width=\textwidth]{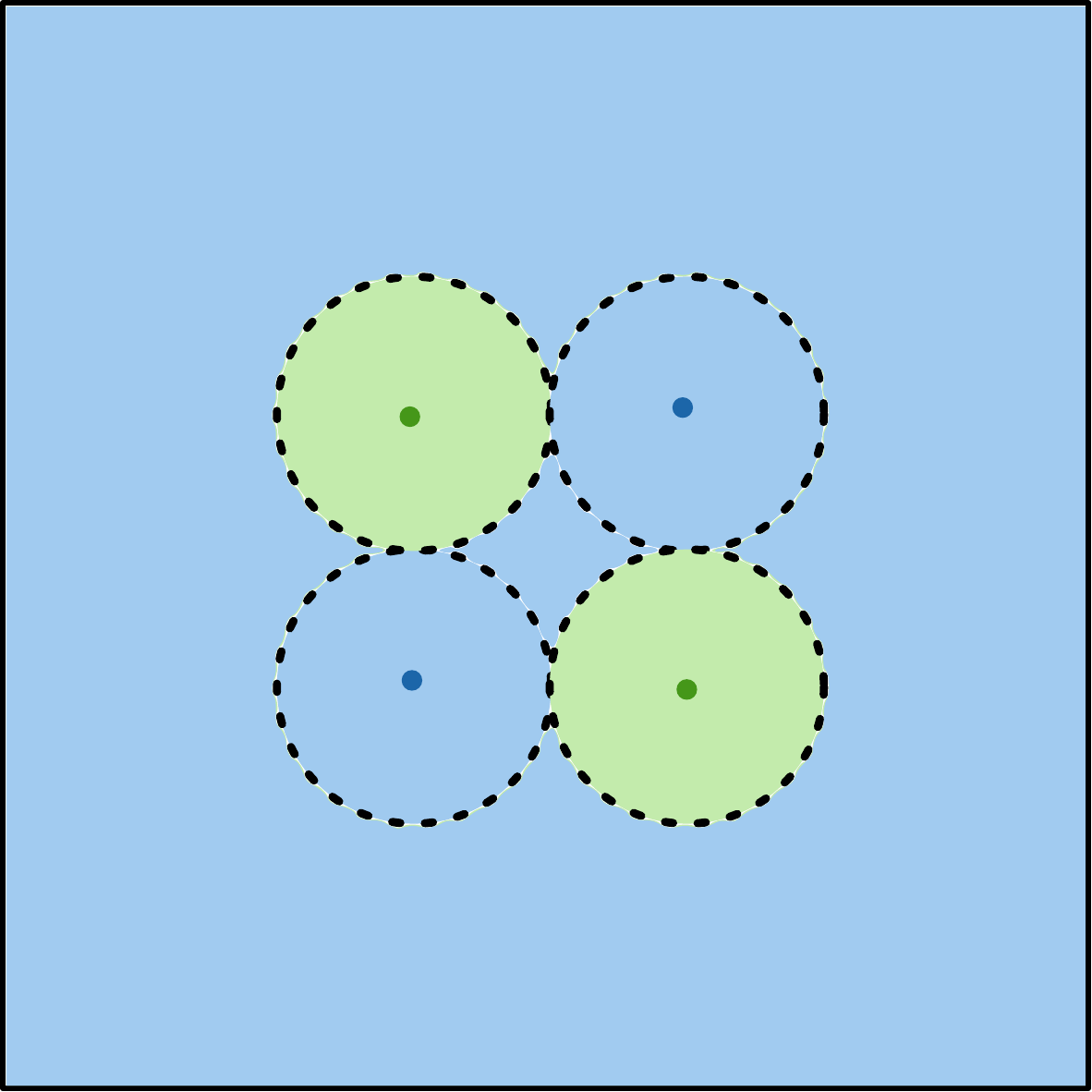}
         \caption{Largest adversarial minimizer}
         \label{fig:four-balls-b}
     \end{subfigure}
        \caption{\first Situation from \cref{ex:four-balls}  with non-unique, smooth minimizers. Here the four Dirac masses are displayed as well as the balls of radius $\e$ which surround them. Infinitely many minimizers of the adversarial risk exist, we display here the largest and smallest possible minimizers in blue color.}
        \label{fig:four-balls}
\end{figure}

\nc

The previous example demonstrates that one cannot hope for any type of uniqueness, or even that \textit{all} minimizers will necessarily be regular. Although the previous example utilized Dirac masses for simplicity, we suspect that many of the same issues can arise for distributions with smooth densities.

Despite the previous considerations, it is possible to obtain some positive results.
In particular, we can define notions of maximal and minimal minimizers to \labelcref{eq:baseline_adv} which are then shown to be unique.
Moreover, we show that although there may be irregular minimizers, we can always find regular minimizers provided that we define an appropriate notion of regularity relative to the metric $\metric$.
\editscolor 
Proving this will be the content of the following two sections.
\nc 

\subsection{Extremal Solutions}
\label{sec:ExtremalSolutions}

For notational convenience, we define the adversarial risk associated to \labelcref{eq:baseline_adv} as
\begin{align}\label{eq:adv_risk}
    \widetilde R_\eps(A)
    :=
    \Exp{(x,y)\sim\mu}{\sup_{\tilde x\in B_\eps(x)}\abs{1_A(\tilde x)-y}}
    =
    \Exp{(x,y)\sim\mu}{\abs{1_A(x)-y}}+\eps\PrePer_\eps(A;\mu).
\end{align}
To begin, we prove submodularity of the adversarial risk and show that the set of minimizers is closed under unions and intersections.
\begin{lemma}\label{lem:risk_submod}
The adversarial risk is submodular, meaning that it satisfies
\begin{align}
    \widetilde R_\eps(A \cup B) + \widetilde R_\eps(A \cap B) \leq \widetilde R_\eps(A) + \widetilde R_\eps(B),\quad\forall A,B\in\B(\domain).
\end{align}
\end{lemma}
\begin{proof}
We first notice that
\begin{align*}
\Exp{(x,y)\sim\mu}{\abs{1_A(x)- y}} + \Exp{(x,y)\sim\mu}{\abs{1_B(x)- y}} = \Exp{(x,y)\sim\mu}{\abs{1_{A\cap B}(x)- y}} + \Exp{(x,y)\sim\mu}{\abs{1_{A \cup B}(x)- y}}.    
\end{align*}
This fact can be directly proved by decomposing $\domain$ into $A\cap B$, $B \setminus A$, $A \setminus B$ and $\domain \setminus (A \cup B)$, splitting the integrals, and then reassembling.
Together with the submodularity of the pre-perimeter (cf. \cref{rem:submod_preper}) this implies the assertion.
\end{proof}
\begin{proposition}
\label{prop:ClosureUnderUnion}
Let $A$ and $B$ be minimizers of the adversarial problem \labelcref{eq:baseline_adv} with parameter $\eps\geq0$. 
Then both $A \cap B$ and $A \cup B$ are both also minimizers.
\end{proposition}
\begin{proof}
Using \cref{lem:risk_submod}, it is immediate that, for any $A,B$, either $\widetilde R_\eps(A\cup B) \leq \frac{\widetilde R_\eps(A) + \widetilde R_\eps( B)}{2}$ or $\widetilde R_\eps(A\cap B) \leq \frac{\widetilde R_\eps(A) + \widetilde R_\eps( B)}{2}$: suppose that the former is true. Then if $A$ and $B$ are both minimizers then we immediately obtain that $A\cup B$ is also a minimizer. Subtracting the minimal risk from both sides then also implies that $\widetilde R_\eps(A \cap B) = \widetilde R_\eps(A)$. The other case is completely analogous.
\end{proof}
We now proceed to introduce the setting under which we can make sense of maximal and minimal solutions to problem \labelcref{eq:baseline_adv}.
In fact, we first work with the relaxed problem \labelcref{eq:adv_perimeter} and then use \cref{lem:minimizer2minimizer} to obtain statements about the original problem \labelcref{eq:baseline_adv}.
We introduce the following notation for sets $A,A'\in\B(\domain)$:
\begin{align}
   A \preceq_\nu A'
   :\iff
   1_{A}(x) \leq 1_{A'}(x)\quad\text{for $\nu$-a.e. $x\in\domain$}.
\end{align}
Notice that the relation $\preceq_\nu$ above induces a partial order in the set of equivalence classes of~$\sim_{\nu}$, in other words in the quotient $\sigma$-algebra $\B_{\nu}(\domain)$.
We now define maximal and minimal solutions and show their existence and uniqueness, \editscolor see \cref{fig:maximal-ex1} for an example. \nc 
\begin{definition}[Maximal (minimal) solutions]
We say that $A\in\B(\domain)$ is a maximal (minimal) solution of \labelcref{eq:adv_perimeter} if $A$ is a solution with the property that any other solution $A'\in\B(\domain)$ to \labelcref{eq:adv_perimeter} satisfying $A \preceq_\nu A'$ ($A' \preceq_\nu A$) must satisfy $A \sim_{\nu} A'$.
\end{definition}
\begin{proposition}\label{prop:max_min_sol}
Assume that $\nu$ is a finite measure on $\domain$.
Then there exists a unique maximal (minimal) solution to problem \labelcref{eq:adv_perimeter} up to $\nu$-equivalence. The maximal solution is denoted with $A_\mathrm{max}$ while the minimal solution is denoted with $A_\mathrm{min}$. 
\end{proposition}
\begin{proof}
We follow an argument in \cite{chambolle2015nonlocal} and proceed as follows. Since $\nu$ is a finite measure
\[ 
m:= \sup \{ \nu(A) \: : \: A \text{ solution of }  \labelcref{eq:adv_perimeter}\}<\infty.
\]
Take a maximizing sequence $\{A_n \}_{n \in \N}\subseteq\B(\domain)$ in the definition of $m$ so that $\lim_{n \to \infty} \nu(A_n) = m $. From \cref{prop:ClosureUnderUnion} we know that for each $n\in\N$ the set $\bigcup_{k=1}^n A_k$ is also a solution to problem \labelcref{eq:adv_perimeter}. 
Let $A:=\bigcup_{k=1}^\infty A_k $, then it holds
\[ 
1_{\bigcup_{k=1}^n A_k} \to 1_{A}, \quad \text{in $L^1(\domain;\nu)$ as } n \to \infty. 
\]
From the above strong convergence it is immediate that $A$ is also a solution to \labelcref{eq:adv_perimeter} and we have
\[ 
m = \lim_{n \rightarrow \infty} \nu(A_n) \leq \lim_{n \rightarrow \infty} \nu\left(\bigcup_{k=1}^n A_k \right) = \nu(A) \leq m. 
\]
Now, notice that if there was a solution $A'$ such that $A \preceq_\nu A'$ and $A \not \sim_{\nu} A' $, then we would have $m= \nu(A)< \nu(A')$ which would contradict the definition of $m$. Likewise, if there were two solutions $A,A'$ with $\nu(A)=m = \nu(A')$ and the two sets were not equivalent, then by taking their union we would be able to obtain a solution with $\nu$-volume strictly larger than $m$. This shows the existence and uniqueness of maximal solutions. 
A similar proof can be used to deduce the existence and uniqueness of minimal solutions.
\end{proof}

We can also introduce a notion of maximality and minimality of solutions for problem \labelcref{eq:baseline_adv}, at least when restricting to a class of solutions obtained by considering specific representatives of solutions $A\in\B(\domain)$ to \labelcref{eq:baseline_adv}. 
In contrast to the definition of $A^\star$ in \cref{lem:ModifySets}, which in general is representative dependent, the following notions are independent of the representative of $A$ in the quotient $\sigma$-algebra $\B_{\nu}(\domain)$. 
Given $A\in\B(\domain)$ we define  Borel sets $A^+$ and $A^-$ through their indicators according to the formulas: 
\begin{align*}
1_{A^+}(x)
&:= 
\begin{cases} 
1 \text{ if } x\in \supp(\nu) \text{ and } \limsup_{r \downarrow 0} \frac{\nu( B_r(x) \cap A)}{\nu(B_r(x))} >0, \\ 
1 \text{ if } x \notin \supp(\nu), \\ 
0 \text{ if }  x\in \supp(\nu) \text{ and } \limsup_{r \downarrow 0} \frac{\nu( B_r(x) \cap A)}{\nu(B_r(x))}= 0 , \end{cases} 
\quad x\in \domain,
\\
1_{A^-}(x)
&:= 
\begin{cases} 
1 \text{ if } x\in \supp(\nu) \text{ and } \liminf_{r \downarrow 0} \frac{\nu( B_r(x) \cap A)}{\nu(B_r(x))} =1,  \\ 
0 \text{ if } x \notin \supp(\nu), \\ 
0 \text{ if }  x\in \supp(\nu) \text{ and } \liminf_{r \downarrow 0} \frac{\nu( B_r(x) \cap A)}{\nu(B_r(x))} < 1,
\end{cases} 
\quad x\in \domain.
\end{align*}
Notice that for any Borel set $A$ we have $A^-\subseteq A^+$. In addition, notice that $A^+ = (A^+)^\star$ as well as $A^- = (A^-)^\star$. In particular, if $A$ is a solution to problem \labelcref{eq:adv_perimeter}, then both $A^+$ and $A^-$ are solutions to problem \labelcref{eq:baseline_adv} according to \cref{lem:minimizer2minimizer}.

\begin{figure}
     \centering
     \begin{subfigure}[b]{0.3\textwidth}
         \centering
         \includegraphics[width=\textwidth]{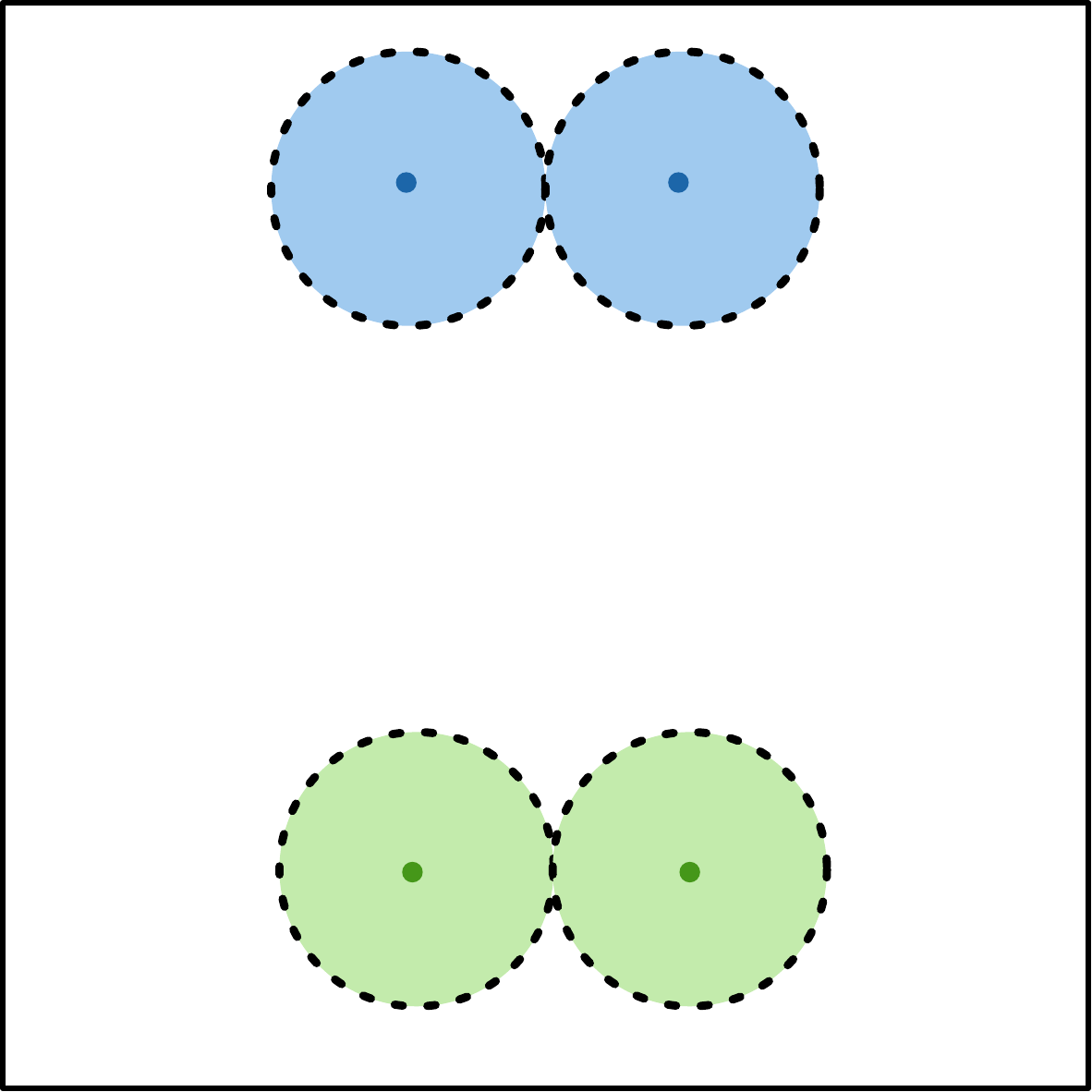}
         \caption{Data surrounded by $\eps$-balls}
         \label{fig:maximal-ex1a}
     \end{subfigure}
     \hfill
     \begin{subfigure}[b]{0.3\textwidth}
         \centering
         \includegraphics[width=\textwidth]{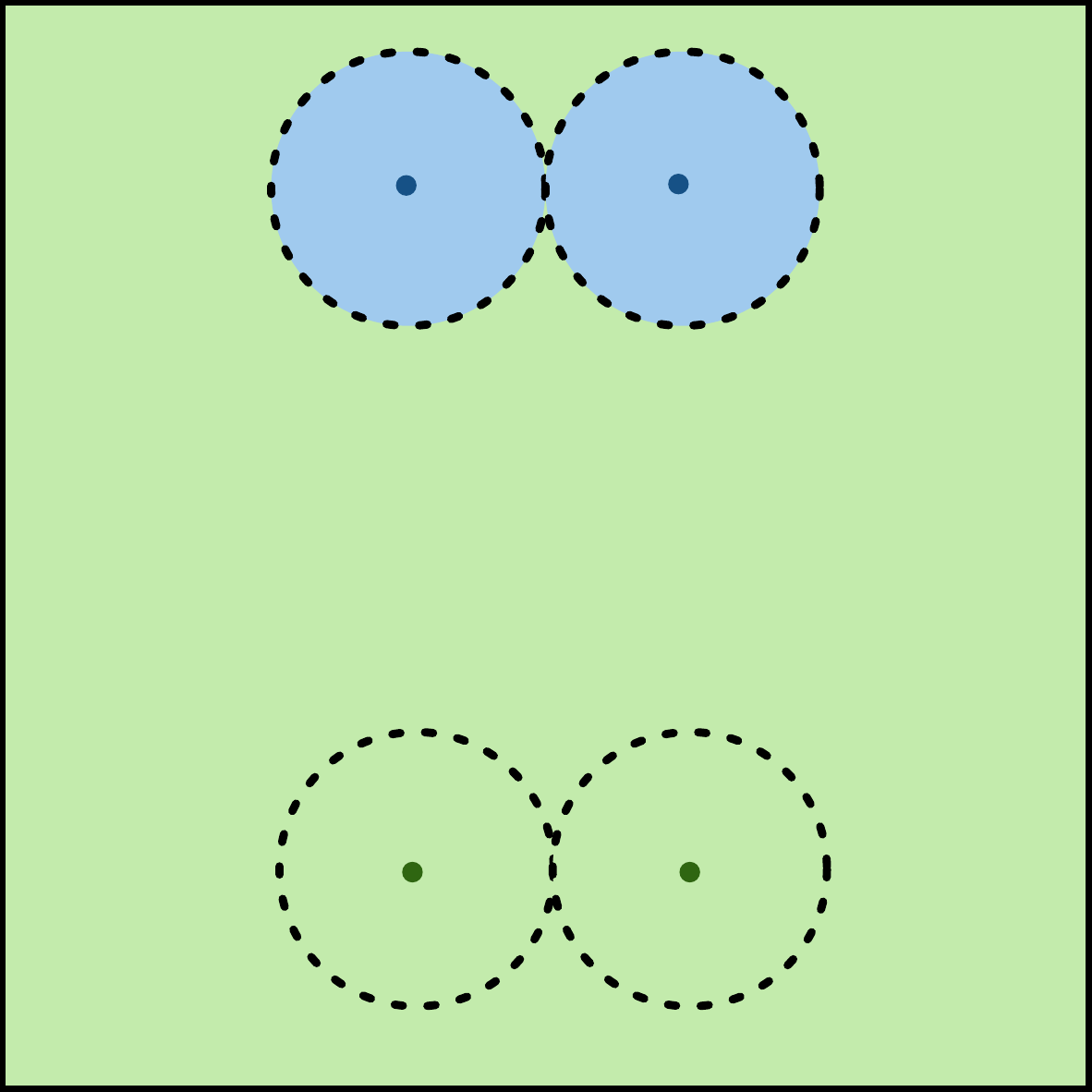}
         \caption{Minimal set}
         \label{fig:maximal-ex1b}
     \end{subfigure}
     \hfill
     \begin{subfigure}[b]{0.3\textwidth}
         \centering
         \includegraphics[width=\textwidth]{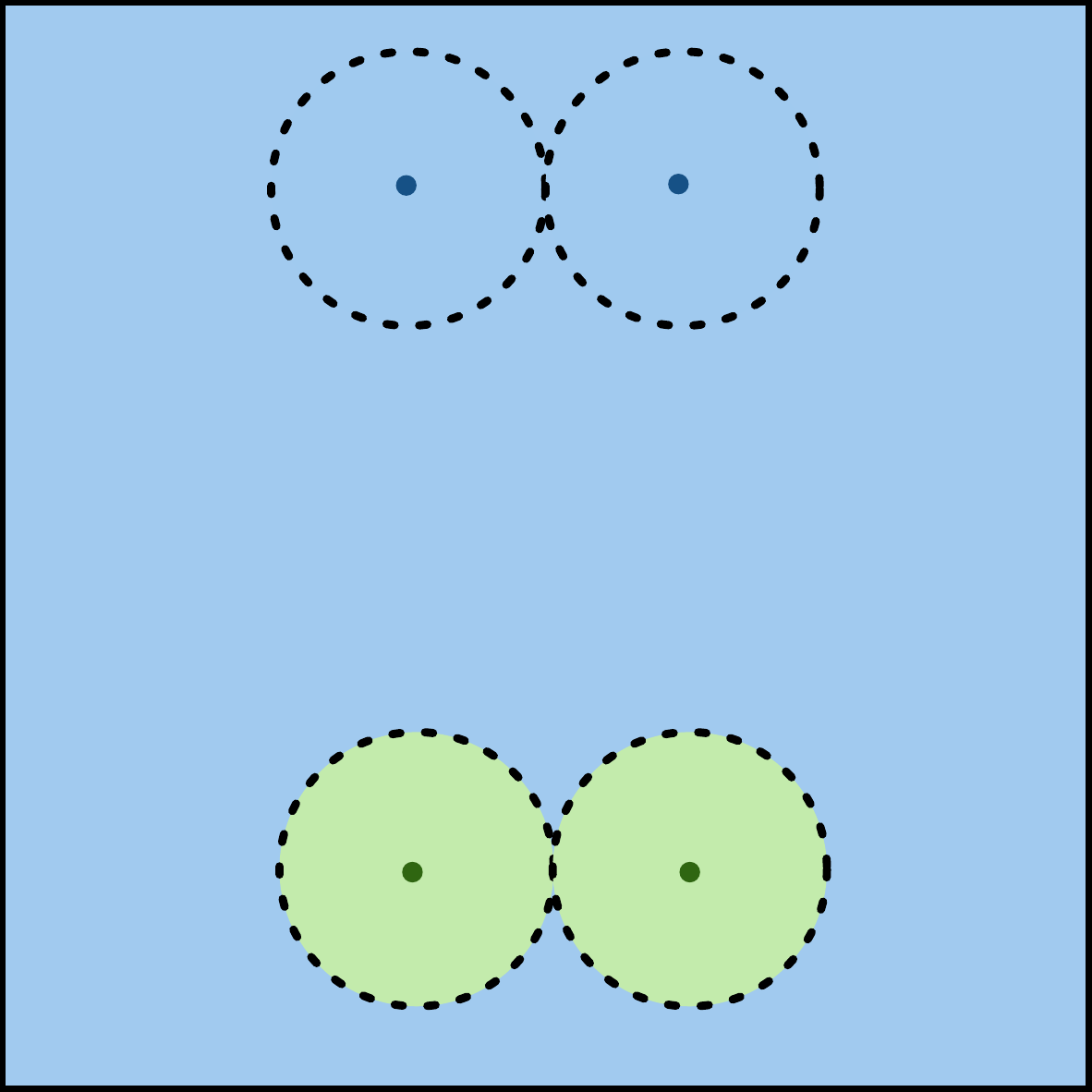}
         \caption{Maximal set}
         \label{fig:maximal-ex1c}
     \end{subfigure}
        \caption{The maximal and minimal sets \editscolor(in blue color) \nc associated with a particular distribution of point masses. Here the maximal and minimal sets have boundaries that cannot even be represented as graphs of a function at every point. In this case any intermediate set, in the sense of inclusion, will also be a minimizer, and many smooth minimizers are possible.}
        \label{fig:maximal-ex1}
\end{figure}

The next is an immediate consequence of \cref{prop:max_min_sol} and the above definitions.

\begin{corollary}\label{cor:max_min_sets}
\editscolor Assume that $\nu$ is a finite measure.
\nc
Among the set of solutions to problem \labelcref{eq:baseline_adv} of the form $A^+$ for some solution $A\in\B(\domain)$ of problem \labelcref{eq:adv_perimeter}, $A_\mathrm{max}^+$ is maximal in the sense of inclusions. 
Likewise, among the set of solutions to problem \labelcref{eq:baseline_adv} of the form $A^-$ for some solution $A$ of problem \labelcref{eq:adv_perimeter}, $A_\mathrm{min}^-$ is maximal in the sense of inclusions. In addition, $A_\mathrm{min}^- \subseteq A_\mathrm{max}^+$.
\end{corollary}

\begin{proof}
Notice that if $A \preceq_\nu A'$ we immediately have $A^+ \subseteq (A')^+$ and $A^- \subseteq (A')^-$.
Since we have $A_\mathrm{min} \preceq_\nu A \preceq_\nu A_\mathrm{max}$ for any solution $A$ of \labelcref{eq:adv_perimeter}, the first part of the corollary follows. 
The inclusion $A_\mathrm{min}^- \subseteq A_\mathrm{max}^+$ follows from $A_\mathrm{min}^- \subseteq A_\mathrm{min}^+ \subseteq A_\mathrm{max}^+$.
\end{proof}
\nc

\subsection{Regularity}
\label{sec:RegularitySolutions}

\first
The goal of this section is to prove that, in an Euclidean setting, it is possible to construct a smooth minimizer of the adversarial problem. We offer a direct construction, under which normal vectors of the boundary are Hölder continuous and shall prove the following statement. 

\begin{theorem}\label{thm:regular-sol}
Consider the case where $\domain = \R^d$, equipped with the standard Euclidean metric. Then for any $\eps>0$ there exists a minimizer $B\in\B(\Rd)$ to the adversarial problem \labelcref{eq:baseline_adv} which is locally the graph of a $C^{1,1/3}$ function.
\end{theorem}
\nc

\editscolor
We will first deduce a series of regularity properties of minimizers that, although not as strong as those in \cref{thm:regular-sol}, hold for general metric measure spaces $(\mathcal{X}, \de, \nu)$ satisfying \cref{ass:nu_eps}, before proceeding to the proof of \cref{thm:regular-sol}. 
\nc 
We start by introducing some fundamental concepts of mathematical morphology. In particular, we define the following important concepts from mathematical morphology (see, e.g., Chapter 2 in \cite{serra1986introduction}).
\begin{definition}[Morphology]
Let $A\subseteq\domain$ be a set and $\eps>0$. We define its
\begin{itemize}
    \item dilation as $A^{\eps}:= \{ x \in \domain \st \dist(x,A) < \eps \}$,
    \item erosion as $A^{-\eps}:= \{ x \in \domain \st \dist(x,A^c)\geq \eps \}$,
    \item closing as $\cl_\eps(A):=(A^\eps)^{-\eps}$,
    \item opening as $\op_\eps(A):=(A^{-\eps})^{\eps}$.
\end{itemize}
\end{definition}
Notice that all these sets are measurable as they are open or closed sets. 
In the following proposition we collect a couple of important properties of these operations, which can be proved in a straightforward way (see \cite{haralock1991computer}).
\begin{proposition}\label{prop:morphology}
The following statements hold true:
\begin{itemize}
    \item $\cl_\eps(A)$ is a closed set that contains $A$,
    \item $\op_\eps(A)$ is an open set contained in $A$,
    \item $\cl_\eps(A)^\eps=A^\eps$,
    \item $\op_\eps(A)^{-\eps}=A^{-\eps}$,
    \item $A^{-\eps} = ((A^c)^\eps)^c$,
    \item $\cl_\eps(A^c)=\op_\eps(A)^c$.
\end{itemize}
\end{proposition}

The following definition of one-sided regularity of sets is strongly connected to the opening and closing procedures.
\begin{definition}[Inner and outer regularity]\label{def:inner_outer_reg}
A set $A\subseteq\domain$ is called $\eps$ inner regular relative to the metric $\metric$ if, for any point $x \in \partial A$ then there exists a point $y\in\domain$ so that $\metric(x,y) = \eps$ and $B_\eps(y) \subseteq A$. A set $A\subseteq\domain$ is called $\eps$ outer regular relative to the metric $\metric$ if instead we can always find such a $y\in\domain$ satisfying the inclusion $B_\eps(y) \subseteq A^c$.
\end{definition}
Note that by definition, for any set $A\subseteq\domain$, its closing $\cl_\eps(A)$ is $\eps$ outer regular whereas its opening $\op_\eps(A)$ is $\eps$ inner regular.
Furthermore, in $\domain=\Rd$ equipped with the Euclidean metric, it was shown in \cite{lewicka2020domains} that a set which is both $\eps$ inner and outer regular has a $C^{1,1}$ boundary.
\second 
A similar concept of regularity, called pseudo-certifiable robustness, is introduced and used in \cite{awasthi2021existence_extended}. 
There, a set $A$ is called pseudo-certifiably robust if every point in the set (or its complement) is an element of an $\eps$-ball contained in the set (or its complement).
It is easy to show that this notion of regularity implies inner and outer regularity in the sense of \cref{def:inner_outer_reg}.
\nc

We now show that the opening and closing operations do not increase the adversarial risk~\labelcref{eq:adv_risk}.
As a consequence, the operations turn minimizers into minimizers.

\begin{lemma}\label{lem:clos_op}
For $A\in\B(\domain)$ it holds
\begin{align*}
    \widetilde{R}_\eps(\cl_\eps(A) ) \leq \widetilde{R}_\eps(A),
    \qquad
    \widetilde{R}_\eps(\op_\eps(A) ) \leq \widetilde{R}_\eps(A).
\end{align*}
\end{lemma}
\begin{proof}
Using \cref{prop:morphology} we can rewrite the adversarial risk as follows:
\begin{align*}
    \widetilde R_\eps(A) 
    &=
    w_0 \int_\domain \sup_{B_\eps(x)} 1_A\de\rho_0(x) 
    + w_1
    \int_\domain \sup_{B_\eps(x)} 1_{A^c}\de\rho_1(x) 
    \\
    &=
    w_0\rho_0(A^\eps) + w_1\rho_1((A^c)^\eps)
    \\
    &=
    w_0\rho_0(A^\eps) + w_1 - w_1\rho_1(((A^c)^\eps)^c)
    \\
    &=
    w_0\rho_0(A^\eps) + w_1 - w_1\rho_1(A^{-\eps}).
\end{align*}
Using \cref{prop:morphology} again we get
\begin{align*}
    \widetilde R_\eps(\cl_\eps(A)) 
    &= 
    w_0\rho_0(\cl_\eps(A)^\eps) + w_1 - w_1\rho_1(\cl_\eps(A)^{-\eps})
    \\
    &\leq
    w_0\rho_0(A^\eps) + w_1 - w_1\rho_1(A^{-\eps}) = \widetilde R_\eps(A),
    \\
    \widetilde R_\eps(\op_\eps(A))
    &= 
    w_0\rho_0(\op_\eps(A)^{\eps}) + w_1 - w_1\rho_1(\op_\eps(A)^{-\eps})
    \\
    &\leq
    w_0\rho_0(A^\eps) + w_1 - w_1\rho_1(A^{-\eps}) = \widetilde R_\eps(A).
\end{align*}
\end{proof}
\begin{corollary}\label{cor:clos_op_minim}
Let $A\in\B(\domain)$ be a minimizer of \labelcref{eq:baseline_adv}.
Then $\op_\eps(A)$ and $\cl_\eps(A)$ are also minimizers.
\end{corollary}

We can now show that one can always construct a closed and $\eps$ outer regular maximal set and an open and $\eps$ inner regular minimal set which solves the adversarial problem \labelcref{eq:baseline_adv}.

\begin{proposition}
\editscolor Assume that $\nu$ is a finite measure.
\nc
There exist two solutions $A'_+$ and $A'_-$ to \labelcref{eq:baseline_adv} with the following properties:
\begin{enumerate}
    \item $A_-'  \subseteq A_+'$.
    \item $A'_+$ is a closed set and $A'_-$ is an open set.
    \item $A'_+$ is $\eps$ outer regular relative to the metric $\metric$ and $A'_-$ is $\eps$ inner regular with respect to the metric $\metric$.
\item $A_\mathrm{max} \sim_{\nu} A_+'$  and $A_\mathrm{min} \sim_{\nu} A_-'$.
\end{enumerate}
\end{proposition}
\begin{proof}
Let $A_+':= \cl_\eps(A_\mathrm{max}^+)$ and let $A_-':=\op_\eps(A_\mathrm{min}^-)$. 
Notice that by \cref{cor:clos_op_minim,cor:max_min_sets} we have:
\[  A_-' \subseteq A_\mathrm{min}^- \subseteq  A_\mathrm{max}^+ \subseteq A_+'.\]
(2) and (3) on the other hand follow directly from the definitions of $A'_\pm$ as closing and opening.
Finally, since $A_\mathrm{max}$ (and hence also $A_\mathrm{max}^+$) is a maximal solution of \labelcref{eq:adv_perimeter} and by definition $A'_+\supseteq A_\mathrm{max}^+$, it has to hold $A'_+\sim_{\nu}A_\mathrm{max}^+\sim_{\nu}A_\mathrm{max}$.
An analogous argument applies to $A'_-$.
\end{proof}

\begin{remark}
In the case where $\domain=\R^d$ under the standard Euclidean metric, one can directly conclude some mild regularity of maximal and minimal sets. For example, using the results in \cite{oleksiv1985finiteness} one may conclude that the boundaries of the maximal and minimal sets are sets of locally finite classical perimeter of order $\eps^{-1}$; see \cite{ambrosio2000functions} for a definition of classical perimeter. Similar results were examined in \cite{jog2021reverse}. Furthermore, at any point where curvatures are defined the outer (inner) regularity provides a uniform $\eps^{-2}$ upper (lower) bound on the sectional curvatures. However, as manifest in the example in \cref{fig:maximal-ex1}, the maximal and minimal sets need not have boundaries that are even graphs of functions at every point.
\end{remark}
\nc




\nc 

The next statement asserts that any intermediate set between the opening and the closing of a minimizer is again a minimizer.
\begin{proposition}\label{prop:intermediate}
Let $A\in\B(\domain)$ be a minimizer of \labelcref{eq:baseline_adv} and let $B\in\B(\domain)$ satisfy
\begin{align*}
    \op_\eps(A)\subseteq B \subseteq \cl_\eps(A).
\end{align*}
Then $B$ is also a minimizer of \labelcref{eq:baseline_adv}.
\end{proposition}
\begin{proof}
We abbreviate $\hat A = \cl_\eps(A)$ and $\tilde A = \op_\eps(A)$ and notice that $\tilde A \subseteq \hat A$. Furthermore, we notice that by the definition of closing and opening we have
\[
(\hat A)^{-\eps} = (\tilde A)^{-\eps},
\]
which in turn implies that for any set $B\in\B(\domain)$ with $\tilde A \subseteq B \subseteq \hat A$ it holds $B^{-\eps} = (\hat A)^{-\eps}$. Similarly, we have $\tilde A^\eps \subseteq B^\eps \subseteq (\hat A)^\eps$.
We then note that as in the proof of \cref{lem:clos_op}
\begin{align*}
    \widetilde R_\eps(A) 
    = 
    w_0\rho_0(A^\eps) + w_1\rho_1((A^{-\eps})^c).
\end{align*}
However, given the previous set inclusions, and the fact that thanks to \cref{cor:clos_op_minim} it holds $\widetilde R_\eps(\hat A) = \widetilde R_\eps(\tilde A)$, this then gives that $\widetilde R_\eps(B) = \widetilde R_\eps(\tilde A)$ or in other words $B$ also minimizes the adversarial risk. 
\end{proof}
\editscolor
It might be tempting to think that one could just consider the opening of the closing (or vice versa) of a set to generate a minimizer which is both outer and inner regular.
However, that this approach fails in general, as the following example shows:
\begin{example}
In this example we consider the set $A$, given by the union of two balls with radius $\eps>0$ with two non-convex triangles, as depicted in \cref{fig:smooth_balls}.
The set can be defined as $A:=\cl_\eps\left(B_\eps(-\eps,0)\cup B_\eps(\eps,0)\right)$.
It satisfies $\op_\eps(A)=B_\eps(-\eps,0)\cup B_\eps(\eps,0)$ and hence $\cl_\eps(\op_\eps(A)) = A$.
Still, it is not inner regular since the boundary points which are contained in the two triangles do not possess a touching ball with radius $\eps$ that is contained in $A$.
\begin{figure}[hbt]
    \centering
    \includegraphics[width=0.5\textwidth]{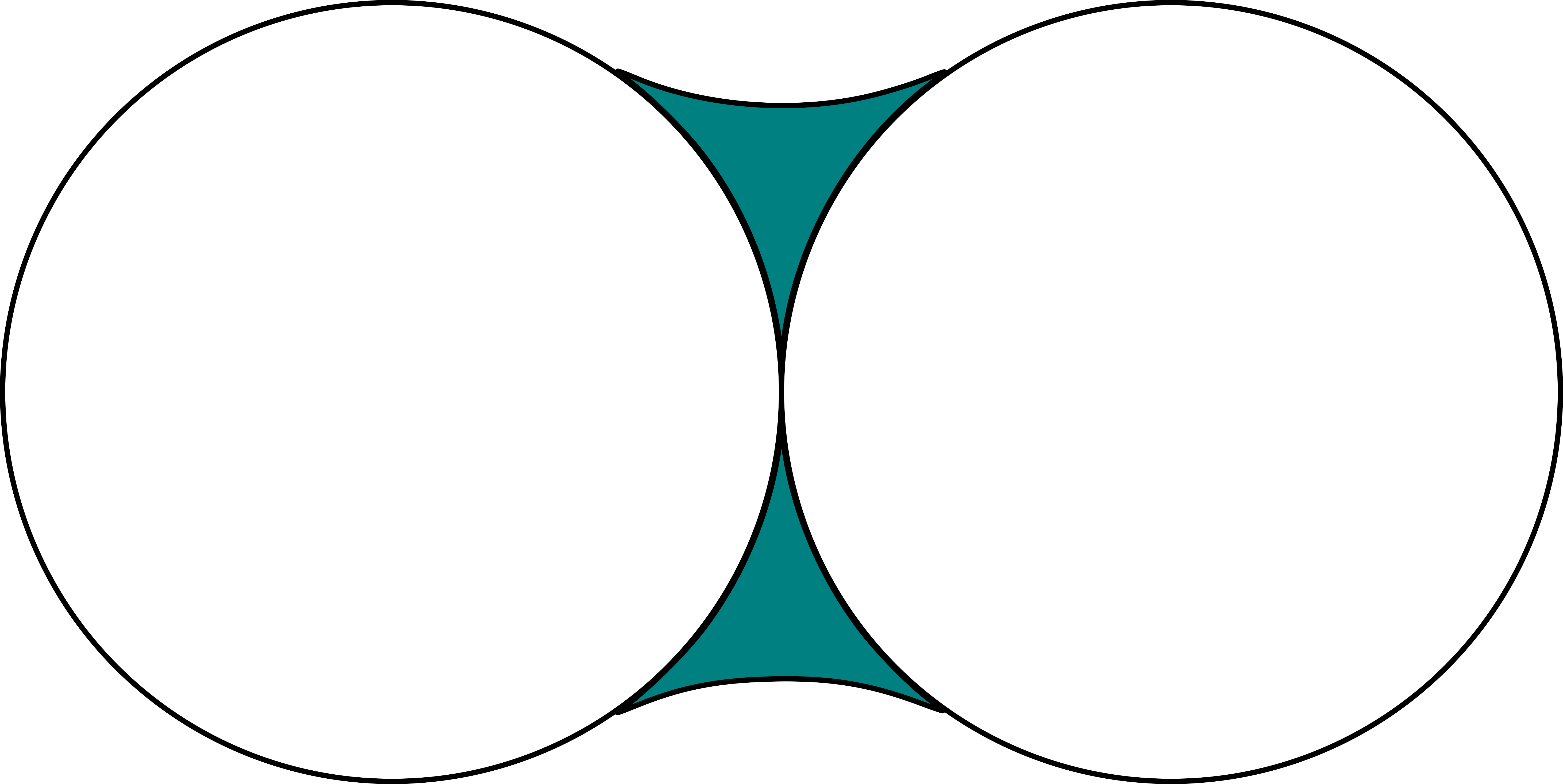}
    \caption{A set which satisfies $\cl_\eps(\op_\eps(A))=A$ but is not inner regular.}
    \label{fig:smooth_balls}
\end{figure}
\end{example}
\nc

\first
The previous example demonstrates that it is not possible to generate $\e$-regular sets by solely utilizing the opening and closing of a set. The example shown in \cref{fig:maximal-ex1} demonstrates that the maximal and minimal sets need not have boundaries that are even locally the graph of a function. Finally, revisiting \cref{fig:four-balls} in \cref{ex:four-balls} shows that in some cases there is no adversarial minimizer which is both $\e$ inner and outer regular: indeed, a minimizer of that problem can be at most $(\sqrt{2}-1)\e \approx 0.41 \e$ inner and outer regular. Hence, care must be taken in order to demonstrate the existence of a regular minimizer to the adversarial classification problem.

\first 
We now proceed to prove our central regularity result, \cref{thm:regular-sol}.
Note that, although generating an inner and outer regular minimizer through morphological operations is not possible, it \second is \nc plausible that one could construct a minimizer which is more regular (for example, possessing a $C^{1,1}$ boundary), but we leave that question to later work. 
\nc
\begin{proof}[Proof of \cref{thm:regular-sol}]

Again we abbreviate $\hat A = \cl_\eps(A)$ and $\tilde A = \op_\eps(A)$. 
We notice that $\hat A$ is $\eps$ outer regular, while $\tilde A$ is $\eps$ inner regular, and that $\tilde A \subseteq \hat A$. 
Thanks to \cref{prop:intermediate} any $B\in\B(\Rd)$ with $\tilde A \subseteq B \subseteq \hat A$ is a minimizer, see \cref{fig:sets} for an illustration.
We now turn to constructing $B$ with the desired properties.

We recall (cf.~\cite[Section 13.1]{leoni2017first} or originally in \cite{lieberman1985regularized}) that for any open set $V$ there exists a \emph{regularized (signed) distance function} $d_r \in C^{\infty}((\partial V)^c)$ satisfying
\begin{equation}\label{eqn:reg-dist-prop}
\frac{1}{2} \leq \frac{d_r(x)}{\bar d(x,V)} \leq \frac{3}{2},\qquad |\partial^\alpha d_r(x)| \leq \frac{c_\alpha}{|d_r(x)|^{|\alpha|-1}}.
\end{equation}
Here $\bar d(x,V)$ is the signed distance with respect to the Euclidean metric, namely
\[
\bar d(x,V) := \begin{cases}
\inf_{y \in V} \abs{x-y} &\text{ if } x \in V^c \\
-\inf_{y \in V^c} \abs{x-y} &\text{ if } x \in V
\end{cases}.
\]
As we will need to check a more detailed property of the function $d_r$, we briefly give its definition: We let $\phi\in C_c^\infty(\R^d)$ be a non-negative function with support on the unit ball and integral $1$. We define
\[
G(x,t) = \int_{\R^d} \bar d\left(x- t\frac{y}{2},V\right) \phi(y) \de y.
\]
One can show that there exists a unique solution to $G(x,t) = t$, and we let $d_r$ be that unique solution, namely $G(x,d_r(x)) = d_r(x)$. \second Indeed, as proved in \cite{lieberman1985regularized}, this follows from Banach fixed point theorem and the fact that $G(x, \cdot)$ is Lipschitz with Lipschitz constant strictly less than 1.  \nc

\begin{figure}[thb]
    \centering
    \begin{tikzpicture}
    \node[inner sep=0pt,anchor=center] (sets) at (0,0)
    {\includegraphics[width=.7\textwidth]{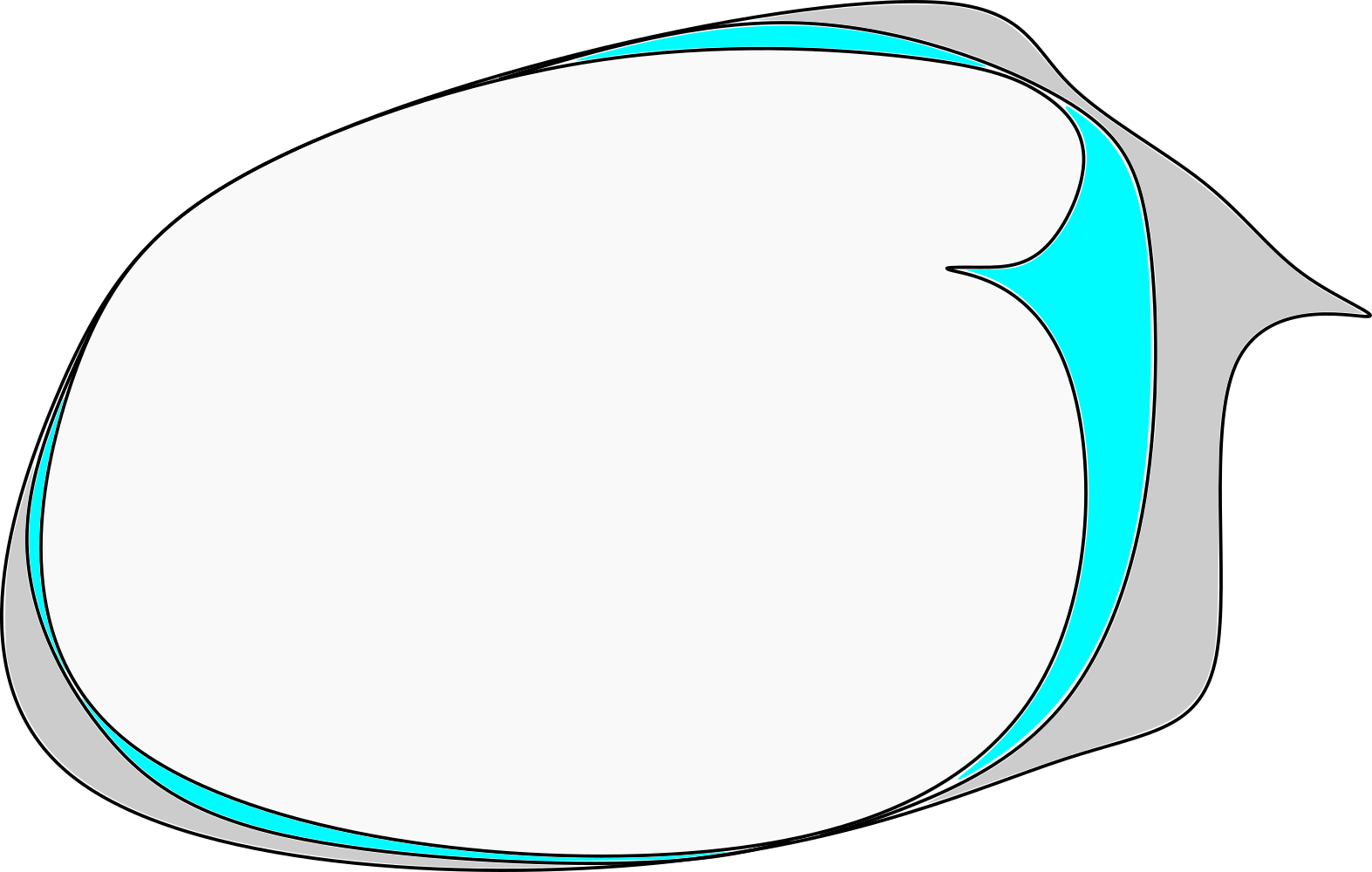}};
    \node at (sets) {\nc$\tilde A$};
    \node at (4.5,1.2) {\nc$\hat A$};
    \node at (3.2,1.2) {\nc$B$};
    \end{tikzpicture}
    \caption{Nested sets $\tilde A\subseteq B\subseteq\hat A$ in the proof of \cref{thm:regular-sol}. $\tilde A$ is inner regular and $\hat A$ is outer regular. The whole region between $\partial\tilde A$ and $\partial\hat A$ is $(\tilde A\cup\hat A^c)^c$ and the boundary of $B$ is by construction a smooth curve contained in that set.}
    \label{fig:sets}
\end{figure}
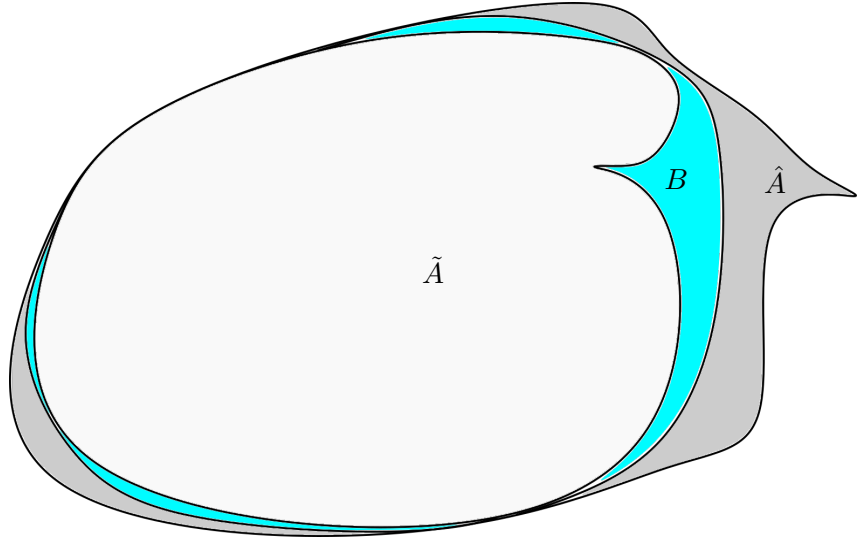

We let $\tilde d_1$ and $\tilde d_2$ be regularized distance functions for the sets $\tilde A$ and $\hat A^c$ respectively. By considering the function ${\tilde d_1}/{\tilde d_2}$ we may use Sard's theorem, which applies as the function is $C^\infty$, to find a $\kappa \in [1/2,2]$ so that $\kappa$ is a regular value of ${\tilde d_1}/{\tilde d_2}$ on the set $(\tilde A \cup \hat A^c)^c$.
here we recall that a regular value of a function is one so that the gradient does not vanish on the entire set $\{x\in\Rd\st{\tilde d_1(x)}/{\tilde d_2(x)} = \kappa\}$.
Our candidate set $B$ will now be $B:=\{x\in\Rd\st\tilde d_1(x) \leq \kappa \tilde d_2(x)\}$. 
Due to the first part of \labelcref{eqn:reg-dist-prop} we know that the signs of the original distance functions to the sets $\tilde A $, $\hat{A}^c$ and the signs of their regularized versions coincide\nc. From this observation it is now straightforward to see that $\tilde A \subseteq B \subseteq \hat A$, and hence $B$ is a minimizer of the adversarial problem according to \cref{prop:intermediate}: 
Thanks to the fact that $\kappa$ is a regular value, anywhere in the interior of  $(\tilde A \cup \hat A^c)^c$ we may express the boundary of $B$ as the graph of a $C^\infty$ function. In light of the main result in \cite{lewicka2020domains}, we also have that this set is locally the graph of a $C^{1,1}$ function away from $(\tilde A \cup \hat A^c)^c$. Thus it only remains to check the regularity up to the boundary points of $(\tilde A \cup \hat A^c)^c$.

To this end, we need to establish regularity estimates on $\nabla \tilde d_1$ and $\nabla \tilde d_2$ which hold uniformly at points on the boundary of $B$ near where the boundaries of $\tilde A$ and $\hat A$ coincide. To begin, we notice that
\[
\partial_i d_r(x) = \frac{\int_{\R^d} \partial_i \bar d\left(x-d_r(x)\frac{y}{2},V\right) \phi(y)\de y}{1 + \frac{1}{2}\int_{\R^d} \partial_i \bar d\left(x-d_r(x)\frac{y}2,V\right) \phi(y)\de y}.
\]
Using the fact that $|\nabla \bar d| \leq 1$ a.e. and that $z\mapsto\frac{z}{1+z/2}$ is uniformly Lipschitz on $[-1, \infty)$, 
it then suffices to estimate the continuity of $x\mapsto\int_{\Rd} \nabla \bar d(x-d_r(x)\frac{y}{2},V) \phi(y)\de y$. To this end, let us consider $x_1,x_2$ in the set where $1/2 < \frac{\tilde d_1}{\tilde d_2}<2$. For such points, let us denote
\[
D(x_1,x_2) =\min(\tilde d_1(x_1),\tilde d_2(x_1),\tilde d_1(x_2),\tilde d_2(x_2)).
\]
We notice that, by the choice of $x_1,x_2$, it holds that $D(x_1,x_2)\geq 0$ and
\begin{equation}\label{eqn:max-est}
2 D(x_1,x_2) \geq \max(\tilde d_1(x_1),\tilde d_2(x_1),\tilde d_1(x_2),\tilde d_2(x_2)).
\end{equation}
We consider separately two cases. First, if
\[
\eps^{\alpha}|x_1-x_2|^{\alpha} \leq D(x_1,x_2),
\]
we then use the classical estimate \labelcref{eqn:reg-dist-prop} to show that
\[
|\nabla d_r(x_1) - \nabla d_r(x_2)| \leq C\frac{|x_1-x_2|}{\min(\tilde d_1(x_1),\tilde d_2(x_1),\tilde d_1(x_2),\tilde d_2(x_2))} \leq C\eps^{-\alpha}|x_1-x_2|^{1-\alpha}.
\]
On the other hand, for the opposite case where
\[
\eps^\alpha|x_1-x_2|^{\alpha} \geq D(x_1,x_2),
\]
by using the estimate from \cref{lem:extended-spheres} below along with equation \labelcref{eqn:max-est} and the fact that $\phi$ has compact support we then may deduce that
\begin{align*}
|\nabla d_r(x_1) - \nabla d_r(x_2)| 
&\leq 
C
\left| 
\int_{\Rd} \nabla \bar d\left(x_1-d_r(x_1)\frac{y}{2},V\right) \phi(y)\de y 
- 
\int_{\Rd} \nabla \bar d\left(x_2-d_r(x_2)\frac{y}{2},V\right) \phi(y)\de y \right| 
\\
&\leq 
C
\int_{\Rd} 
\eps^{-1}\left(\abs{x_1-d_r(x_1)\frac{y}{2} - x_2+d_r(x_2)\frac{y}{2}} + 2 D(x_1,x_2) \right) \phi(y) \de y \\
&\leq C\eps^{-1} \left( |x_1 -x_2| + \sqrt{D(x_1,x_2)} \right) \\
&\leq C\eps^{-1} |x_1-x_2| + C\eps^{\alpha/2-1} |x_1-x_2|^{\alpha/2}.
\end{align*}
Setting $\alpha=2/3$ and applying the result to the regularized distance functions $\tilde d_1,\tilde d_2$ then establishes the fact that the function defining $\partial B$ is uniformly $C^{1,1/3}$, even up to the boundary, concluding the proof.
\end{proof}

We notice that in the previous proof the dependence on $\eps$ in the continuity estimates near the boundary is explicit, and improves as $\eps$ increases. This intuitively makes sense, and although our current estimates in the ``interior'' of our bad region do not give explicit dependence upon $\eps$, it seems plausible that the dependence on $\eps$ should be good. 

We now give the central geometric lemma used in the proof of \cref{thm:regular-sol}.

\begin{lemma}\label{lem:extended-spheres}
Let $\tilde A\subseteq\Rd$ be $\eps$ inner regular, let $\hat A\subseteq\Rd$ be $\eps$ outer regular, and let $\tilde A \subseteq \hat A$. Let $x,y\in\hat A\setminus \tilde A$, and let both be points of differentiability of the distance function from both $\tilde A$ and $\hat A$. Define
\[
D(x,y) := \max(d(x,\tilde A),d(x,\hat A),d(y,\hat A),d(y,\tilde A))
\]
Then
\[
|\nabla d(x,\tilde A) - \nabla d(y,\tilde A)| \leq C\eps^{-1}\left( |x-y| + \sqrt{D(x,y)} \right),
\]
at any points $x,y$ where the distance is differentiable (which holds a.e. by Rademacher's theorem).
\end{lemma}

\begin{proof}
This lemma is a direct extension of the work in \cite{lewicka2020domains} to our setting, and this proof expands upon the four ball lemma given therein. We recall that the gradient of the distance function is given by the unit vector pointing away from the closest point in the set. Let $u_x = \nabla d(x,\tilde A)$,  $u_y = \nabla d(y,\tilde A)$, $v_x = \nabla d(x,\hat A)$, and $v_y = \nabla d(y,\hat A)$. Let $\tilde x,\tilde y$ be the closest points in $\tilde A$ to $x$ and $y$, and let $\hat x,\hat y$ be the closest points in $\hat A$ to $x$ and $y$. By the regularity conditions, we know that there are four balls 
\begin{align*}
\tilde B_x = B_\eps(\tilde x-\eps u_x), \qquad
\tilde B_y = B_\eps(\tilde y-\eps u_y), \qquad
\hat B_x = B_\eps(\hat x-\eps v_x), \qquad
\hat B_y = B_\eps(\hat y-\eps v_y),
\end{align*}
which satisfy $\tilde B_i \cap \hat B_j = \emptyset$ for any $i,j \in \{x,y\}$, and so that $\tilde x,\tilde y$ do not belong to either $\tilde B_x$ or $\tilde B_y$. We also notice that $|\tilde x - \hat x| \leq 2\max(d(x,\tilde A),d(x,\hat A))$. 

We then choose the smallest positive values of $\tilde \delta_x$ and $\tilde \delta_y$ so that boundaries of the dilations $(\tilde B_x)^{\tilde \delta_x}$ and $(\tilde B_y)^{\tilde \delta_y}$ touch the boundaries  of either $\hat B_x$ or $\hat B_y$ at exactly one point. From here on we will assume that $(\tilde B_x)^{\tilde \delta_x}$ touches $\hat B_x$ and $(\tilde B_y)^{\tilde \delta_y}$ touches $\hat B_y$, as the other cases may be handled analogously. Call the points where those boundaries coincide $\bar x$ and $\bar y$. We note that clearly $\tilde \delta_x \leq 2 \max(d(x,\tilde A),d(x,\hat A)$ and $\tilde \delta_y \leq 2 \max(d(y,\tilde A),d(y,\hat A)$.

We may directly apply the four ball lemma from \cite{lewicka2020domains} to conclude that the unit vectors $\bar u_x,\bar u_y$ from the center of each ball to $\bar x$ and $\bar y$ satisfy
\begin{equation}\label{eqn:4-ball}
|\bar u_x - \bar u_y| \leq C\eps^{-1} |\bar x- \bar y|.
\end{equation}
It then remains only to bound the difference between the ``bar'' variables and the original ones.

Let $\bar u_x$ be the unit vector pointing from the center of $(\tilde B_x)^{\tilde \delta_x}$ to $\bar x$. We then may compute
\[
\cos( \theta(\bar u_x,u_x))(\eps + d(x,\tilde A)) + \cos( \theta(-\bar u_x,v_x)(\eps + d(x,\hat A)) = 2(\eps + \tilde \delta_x),
\]
where we are letting $\theta(\cdot,\cdot)$ denote the angle between the vectors. We may conclude that $\theta(-\bar u_x,v_x)$ and $\theta(\bar u_x,u_x)$ are bounded by a constant times $\sqrt{D(x,y)}$. By using the law of cosines we may compute that $|u_x-\bar u_x| < C\eps^{-1} \sqrt{D(x,y)}$ and that $|\bar x - x| < C\sqrt{D(x,y)}$. Using the triangle inequality and combining with \labelcref{eqn:4-ball} then concludes the proof.



\end{proof}

\section{Other Adversarial Models}
\label{sec:OtherModels}

We finish the paper with a couple of generalizations and a discussion on similar adversarial models, some of which also give rise to $L^1+\TV$ problems. In this section we keep the discussion rather formal in order to not distract from the main messages that we want to convey. We also do not make any attempt to interpret or expound upon these models: the goal is simply to identify alternative adversarial models which have analogous variational forms.

\subsection{Regression Problems}
Instead of studying binary classification one can also study adversarial regression problems of the form
\begin{align}
    \inf_{u\in L^1(\domain;\rho)}\Exp{(x,y)\sim\mu}{\sup_{\tilde x\in B_\eps(x)}\abs{u(\tilde x)-y}},
\end{align}
where $y$ can now take any real value.
Subtracting the empirical risk one can easily show that this problem can be reformulated as 
\begin{align}\label{eq:regression_model}
    \inf_{u\in L^1(\domain;\rho)}\Exp{(x,y)\sim\mu}{\abs{u(x)-y}}+\eps\PreTV_\eps(u;\mu),
\end{align}
where the total variation is now given by
\begin{align}\label{eq:TV_regression}
    \begin{split}
    \PreTV_\eps(u;\mu) 
    &= \frac{1}{\eps} \iint_{\domain\times\mathcal Y}\sup_{\tilde x\in B_\eps(x)}\left[\abs{u(\tilde x)-y}-\abs{u(x)-y}\right]\de\mu(x,y)
    \\
    &=
    \frac{1}{\eps}
    \int_\domain \int_{\pi_1^{-1}(x)}\sup_{\tilde x\in B_\eps(\xi)}\left[\abs{u(\tilde x)-y}-\abs{u(\xi)-y}\right]\de\mu_x(\xi,y)\de\rho(x).
    \end{split}
\end{align}
In the disintegrated formulation, $\pi_1:\domain\times\mathcal{Y}\to\domain$ denotes the projection onto the first factor, $\rho:=(\pi_1)_\sharp\mu$ is the first marginal of $\mu$, and $(\mu_x)_{x\in\domain}\subseteq\P(\domain\times\mathcal{Y})$ is a family of disintegrations of $\mu$.

As before, one has to define an essential version of this total variation to have a well-defined functional and introduce the measure $\nu$ as in \cref{sec:analysis}.
The analysis performed there can be generalized to this regression setting; however, in the regression context the interpretation of the associated perimeter is not obvious. 
Indeed, \labelcref{eq:TV_regression} is a highly data-dependent convex regularization functional. This provides theoretical motivation for recent work which studies data-driven convex regularizers for solving inverse problems; see, e.g., \cite{mukherjee2020learned}.
To make this connection a bit clearer we assume for simplicity that the data $y$ are given by $f(x)$.
In this case, \labelcref{eq:regression_model} reduces to the simpler formula
\begin{align}
    \inf_{u\in L^1(\domain;\rho)}
    \left\lbrace
    \int_\domain \abs{u(x)-f(x)}\de\rho(x) + \int_\domain \sup_{\tilde x\in B_\eps(x)}\left[\abs{u(\tilde x)-f(x)}-\abs{u(x)-f(x)}\right]\de\rho(x)
    \right\rbrace.
\end{align}

\subsection{Random Perturbation}

Let us consider random perturbations, i.e.,
\begin{align}\label{eq:random_pert}
   \inf_{A\in\B(\domain)}\Exp{(x,y)\sim\mu}{\Exp{\tilde x \sim \nu_{x}}{\abs{1_A(\tilde x)-y}}} 
\end{align}
where ``nature'' chooses $\tilde x$ randomly following the law of a family of probability measures $(\nu_{x})_{x\in\domain
}$. One natural candidate for such $\nu_x$ would be associated with a random walk \cite{mazon2020total}.
Since in this setting there is no adversarial attack in the game-theoretic sense, which would involve some sort of min-max structure, one cannot rewrite this problem as a variational regularization problem.
Indeed, subtracting the empirical risk $\Exp{(x,y)\sim\mu}{\abs{1_A(x)-y}}$ from the objective does not yield a non-negative term.

However, in the case that $\domain$ is a vector space, we can use the law of total expectation (disintegration) and a change of variables to obtain
\begin{align*}
    &\phantom{=}
    \Exp{(x,y)\sim\mu}{\Exp{\tilde x \sim \nu_{x}}{\abs{1_A(\tilde x)-y}}} 
    \\
    &= 
    w_0\int_\domain\int_{\domain}1_A(\tilde x)\de\nu_{x}(\tilde x)\de\rho_0(x) + 
    w_1\int_\domain\int_{\domain}1_{A^c}(\tilde x)\de\nu_{x}(\tilde x)\de\rho_1(x) 
    \\
    &=
    w_0\int_\domain\int_{\domain}1_A(x+\tilde x)\de(T_x)_\sharp\nu_{x}(\tilde x)\de\rho_0(x) + 
    w_1\int_\domain\int_{\domain}1_{A^c}(x+\tilde x)\de(T_x)_\sharp\nu_{x}(\tilde x)\de\rho_1(x),
\end{align*}
where $T_x:\domain\to\domain$ is defined by $T_x(\tilde x)=\tilde x-x$.
If the push forward measure $(T_x)_\sharp\nu_{x}$ does not depend on $x$, which is the case, e.g., whenever $\nu_x:=(T_{-x})_\sharp\nu$ for some measure $\nu$, we can abbreviate it by $\nu$, and we can rewrite this as:
\begin{align*}
    \Exp{(x,y)\sim\mu}{\Exp{\tilde x \sim \nu_{x}}{\abs{1_A(\tilde x)-y}}} 
    =
    w_0(\nu\star\rho_0)(A) + 
    w_1(\nu\star\rho_0)(A^c)
    = \Exp{(x,y)\sim \tilde\mu}{\abs{1_A(x)-y}}
\end{align*}
where the measure $\tilde\mu$ has the marginals $(\pi_1)_\sharp\tilde\mu=\nu\star\rho$ and $(\pi_2)_\sharp\tilde\mu=(\pi_2)_\sharp\mu$.

Hence, the random perturbation in \labelcref{eq:random_pert} does not actually lead to an adversarial model, but rather, it replaces the data distribution $\rho$ in the space $\domain$ with the convolution $\nu\star\rho$ and likewise changes the conditionals $\rho_0$ and $\rho_1$ to $\nu\star\rho_0$ and $\nu\star\rho_1$. We note that this structure is still similar to the form of $\tilde R_\eps$ shown in the proof of \cref{lem:clos_op}. Problems with similar structure have also been considered in the context of decentralized optimal control \cite{witsenhausen1968counterexample}.

\subsection{Random Perturbation with Adversarial Decision}

A similar model to the one from the previous section, however containing an adversarial action, is the following:
\begin{align}
\label{eqn:AdvProblem2}
    \inf_{A\in\B(\domain)}\Exp{(x,y)\sim\mu}{\Exp{\xi\sim\nu_{x,\eps}}{
    \max_{\tilde x\in\{\xi,x\}}\abs{1_A(\tilde x)-y}
    }}.
\end{align}
In words, in the above model ``nature'' randomly draws a point $\xi$ according to a probability measure (e.g., determined by a random walk) $\nu_{x,\eps}$, which can depend on $x$ and a parameter $\eps>0$, and the adversary can either use this proposed perturbation or reject it. The adversary's decision is, of course, based on whether the randomly chosen point $\xi$ creates a larger loss than the attacked point $x$.
For example, if the attacked point $x$ lies in~$A$ and this point should have the label~$1$, the adversary can do nothing if it draws another point in~$A$. Only if $\nu_{x,\eps}$ draws a point outside of $A$ will the adversary accept it. This adversarial model is reminiscent of \cite{kohn2006deterministic}, where mean curvature flow is obtained as a limit of a game theoretical problem in which an adversary chooses between alternatives.

As it turns out, problem \labelcref{eqn:AdvProblem2} can be rewritten as
\begin{align}
    \inf_{A\in\B(\domain)}\Exp{(x,y)\sim\mu}{\abs{1_A(x)-y}} + \eps\,\widehat{\TV}_\eps(1_A),
\end{align}
where the total variation functional $\widehat{\TV}_\eps$ takes the form:
\begin{align}
    \begin{split}
    \widehat{\TV}_\eps (u) 
    &:= 
    \frac{w_0}{\eps}\int_{\domain}\int_{\domain}
    \left(u(\tilde x) - u(x)\right)_+  \de \nu_{x,\eps}(\tilde x)   \de \varrho_0(x)
    \\
    &\qquad
    +
    \frac{w_1}{\eps}
    \int_{\domain}\int_{\domain}\left(u(x) - u(\tilde x)\right)_+ \de \nu_{x,\eps}(\tilde x) \de  \varrho_1(x).
    \end{split}
\end{align}
Since here the adversarial decision takes place on the finite set $\{\xi,x\}$, this problem is much easier to analyze and does not require a redefinition using an essential total variation or perimeter.
For instance, if $\nu_{x,\eps}\ll\rho$ for all $x\in\domain$ then one can simply work on $L^\infty(\domain;\rho)$.
Indeed this is a highly relevant case as the following example shows.
\begin{example}
If $\varrho_0=\varrho_1$, $w_0 = w_1$ then this reduces to the following nonlocal total variation energy
\begin{align*}
    \frac{1}{\eps}\int_{\domain}\int_{\domain} |u(x) - u(\tilde x)|\de \nu_{x,\eps}(\tilde x) \de \varrho(x),
\end{align*}
whose properties and associated gradient flow have been analyzed in the framework of metric random walk spaces \cite{mazon2020total}.
This nonlocal total variation functional has furthermore been extensively applied in image processing, see \cite{gilboa2009nonlocal,zhang2010bregmanized}.

If we assume that the random walk $\nu_{x,\eps}$ has the special structure $\frac{\de\nu_{x,\eps}}{\de\rho}(\tilde x)=\eta_\eps(x-\tilde x)$ for some function $\eta_\eps:\domain\to\R$, we obtain
\begin{align}\label{eq:nonlocal_TV}
   \frac{1}{\eps}\int_{\domain}\int_{\domain} \eta_\eps(x-\tilde x) |u(x) - u(\tilde x)|\de \varrho(\tilde x) \de \varrho(x).
\end{align}
For the special case when $\rho=\frac{1}{N}\sum_{i=1}^N\delta_{x_i}$ is an empirical measure, \labelcref{eq:nonlocal_TV} reduces to the graph total variation
\begin{align}
    \frac{1}{\eps}\sum_{i,j=1}^N \eta_\eps(x_i-x_j) |u(x_i) - u(x_j)|.
\end{align}
Total variations of these forms and their limits as $\eps\to 0$ have been intensively analyzed in the context of graph-based clustering methods and trend filtering, see, e.g., \cite{GarciaTrillos2016,trillos2016consistency,garciatrillos_murray_2017}. Typical choices for $\eta_\eps$ are $\eta_\eps( z)=\frac{1}{\eps^d} 1_{B_\eps(x)}(z)$ or $\eta_\eps( z)=\frac{1}{\eps^d}\exp(-\abs{z/\eps}^2)$.
\end{example}

\subsection{General Loss Functions}

One can also study adversarial problems with a more general loss function.
The baseline model for this endeavour is the following generalization of \labelcref{eq:baseline_adv}:
\begin{align}\label{eq:adv_general_loss}
    \inf_{A\in\B(\domain)}\Exp{(x,y)\sim\mu}{\sup_{\tilde x\in B_\eps(x)}\loss(1_A(\tilde x),y)}.
\end{align}
Subtracting the empirical risk, we can decompose the adversarial risk as
\begin{align}
    \Exp{(x,y)\sim\mu}{\sup_{\tilde x\in B_\eps(x)}\loss(1_A(\tilde x),y)}
    =
    \Exp{(x,y)\sim\mu}{\loss(1_A(x),y)}
    +
    \eps\PreTV_\eps(1_A;\mu),
\end{align}
where the total variation is given by
\begin{align}\label{eq:TV_general}
    \begin{split}
    \PreTV_\eps(u;\mu) 
    &=
    \frac{w_0}{\eps}
    \int_{\domain}
    \sup_{\tilde x\in B_\eps(x)}\loss(u(\tilde x),0) 
    - \loss(u(x),0)\de\rho_0(x)
    \\
    &\qquad
    +
    \frac{w_1}{\eps}
    \int_{\domain}
    \sup_{\tilde x\in B_\eps(x)}\loss(u(\tilde x),1) 
    - \loss(u(x),1)\de\rho_1(x).
    \end{split}
\end{align}
For instance, for the cross entropy loss $\ell(u,y)=-y\log u - (1-y) \log(1-u)$ this simplifies to
\begin{align}\label{eq:TV_CE}
\begin{split}
    \PreTV_\eps(u;\mu) 
    &=
    \frac{w_0}{\eps}
    \int_{\domain} \log (1-u(x)) - \inf_{\tilde x\in B_\eps(x)}\log (1-u(\tilde x))
    \de\rho_0(x) \\
    &\qquad
    +
    \frac{w_1}{\eps}
    \int_{\domain}
    \log u(x) - \inf_{\tilde x\in B_\eps(x)}\log u(\tilde x)
    \de\rho_1(x)
    ,\quad  \;0\leq u\leq 1,
    \end{split}
\end{align}
which is similar to our total variation \labelcref{eq:PreTV} applied to $\log u$ instead of $u$. Notice that in general problem \labelcref{eq:adv_general_loss} and its relaxation to functions $0\leq u \leq 1$ may not coincide, as the cross entropy example suggests. 
For general loss functions \labelcref{eq:TV_general} is more difficult to interpret: this is a primary reason that we restricted our analysis to the case $\loss(u,y)=\abs{u-y}$.

\section{Conclusions}
\label{sec:Conclusions}
 
In this paper we have studied adversarial training problems in a variety of non-parametric settings and have established an equivalence with regularized risk minimization problems. The regularization terms in these risk minimization problems are explicitly characterized and correspond to a type of nonlocal perimeter/total variation. Our work provides new conceptual insights for adversarial training problems, and introduces new mathematical tools for their quantitative analysis. In particular, we have used tools from the calculus of variations to rigorously prove the existence of solutions, we have identified a convex structure of the problem that allows us to introduce appropriate notions of maximal and minimal solutions and in turn introduced a convenient notion of uniqueness of solutions, and finally, we have presented a collection of results on the existence of regular solutions to the original adversarial training problem.

Some research directions that stem from this work include: 1) the extension of the analysis presented in this work to multi-label classification settings, 2) investigating a sharper analysis of the regularity properties of solutions to adversarial training problems in both the Euclidean setting, as well as for more general distance functions and spaces. 

In addition, as already discussed in the introduction, part of the motivation for this work came from the work \cite{MurrayNGT} where one of the main objectives was to study the regularization effect of adversarial training on the decision boundaries of optimal robust classifiers (starting with the Bayes classifier at $\eps=0$). The structure of solutions studied in the present paper allows us to make the line of work initiated in \cite{MurrayNGT} more concrete, and to approach it with a larger set of mathematical tools at hand. In particular, this work raises the question of whether it is possible to track maximal (minimal) solutions to adversarial training problems as $\eps$ grows from $0$ to infinity. In words, we are interested in defining a suitable notion of \textit{solution path} $(A_\eps)_{\eps>0}$ for the family of adversarial training problems \labelcref{eq:baseline_adv}. The study of solution paths, in particular their algorithmic use and regularity, has quite some tradition in the field of variational regularization methods, see, e.g., \cite{tibshirani2011solution,bungert2019solution,rosset2007piecewise}, but in the context of adversarial training less is known about their properties. Notice that one important difference with the standard regularization setting is that the equivalent regularization formulation of \labelcref{eq:baseline_adv} has a regularization functional that changes with the regularization parameter $\eps$. This feature makes the analysis more challenging. 

It is also interesting to consider the asymptotics as $\eps \downarrow 0$. In the special case where $(\domain,\metric)$ is $\R^d$ with the Euclidean metric and $w_i\rho_i$ \first is replaced by \nc $\mathcal{L}^d$, the functionals $\Per_\eps(\cdot;\mu)$  are known to $\Gamma$-converge to the classical perimeter as $\eps \downarrow 0$ \cite{chambolle2014remark}.
In our more general setting it is particularly interesting to investigate which information of the measures $\rho_0$ and $\rho_1$ ``survives'' in the limit as $\eps\downarrow 0$ and whether a $\Gamma$-convergence result can be proven.
\first
Finally, in this regime the proof of our regularity result \cref{thm:regular-sol} deteriorates and one can at most expect a set of finite perimeter.
\nc

\section*{Acknowledgments}
The authors would like to thank Antonin Chambolle, Matt Jacobs, Meyer Scetbon, Simone Di Marino and Khai Nguyen for enlightening discussions and for sharing useful references. 
This work was done while LB and NGT were visiting the Simons Institute for the Theory of Computing to participate in the program ``Geometric Methods in Optimization and Sampling'' during the Fall of 2021 and LB and NGT are very grateful for the hospitality of the institute. LB acknowledges funding by the Deutsche Forschungsgemeinschaft (DFG, German Research Foundation) under Germany's Excellence Strategy - GZ 2047/1, Projekt-ID 390685813. 
NGT was partially supported by NSF-DMS grant 2005797 and would also like to thank the IFDS at UW-Madison and NSF through TRIPODS grant 2023239 for their support.

\begin{appendix}
\editscolor
\section{Technical Definitions}
\label{sec:appendix_technical_defs}

In this section we provide various technical definitions used throughout this work.
\nc

\third 
\begin{definition}[H\"older spaces and sets]
Let $U\subseteq\R^d$ be an open subset of $\R^d$.
For $0 < \alpha \leq 1$ a function $f:U\to\R$ is called $\alpha$-Hölder continuous if 
\begin{align*}
    \sup\left\lbrace\frac{\abs{f(x)-f(y)}}{\abs{x-y}^\alpha} \st x,y\in U,\,x\neq y\right\rbrace < \infty.
\end{align*}
For $k \in \N$, a function $f:U\to\R$ is said to belong to $C^{k,\alpha}(U)$ if it is $k$ times differentiable on $U$ and its $k-$th derivatives are $\alpha$-H\"older continuous on $U$.
We say that an open subset of $\R^d$ belongs to $C^{k,\alpha}$ if it can be locally represented as the subgraph of a $C^{k,\alpha}(U)$ function.
\end{definition}
\nc 
\first
\begin{definition}[Push-forward measure]
Let $(X_1,\Sigma_1), (X_2,\Sigma_2)$ be two measurable spaces and let $f:X_1 \to X_2$ be measurable. Given a measure $\mu$ on $X_1$ we define the push-forward measure on $X_2$ by the formula
\[
f_\sharp \mu (B) := \mu(\{ x \in X_1 : f(x) \in B\}),
\]
where $B$ is an arbitrary set in $\Sigma_2$.
\end{definition}

The next proposition is a classical result in measure theory, and may be found, e.g. in \cite{bogachev2007measure}, Section 3.1.

\begin{proposition}[Hahn decomposition]
Let $\mu$ be a signed measured on a measure space $(X,\Sigma)$. Then there exists two measurable sets $P,N$ so that $P\cup N = X$, $P \cap N = \emptyset$ and so that $\mu(E) \geq 0$ for all $E \subseteq P$ and $\mu(F) \leq 0$ for all $F \subseteq N$.
\end{proposition}
\nc 
\editscolor
\begin{theorem}[Lebesgue differentiation theorem, see e.g. Section 3.4 in \cite{heinonen2015sobolev}]\label{thm:Lebesgue}
Let $(\domain,\metric,\nu$) be a metric measure space with $\nu$ satisfying \labelcref{eq:local_doubling}, and let $f \in L^1(\domain,\nu)$. Then for $\nu$-almost every $x\in\domain$ we have that
\[
\lim_{r \downarrow 0} \frac{1}{\nu({B(x,r)})} \int_{{B(x,r)}} f(y) \de\nu(y) = f(x).
\]
\end{theorem}

\begin{definition}[Weak-* convergence and compactness]
Let $X$ be a Banach space of $\R$ with dual $X^*$.
We say that a sequence $(y_k)_{k\in\N}\subseteq X^*$ is weak-* convergent to $y\in X^*$ if
\begin{align*}
    \lim_{k\to\infty}y_k(x) = y(x),\qquad\forall x\in X.
\end{align*}
\end{definition}
\begin{theorem}[Banach--Alaoglu]\label{thm:Banach-Alaoglu}
Let $X$ be a Banach space of $\R$ with dual $X^*$.
Then any bounded subset of $X^*$ is precompact in the weak-* topology.
\end{theorem}
\nc

\section{Alternative Formulations of the Adversarial Problem}

\editscolor
At this point we review a few other established formulations of the adversarial problem and add pointers towards the relevant literature. These reformulations provide different ways of understanding and analyzing the original adversarial problem. 
\nc 

\editscolor

\third 
\subsection{Open vs Closed Balls}
\label{sec:open_vs_closed_balls}
We would like to continue the discussion in \cref{rem:previous} and elaborate on why the adversarial model with open balls that we study here does not require the universal $\sigma$-algebra. We observe that the closed norm balls model which was considered in \cite{awasthi2021existence_extended,pydi2019adversarial} suffers from the problem that
\begin{align*}
    \sup_{\closure{B}_\eps(x)}1_{A} = 1_{A^{\closure\oplus\eps}},
\end{align*}
where the set $A^{\closure\oplus\eps}:=\bigcup_{x\in A}\closure{B}_\eps(x)$ is in general not Borel measurable even though $A$ might be.
Here $\closure B_\eps(x):=\{y\in\domain\st\metric(x,y)\leq\eps\}$ denotes the closed $\eps$-ball around $x\in\domain$.
One can sandwich $A^{\closure\oplus\eps}$ between open and closed parallel sets (in particular Borel sets) like this:
\begin{align*}
    \{x\in\domain\st\dist(x,A)<\eps\} \subset A^{\closure\oplus\eps}\subset \{x\in\domain\st\dist(x,A)\leq \eps\},
\end{align*}
but the inclusions may be strict in general, see \cite{pydi2019adversarial}. The situation is markedly different for open balls, where one has
\begin{align*}
    \sup_{{B}_\eps(x)}1_{A} = 1_{A^{\oplus\eps}},
\end{align*}
where $A^{\oplus\eps}:=\bigcup_{x\in A}{B}_\eps(x)$ is an open set and satisfies the following:
\begin{lemma}
It holds that
\begin{align*}
    \{x\in\domain\st\dist(x,A) < \eps\} = \bigcup_{x\in A}B_\eps(x).
\end{align*}
\end{lemma}
\begin{proof}
Let $y\in\domain$ such that $\dist(y,A)<\eps$. 
Then there exists a sequence of points $(x_k)_{k\in\N}\subset A$ with $\lim_{k\to\infty}\metric(y,x_k) = \dist(y,A) < \eps$.
Hence, there exists $K\in\N$ such that for all $k\geq K$ it holds $\metric(y,x_k) < \eps$ and therefore
\begin{align*}
    y \in \bigcup_{k\geq K}B_\eps(x_k) \subset \bigcup_{x\in A}B_\eps(x).
\end{align*}
This establishes the inclusion ``$\subset$''.
For the converse inclusion, let $y\in\bigcup_{x\in A}B_\eps(x)$.
Then there exists $x\in A$ such that $y\in B_\eps(x)$ and therefore
\begin{align*}
    \dist(y,A) \leq \metric(y,x) < \eps,
\end{align*}
which establishes the inclusion ``$\supset$'' and concludes the proof.
\end{proof}
\nc

\subsection{\texorpdfstring{$\infty$}{Infinity}-Wasserstein DRO Problem}
\label{sec:DROReformulation}

It is well-known \cite{pydi2019adversarial} that the closed ball adversarial problem is indeed a DRO problem in the form of \labelcref{Robust problem:Intro} with respect to a special $\infty$-Wasserstein distance.
To this end we introduce an $\infty$-Wasserstein distance between two measures $\mu$ and $\tilde\mu$ as
\begin{align}\label{eq:inf-Wasserstein_dist}
    W_\infty(\mu, \tilde \mu):= \inf_{\pi \in \Gamma(\mu, \tilde \mu)} \esssup[\pi] c_\infty,
\end{align}
where the cost function is given by
\begin{subequations}
\begin{align}
    c_\infty 
    &:\big(\domain\times\{0,1\}\big)^2\to[0,+\infty],\\
    c_\infty\big((x,y), (\tilde x,\tilde y)\big) 
    &:= 
    \begin{cases} 
    \metric(x,\tilde x) \quad &\text{ if } y = \tilde y, \\
    +\infty \quad &\text{ if } y \not = \tilde y. \end{cases}
\end{align}
\end{subequations}

\second 
\begin{proposition}[{\cite[]{pydi2019adversarial}}]
Let $\domain$ be Polish.
Then the adversarial risk of $A\in\B(\domain)$ can be reformulated as
\begin{align}\label{eq:baseline_adv_wasserstein}
    \Exp{(x,y)\sim\mu}{\sup_{\tilde x \in \closure B_\eps(x)}\abs{1_A(\tilde x) - y}}
    =
    \sup_{\substack{\tilde \mu\in\P(\domain\times\{0,1\})\\ W_\infty(\mu, \tilde \mu ) \leq \eps}} \Exp{(x,y)\sim\tilde \mu}{ \abs{1_A(x)-y}}.
\end{align}
\end{proposition}
\nc 

\subsection{Dual of an Optimal Transport Problem}
\label{sec:OTReformulation}

It is also known \cite{BhagojietAl,pydi2019adversarial,MurrayNGT} that the adversarial problem may be reformulated as the dual of an optimal transport problem. 
The following result is stated in the setting $\domain = \R^d$ endowed with the Euclidean distance but can be generalized to arbitrary metric spaces in a straightforward way.  
Let $\mu^S$ be the probability distribution on $\domain \times \{ 0,1 \}$ defined as:
\[ \mu^S:= T^S_{\sharp } \mu, \quad \text{ where } T^S(x,y):=(x,1-y), \quad \forall (x,y) \in \domain \times \{0,1 \}.\]
The map $T^S$ can be interpreted as a transformation that leaves features unchanged while swapping labels. The following statement is proven in \cite{MurrayNGT}.
\third 
\begin{proposition}[{cf. \cite[Corollary 3.2]{MurrayNGT}}]
\label{prop:OTDual} 
Let $c_{\eps}: (\R^d \times \{ 0,1\})^2 \rightarrow \R$ be the function defined by
\[  c_{\eps}(z_1, z_2):= {1}_{\{|x_1-x_2|> 2\eps \} \cup \{ y_1\not =y_2\} }, \]
where we write $z_i = (x_i,y_i)$.
Then,
\begin{equation}
\inf_{A\in\B(\domain)}\Exp{(x,y)\sim\mu}{\sup_{\tilde x \in \closure B_\eps(x)}\abs{1_A(\tilde x) - y}} =  \frac{1}{2} - \frac{1}{2} \inf_{\pi \in \Gamma(\mu, \mu^S)} \iint_{(\R^d \times \{0,1 \})^2}c_\eps(z_1, z_2) \de\pi(z_1,z_2). 
\label{eqn:OTFormulation}
\end{equation}
\end{proposition}
\nc 
This type of result was first established independently in \cite{BhagojietAl,pydi2019adversarial} where the balanced case $w_0=w_1=1/2$ was considered. The OT problem on the right hand side of \labelcref{eqn:OTFormulation} is an alternative way to compute the optimal adversarial risk. This alternative has clear advantages over the original formulation of the problem in situations like when $\mu$ is an empirical measure (the standard setting in practice). Indeed, in that setting, problem \labelcref{eq:baseline_adv} is in principle an infinite dimensional problem, while the OT problem will always be finite dimensional. One may speculate further and wonder whether there is a connection between solutions to the OT problem and optimal adversarially robust classifiers, or in other words, whether one can construct adversarially robust classifiers from a solution to the OT problem. This is indeed the case and such results will be elaborated on in future work. \nc

\section{Additional Proofs}
\label{sec:appendix_proofs}

Here we prove \cref{lem:ModifyFcts}, which we restate for convenience.
\begin{lemma}
Under \cref{ass:nu_eps} for any Borel measurable function $u\in L^\infty(\domain;\nu)$ there exists $u^\star:\domain\to\R$ such that $u=u^\star$ holds $\nu$-almost everywhere and
\begin{equation}
    \sup_{B_\eps(x)} u^\star = \esssup[\nu]_{B_\eps(x)} u^\star , \quad  \inf_{B_\eps(x)} u^\star  = \essinf[\nu]_{B_\eps(x)} u^\star , \quad \forall x \in \supp\rho.  
    \label{eqn:PropertySupEqualEsssupFunctions}
\end{equation}
\end{lemma}

\begin{proof}

\textbf{Step 1:} Let $t_1, t_2, t_3, \dots$ be an enumeration of the rational numbers. In what follows we construct a collection of measurable sets $A^\star_{t_1}, A^\star_{t_2},A^\star_{t_3}, \dots $ satisfying the following properties:
\begin{enumerate}
    \item  For every $k$, we have $\{ u \geq t_k \} \sim_{\nu} A^\star_{t_k} $.
    \item For every $k$, $A^\star_{t_k}$ satisfies \labelcref{eqn:PropertySupEsssup}.
    \item For any two $k \not =l$, if $t_k <t_l$, then $1_{A^\star_{t_l}}(x) \leq 1_{A^\star_{t_k}} (x)$ for every  $x \in \supp\nu$.
\end{enumerate}
We construct these sets inductively. First, following the proof of \cref{lem:ModifySets} applied to the set $A= \{ u \geq t_1 \}$ we obtain the set $A^\star_{t_1}$ defined through its indicator function according to
\[ 1_{A^\star_{t_1}}(x) := \begin{cases} 1 \quad &\text{ if }  x \in D_+^\eps(t_1) \\ 0 \quad &\text{ if } x \in D_- ^\eps(t_1)  \\ 1_{\{u\geq t_1\}} \quad &\text{ if } x \in \R^d \setminus  (D_+^\eps(t_1) \cup D_-^\eps(t_1) ). \end{cases}\]
In the above, we use the notation $D_+^\eps(t_1)$, $D_-^\eps(t_1)$, as well as the notation  $D_+(t_1)$, $D_-(t_1)$, to denote the sets introduced in the proof of \cref{lem:ModifySets} emphasizing that these sets are associated to the set $\{ u \geq t_1\}$. 

Now, suppose that we have constructed the sets $A^\star_{t_1}, \dots, A^\star_{t_L}$ (in terms of associated sets $D_+(t_l), D_-(t_l),D^\eps_+(t_l), D^\eps_-(t_l)  $ for every $l=1, \dots, L$) and suppose that these sets satisfy (1)-(3) when restricted to $k,l\in \{ 1, \dots, L\}$. 
We now discuss how to construct the set $A^\star_{t_{L+1}}$.  Suppose that $ t_k < t_{L+1} < t_l$ for some $k, l \in \{ 1, \dots, L \}$ and suppose that these indices are chosen so that there is no element in $\{ t_1, \dots, t_L \}$ strictly between $t_k$ and $t_l$ (if $t_{L+1}$ was bigger, or smaller, than all the $t_k$ with $k=1, \dots, L$, a similar construction to the one we exhibit next would apply and because of this we focus on the case mentioned earlier for brevity). 
We start by defining
\[ 1_{\tilde A_{t_{L+1}}}(x) := \begin{cases} 1 \quad &\text{ if }  x \in D_+^\eps(t_{L+1}) \\ 0 \quad &\text{ if } x \in D_- ^\eps(t_{L+1})  \\ {1}_{\{ u \geq t_{L+1} \}}(x) \quad &\text{ if } x \in \R^d \setminus  (D_+^\eps(t_{L+1}) \cup D_-^\eps(t_{L+1}) ), \end{cases}\]
obtained following the construction in \cref{lem:ModifySets} when applied to the set $\{ u \geq t_{L+1}\}$. 
We now modify$\tilde A_{t_{L+1}}$ slightly to eventually satisfy property (3). 
Indeed, since $ \{ u \geq t_l\}\subseteq \{ u \geq t_{L+1} \} \subseteq \{ u \geq t_k  \}$ and $ A^\star_{t_k} \sim_{\nu} \{ u \geq t_k \}$, $ A^\star_{t_l} \sim_{\nu} \{ u \geq t_l \}$, $ \tilde A_{t_{L+1}} \sim_{\nu} \{ u \geq t_{L+1} \}$ we conclude that
\[  1_{A^\star_{t_l}}(x) \leq 1_{ \{ u \geq t_{L+1} \}}(x) \leq 1_{A^\star_{t_k}}(x) \]
for every $x \in \R^d \setminus \mathcal{N}_{L+1} $
where $\mathcal{N}_{L+1}$ is some ${\nu}$-null set. 
We then define
\[ 
1_{ A^\star_{t_{L+1}}}(x) := 
\begin{cases}  
1_{A^\star_{t_l}}(x) \quad &\text{if } x \in (\R^d \setminus  (D_+^\eps(t_{L+1}) \cup D_-^\eps(t_{L+1}) ))\cap \mathcal{N}_{L+1}\\
1_{\tilde A_{t_{L+1}}}(x) \quad &\text{otherwise}.
\end{cases}
\]
It is evident that $A^\star_{t_{L+1}}$ is ${\nu}$-equivalent to $\{ u \geq t_{L+1}  \}$ and that it satisfies \labelcref{eqn:PropertySupEsssup} (since we do not modify the $\tilde A_{t_{L+1}}$ inside $D_\eps^+(t_{L+1})$ or $D_{\eps}^-(t_{L+1})$). 
It remains to show that for every $x \in \supp{\nu}$ we have $  1_{A^\star_{t_l}}(x) \leq 1_{ A^\star_{t_{L+1}}}(x) \leq 1_{A^\star_{t_k}}(x) $. By definition of $A^\star_{t_{L+1}}$ and the relation between $A^\star_{t_k}$ and $A^\star_{t_l}$ the inequality is immediate if $x \in \R^d \setminus  (D_+^\eps(t_{L+1}) \cup D_-^\eps(t_{L+1}) )$.
Thus it suffices to show the inequality when $x\in D_+^\eps(t_{L+1})$ (as the case $x\in D_-^\eps(t_{L+1})$ is completely analogous). 
In turn, we just have to show that $ 1_{A^\star_{t_k}}(x)=1 $. 
Now, notice that $x \in D_+^\eps(t_{L+1})$ means that there is $x_1 \in \supp\rho$ such that $x_1 \in D_+(t_{L+1})$ and $\metric(x, x_1) <\eps$. 
In particular, ${1}_{ \{u\geq t_{L+1}\}}(\tilde x) =1$ for ${\nu}$-a.e. $\tilde x \in B_\eps(x_1)$. Given that $ \{ u \geq t_{L+1} \} \subseteq \{ u \geq t_k \}$, we also have that ${1}_{ \{u\geq t_{k}\}}(\tilde x) =1$ for ${\nu}$-a.e. $\tilde x \in B_\eps(x_1)$, and hence $x_1 \in D_+(t_k)$. We conclude that $x \in D_{+}^\eps(t_k)$ and in turn that $1_{A^\star_{t_k}}(x) =1$. 

\textbf{Step 2:} Using the family of sets $A^\star_{t_1}, A^\star_{t_2}, A^\star_{t_3}, \dots$ we construct the function $u^\star $ according to
\[ u^\star (x):= \sup \{ t \in \Q \st x \in A^\star_t  \}. \]

We now show the following relations between level sets of $u^\star $ and the sets $A^\star_{t}$: for every $t \in \Q$ we have
\begin{equation}
 1_{A^\star_t}(x) \leq  1_{\{ u^\star  \geq t \}}(x), \quad \forall x \in \supp{\nu},
 \label{eqn:Inclusion1}
\end{equation}
and for every $s,t\in \Q$ with $s< t$ we have
\begin{equation}
  1_{\{ u^\star  \geq t \}}(x) \leq  1_{A^\star_s}(x), \quad \forall x \in \supp{\nu}. 
 \label{eqn:Inclusion2}
\end{equation}

Inequality \labelcref{eqn:Inclusion1} follows from the fact that if $x \in A^\star_{t}$, then by definition of $u^\star $ we have $u^\star (x) \geq t$. So in this case we actually have the stronger condition $A^\star_t \subseteq \{u^\star  \geq t \}$.

To obtain inequality \labelcref{eqn:Inclusion2} take $x \in \supp{\nu}$ such that $u^\star (x) \geq t$. Then there must exist a rational $r \geq s$ such that $x \in A^\star_r$ for otherwise $u^\star (x)$ would be less than $s$. From property (3) of the sets $A^\star$ we deduce that $x \in A^\star_s$ also.

\textbf{Step 3:} We now show that $u^\star $ satisfies \labelcref{eqn:PropertySupEqualEsssupFunctions}. To see this, let $x \in \supp\rho$ and suppose for the sake of contradiction that $\sup_{B_\eps(x)} u^\star  > \esssup[\nu]_{B_\eps(x)} u^\star $. Pick $t \in \Q$ strictly between these two values. Then, there is $\tilde x \in B_\eps(x)$ (notice that $\tilde x \in \supp{\nu}$) such that $u^\star (\tilde x) \geq  t $. In particular, from \labelcref{eqn:Inclusion2} it follows that ${1}_{A^\star_s}(x) =1$ for a rational $s$ with $s<t$ that is also strictly larger than $\esssup[\nu]_{B_\eps(x)} u^\star $, and in turn we deduce that $\sup_{B_\eps(x)} {1}_{A^\star_s} =1$. On the other hand, from the fact that $\esssup[\nu]_{B_\eps(x)} u^\star  < s $, it is clear that ${\nu}(\{ u^\star  \geq s \} \cap B_\eps(x))=0$, and thus if we combine with  \labelcref{eqn:Inclusion1} we deduce that ${\nu}(A^\star_s \cap B_\eps(x))=0$ also. This means that $ \esssup[\nu]_{B_\eps(x)} {1}_{A^\star_s} =0$, contradicting in this way property (2) for the set $A^\star_s$. In conclusion: for every $x \in \supp\rho$ we have $\sup_{B_\eps(x)} u^\star  = \esssup[\nu]_{B_\eps(x)} u^\star $.

To show the second part of \labelcref{eqn:PropertySupEqualEsssupFunctions} we follow a similar strategy. Namely, suppose for the sake of contradiction that there is $x \in \supp\rho$ such that $ \inf_{B_\eps(x)} u^\star  < \essinf_{B_\eps(x)} u^\star $. Let $s$ be a rational number strictly between these two values. Then there is $\tilde x \in B_\eps(x)$ such that $u^\star (\tilde x) < s$ and thus we must have $\tilde x \notin A^\star_s$. In particular, we have $\inf_{B_\eps(x)} 1_{A^\star_s} =0$. Picking now a rational $t$ strictly between $s$ and $\essinf_{B_\eps(x)} u^\star $ we conclude that $ {1}_{\{ u^\star  \geq t \}} =1$ ${\nu}$-a.e. $\tilde x \in B_\eps(x)$. By \labelcref{eqn:Inclusion2} the same is true when we replace $\{ u^\star  \geq t \}$ with $A^\star_s$. Thus, we conclude that $\essinf_{B_\eps(x)} 1_{A^\star_s} =1$, contradicting property (2) for the set $A^\star_s$. In conclusion: for every $x \in \supp\rho$ we have $\inf_{B_\eps(x)} u^\star  = \essinf_{B_\eps(x)} u^\star $.

\textbf{Step 4:} Finally, we can combine property (1) of the sets $A^\star$, inequalities \labelcref{eqn:Inclusion1} and   \labelcref{eqn:Inclusion2}, and a similar argument (i.e., by contradiction) to the one used in Step 3, in order to conclude that $u^\star = u$ holds ${\nu}$-a.e.
\end{proof}

\end{appendix}

\printbibliography

\end{document}